\documentclass[twoside,11pt]{article}

\usepackage{blindtext}

\usepackage{jmlr2e}

\usepackage{import}
\allowdisplaybreaks[4]

\usepackage{lastpage}

\ShortHeadings{Convergence Analysis of Flow Matching}{Jiao, Lai, Wang and Yan}
\firstpageno{1}

\begin{document}
\thispagestyle{plain}
\title{Convergence Analysis of Flow Matching in Latent Space with Transformers}

\author{\name Yuling Jiao \email yulingjiaomath@whu.edu.cn \\
\addr School of Mathematics and Statistics \\
and Hubei Key Laboratory of Computational Science\\
Wuhan University, Wuhan 430072, China
\AND
\name Yanming Lai \email ylaiam@connect.ust.hk \\
\addr Department of Mathematics\\
The Hong Kong University of Science and Technology\\
Clear Water Bay, Kowloon, Hong Kong, China
\AND
\name Yang Wang \email yangwang@ust.hk \\
\addr Department of Mathematics\\
The Hong Kong University of Science and Technology\\
Clear Water Bay, Kowloon, Hong Kong, China
\AND
\name Bokai Yan\thanks{Corresponding author.} \email byanac@connect.ust.hk \\
\addr Department of Mathematics\\
The Hong Kong University of Science and Technology\\
Clear Water Bay, Kowloon, Hong Kong, China
}

\editor{My editor}

\maketitle

\begin{abstract}
We present theoretical convergence guarantees for ODE-based generative models, specifically flow matching. We use a pre-trained autoencoder network to map high-dimensional original inputs to a low-dimensional latent space, where a transformer network is trained to predict the velocity field of the transformation from a standard normal distribution to the target latent distribution. Our error analysis demonstrates the effectiveness of this approach, showing that the distribution of samples generated via estimated ODE flow converges to the target distribution in the Wasserstein-2 distance under mild and practical assumptions. Furthermore, we show that arbitrary smooth functions can be effectively approximated by transformer networks with Lipschitz continuity, which may be of independent interest.
\end{abstract}


\section{Introduction}

A wide variety of statistics and machine learning problems can be framed as generative modeling, especially when there is an emphasis on accurately modeling and efficiently sampling from intricate distributions, including those associated with images, sound, and text. The essence of generative modeling lies in its ability to learn a target distribution from finite samples, a task at which models incorporating deep neural networks have recently achieved considerable success.

Generative Adversarial Networks (GANs; \citet{goodfellow2014generative, arjovsky2017wasserstein}), as a flagship example of deep generative models, have successfully been applied to a wide range of application challenges, including the synthesis of photorealistic images and videos \citep{radford2015unsupervised, wang2018high, chan2019everybody}, data augmentation \citep{frid2018synthetic}, style transfer \citep{zhu2017unpaired}, and facial editing \citep{karras2019style}. Additionally, significant research has been conducted to analyze the theoretical properties of GANs. \citet{bai2018approximability} demonstrated that GANs could learn distributions within the Wasserstein distance, provided the discriminator class has sufficient distinguishing capability against the generator class. \citet{chen2020statistical} established a minimax optimal convergence rate based on optimal transport theory, which necessitates the input and output dimensions of the generator to be identical. \citet{huang2022error} proved that GANs could learn any distribution with bounded support. Despite their theoretical elegance and practical achievements, GANs often encounter challenges such as training instability, mode collapse, and difficulties in evaluating the quality of generated data.

The recent breakthrough known as the diffusion model has gained notable attention for its superior sample quality and a significantly more stable and controllable training process compared to GANs. The initial concept of the diffusion model involves training a denoising model to progressively transform noise data into samples that adhere to the target distribution \citep{ho2020denoising}, which has soon been mathematically proven to correspond to learning either the drift term of a Stochastic Differential Equation (SDE) or the velocity field of an Ordinary Differential Equation (ODE) \citep{song2021scorebased}. In SDE-based methods, the target data density degenerates into a simpler Gaussian density through the Ornstein-Uhlenbeck (OU) process, followed by solving a reverse-time SDE to generate samples from noise \citep{ho2020denoising, song2021scorebased, meng2021sdedit}. Researchers have also proposed the diffusion Schr\"odinger Bridge (SB), which formulates a finite-time SDE, effectively accelerating the simulation time \citep{de2021diffusion}. The achievements of ODE-based methods are equally remarkable, with most adopting an approach involving interpolative trajectory modeling \citep{liu2022flow, albergo2022building, liu2023flowgrad, xu2022poisson}. \citet{liu2022flow} employs linear interpolation to connect the target distribution with a reference distribution, while \citet{albergo2022building} extends this interpolation to nonlinear cases. Further \citet{gao2023gaussian} uses interpolation to analyze the regularity of a broad class of ODE flows.

In the past few years, there has been an explosive development in SDE/ODE-based generative models, with many models showcasing outstanding performance across a diverse array of application challenges. \citet{dhariwal2021diffusion} have demonstrated that diffusion models outperform GANs in both unconditional and conditional image synthesis, setting a new benchmark in the quality of generated images. \citet{rombach2022high} showed that generative processes operating in a latent space can significantly reduce computational resources while maintaining high-quality text-to-image generation. A line of research \citep{kong2020diffwave, chen2020wavegrad, popov2021grad, liu2022diffsinger} introduced versatile diffusion models capable of synthesizing high-fidelity audio, marking considerable progress in the quality of speech and music generation. Additionally, considerable research has concentrated on text-to-video generation, aiming to create long videos while maintaining high visual quality and adherence to the user's prompt \citep{blattmann2023stable, blattmann2023align, wu2023tune, chen2024videocrafter2, wang2024videocomposer, videoworldsimulators2024}. Despite these models being tailored for various tasks, they typically share two common features. Firstly, they utilize an encoder-decoder architecture to map high-dimensional original inputs to a low-dimensional latent space, where the SDE/ODE-based generative process takes place. Secondly, they employ transformers as the backbone architecture.

Although some analyses have attempted to explain the success of SDE/ODE-based generative models, these analyses either involve technical and unverifiable assumptions or do not align with the models actually used in practice. In a series of studies \citep{lee2022convergence, lee2023convergence, de2022convergence, chen2022sampling, chen2023improved, benton2023linear, conforti2023score}, researchers systematically examined the sampling errors of diffusion models across various target distributions and have determined the optimal sampling error order. Their analysis assumes that the velocity field or drift term in diffusion models has been well-trained, without considering the training process and model selection, thus not providing an end-to-end analysis. It should be noted that end-to-end error analysis is rarely observed even in the domain of general ODE/SDE generative methods. To our knowledge, \cite{wang2021deep} first proved the consistency of the Schr\"odinger Bridge approach through an end-to-end analysis. \citet{oko2023diffusion} proved that in an SDE-based generative model, when the true density function has certain regularities and the empirical score matching loss is properly minimized, the generated data distribution achieves nearly minimax optimal estimation rates in total variation distance and Wasserstein-1 distance. \citet{tang2024adaptivity} further extended the analysis to the intrinsic manifold assumption. \cite{chen2023score} considered a special case in which the encoder and decoder are linear models. \citet{chang2024deep} developed an ODE-based framework and derived a non-asymptotic convergence rate in the Wasserstein-2 distance. However, these analyses do not consider the transformer architecture or incorporate pre-training, which are commonly used in practical implementations, leaving a gap in explaining the success of SDE/ODE-based generative models.

In this paper, we mathematically prove that the distribution of the samples generated via ODE flow converges to the target distribution in the Wasserstein-2 distance under mild and practical assumptions, providing the first comprehensive end-to-end error analysis that considers the transformer architecture and allows for domain shift in pre-training.

\subsection{Our main contributions}
Our main contributions are summarized as follows.
\begin{itemize}
    \item We establish approximation guarantees for transformer networks subject to Lipschitz continuity constraints, which may be of independent interest. (Theorem \ref{theorem: app 3} and \ref{corollary: app 1}). Specifically, we prove that the transformer network can approximate any function, with the Lipschitz continuity of the network remaining independent of the approximation error. Under the assumption that the target distribution has bounded support, we show that the ground truth velocity field is a smooth function, allowing it to be sufficiently approximated by a properly chosen transformer network. 

    \item We establish statistical guarantees for pre-training using the learned encoder and decoder network (Lemma \ref{lemma: ae rate}). Choosing transformer networks as our encoder and decoder, we show that the excessive risk of reconstruction loss converges at a rate of $\widetilde{\mathcal{O}} (m^{-\frac{1}{D+2}})$, where $m$ is the pre-training sample size, only under the assumptions that the pre-trained data distribution has bounded support and that there exist smooth functions minimizing the reconstruction loss.

    \item We establish estimation guarantees for the target distribution using the estimated velocity field (Theorem \ref{theorem: main result}). By choosing proper discretization step size and early stopping time for generating samples, we prove that $\mathbb{E}_{\mathcal{Y}, \mathcal{X}} [W_2(\widehat{\gamma}_T, \gamma_1)] = \mathcal{O} (\sqrt{\varepsilon_{\widetilde{\gamma}_1}} + \varepsilon_{\widetilde{\gamma}_1,  \gamma_1})$, where $\widehat{\gamma}_T$ is the generated data distribution, $\gamma_1$ is the target distribution, $\varepsilon_{\widetilde{\gamma}_1, \gamma_1}$ denotes the domain shift between the target distribution and the pre-trained data distribution, and $\varepsilon_{\widetilde{\gamma}_1}$ is the minimum reconstruction loss achievable by the encoder-decoder architecture. Specifically, if there is no domain shift and the encoder-decoder architecture can perfectly reconstruct the distribution, our results show that the generated data distribution converges to the target distribution in Wasserstein-2 distance.
\end{itemize}

\subsection{Organization}
The rest of the paper is organized as follows. In Section \ref{sec: preliminaries}, we provide notations and introduce key concepts. In Section \ref{sec: approximation}, we show that the true velocity field can be well approximated by a Lipschitz transformer network. In Section \ref{sec: generalization and sampling}, we show that the true velocity field can be efficiently estimated, and analyze the error of distribution recovery using the estimated velocity field. Finally, in Section \ref{sec: end-to-end error}, we analyze the error introduced by the pre-trained autoencoder.

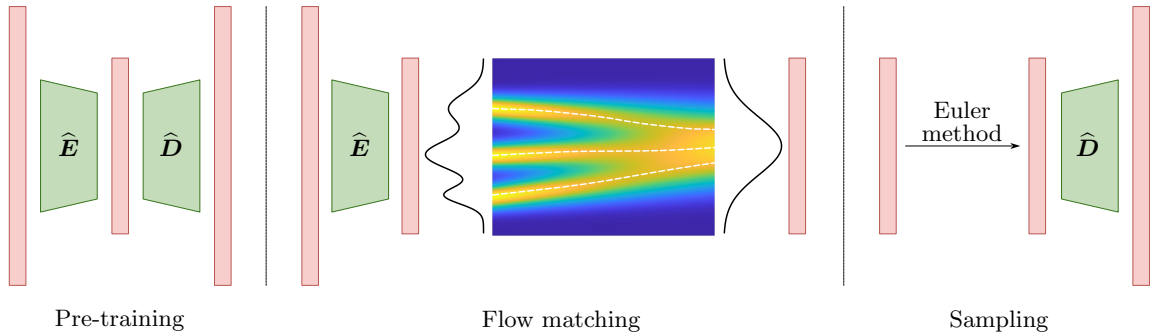
\begin{figure*}[t!]
    \vspace{-0.4cm}
    \centering
    \hspace*{-0.2cm}
    \begin{tikzpicture}
        \node at (0,0) {\includegraphics[width=\linewidth]{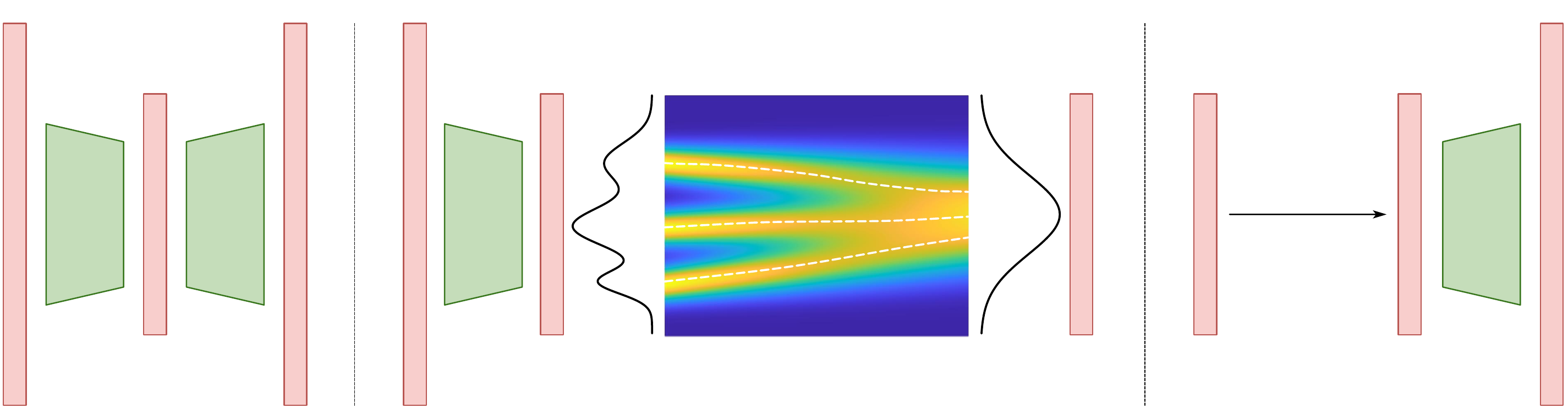}};
        \node at (-6.8,-0.05) {\footnotesize $\widehat{\boldsymbol{E}}$};
        \node at (-5.44,-0.05) {\footnotesize $\widehat{\boldsymbol{D}}$};
        \node at (-2.93,-0.05) {\footnotesize $\widehat{\boldsymbol{E}}$};
        \node at (6.75,-0.05) {\footnotesize $\widehat{\boldsymbol{D}}$};
        \node at (5.08,0.4) {\footnotesize Euler};
        \node at (5.08,0.1) {\footnotesize method};
        \node at (-6.12,-2.4) {\footnotesize Pre-training};
        \node at (-0.25,-2.4) {\footnotesize Flow matching};
        \node at (5.57,-2.4) {\footnotesize Sampling};
    \end{tikzpicture}
    \vspace*{-0.7cm}
    \caption{An illustration of our framework. Pre-training: Based on $m$ samples $\mathcal{Y} = \{\boldsymbol{y}_i\}_{i=1}^m$ drawn i.i.d. from pre-trained data distribution $\widetilde{\gamma}_1$, we minimize the empirical reconstruction loss to obtain an encoder $\widehat{\boldsymbol{E}}: [0,1]^D \rightarrow [0,1]^d$ and the corresponding decoder $\widehat{\boldsymbol{D}}: [0,1]^d \rightarrow \mathbb{R}^D$. These will serve as the bridge linking the high-dimensional input space and the low-dimensional latent space. 
    Flow matching: For the target distribution $\gamma_1$ and $n$ samples $\mathcal{X} = \{\boldsymbol{x}_i\}_{i=1}^n$ drawn from it, the encoder $\widehat{\boldsymbol{E}}$ maps them to the latent space with $\pi_1 = \widehat{\boldsymbol{E}}_{\#} \gamma_1$ and $\widehat{\boldsymbol{E}}(\mathcal{X}) = \{\widehat{\boldsymbol{E}}(\boldsymbol{x}_i)\}_{i=1}^n$. Flow matching is then applied within the latent space, where a transformer network is trained to predict the velocity field of the transformation from a standard normal distribution $\pi_0 = \mathcal{N}(0, I_d)$ to the target latent distribution $\pi_1$. 
    Sampling: Given the estimated velocity field, we can generate samples from an approximation of the continuous flow ODE starting from the prior distribution $\pi_0$. The generated latent data distribution $\widehat{\pi}_T$ will be mapped back to the high-dimensional space by the decoder $\widehat{\boldsymbol{D}}$, resulting in the generated data distribution $\widehat{\gamma}_T = \widehat{\boldsymbol{D}}_{\#} \widehat{\pi}_T$.}
    \vspace*{-0.0cm}
    \label{fig: frame}
\end{figure*}

\section{Preliminaries}
\label{sec: preliminaries}
In this section, we introduce the notations used throughout this paper. Additionally, we provide details about transformer networks, pre-training, and flow matching.

\noindent \textbf{Notations.}
Here we summarize the notations. Given a real number $\alpha$, we denote $\lfloor\alpha\rfloor$ as the largest integer smaller than $\alpha$ (in particular, if $\alpha$ is an integer, $\lfloor\alpha\rfloor=\alpha-1)$. For a vector $\boldsymbol{x} \in \mathbb{R}^d$, we denote its $\ell^2$-norm by $\|\boldsymbol{x}\|$, the $\ell^{\infty}$-norm by $\|\boldsymbol{x}\|_{\infty}=\max_i |x_i|$. We define $\boldsymbol{x}^{\otimes 2}:=\boldsymbol{x} \boldsymbol{x}^\top$. We define the operator norm of a matrix $A$ as $\|A\|_{\text{op }}:=\sup_{\|\boldsymbol{x}\| \leq 1} \|A \boldsymbol{x}\|$. For two matrices $A, B \in \mathbb{R}^{d \times d}$, we say $A \preceq B$ if $B-A$ is positive semi-definite. We denote the identity matrix in $\mathbb{R}^{d \times d}$ by $I_d$. For a twice continuously differentiable function $f: \mathbb{R}^d \rightarrow \mathbb{R}$, let $\nabla f, \nabla^2 f$, and $\Delta f$ denote its gradient, Hessian, and Laplacian, respectively. For a probability density function $\pi$ and a measurable function $f: \mathbb{R}^d \rightarrow \mathbb{R}$, we define the $L^2(\pi)$-norm of $f$ as $\|f\|_{L^2(\pi)}:=(\int(f(\boldsymbol{x}))^2 \pi(\boldsymbol{x}) \mathrm{d} \boldsymbol{x})^{1/2}$. We define $L^{\infty}(K)$-norm as $\|f\|_{L^{\infty}(K)}:=\sup _{\boldsymbol{x} \in K} |f(\boldsymbol{x})|$. The function composition operation is marked as $g \circ f := g(f(x))$ for functions $f$ and $g$. We use the asymptotic notation $f(x)=\mathcal{O}(g(x))$ to denote the statement that $f(x) \leq C g(x)$ for some constant $C>0$ and $\widetilde{\mathcal{O}}(\cdot)$ to ignore the logarithm. For a vector function $\boldsymbol{v}: \mathbb{R}^d \rightarrow \mathbb{R}^{d^\prime}$, we define its $L^2(\pi)$-norm as $\|\boldsymbol{v}\|_{L^2(\pi)}:=\|\| \boldsymbol{v}\|\|_{L^2(\pi)}$ and $L^{\infty}(K)$-norm as $\|\boldsymbol{v}\|_{L^{\infty}(K)}:=\|\| \boldsymbol{v}\|\|_{L^{\infty}(K)}$. For any dataset $\mathcal{D}=\{\boldsymbol{x}_i\}_{i=1}^n$, we define the image of $\mathcal{D}$ under $\boldsymbol{v}$ as $\boldsymbol{v}(\mathcal{D}) := \{\boldsymbol{v}(\boldsymbol{x}_i)\}_{i=1}^n$. Given two distributions $\mu$ and $\nu$, the Wasserstein-2 distance is defined as $W_2(\mu, \nu):=\inf_{\pi \in \Pi(\mu, \nu)} \mathbb{E}_{(x, y) \sim \pi}[\|x-y\|^2]^{1/2}$, where $\Pi(\mu, \nu)$ is the set of all couplings of $\mu$ and $\nu$. A coupling is a joint distribution on $\mathbb{R}^d \times \mathbb{R}^d$ whose marginals are $\mu$ and $\nu$ on first and second factors, respectively. Let $\boldsymbol{f}: \mathbb{R}^d \rightarrow \mathbb{R}^{d^\prime}$ be a measurable mapping and $\mu$ be a probability measure on $\mathbb{R}^d$. The push-forward measure $\boldsymbol{f}_{\#} \mu$ of a measurable set $K$ is defined as $\boldsymbol{f}_{\#} \mu := \mu(\boldsymbol{f}^{-1}(K))$. In neural networks, the Rectified Linear Unit (ReLU) activation function is denoted by $\sigma(x) = \max\{x, 0\}$ and is applied element-wise to vectors or matrices. We define the hardmax operator as $\sigma_H (\boldsymbol{x}) := \lim_{c\rightarrow +\infty} \exp(c\boldsymbol{x})/\|\exp(c\boldsymbol{x})\|_1$, where the operation is performed column-wise if the input to $\sigma_H$ is a matrix. The Hadamard product $\odot$ refers to the element-wise multiplication of two vectors or matrices of the same dimensions.

\subsection{Transformer networks}
In the last few years, academic inquiry has concentrated on the approximation power and generalization capability of ReLU neural networks \citep{yarotsky2017error, suzuki2018adaptivity, bartlett2019nearly, yarotsky2020phase, schmidt2020nonparametric, lu2021deep, shen2022optimal}. These networks become the preferred choice for theoretical analysis and are able to achieve the minimax optimal rate in many problems \citep{huang2022error, duan2022convergence, jiao2023deep, oko2023diffusion, liu2024deep}. In contrast, the theoretical understanding of transformer networks remains limited, despite their resounding success in practical applications. \citet{gurevych2022rate} recently provided a framework to study the approximation properties and generalization abilities of transformer networks. We adopt their framework and extend it by incorporating control over the regularity of the neural network functions.

Given $d, d^\prime \in \mathbb{N}$, we define a transformer network $\boldsymbol{\phi}: \mathbb{R}^{d} \rightarrow \mathbb{R}^{d^\prime}$ as follows:
\begin{align}
\label{eq: app 10}
\boldsymbol{\phi}=E_{out} \circ F_N^{(FF)} \circ F_N^{(SA)} \circ \cdots \circ F_1^{(FF)} \circ F_1^{(SA)} \circ E_{in} \circ P.
\end{align}

The first layer of the transformer network $P$, known as "patchify", divides the spatial input into patches. Namely, an input $\boldsymbol{x}$ of dimension $d$ is transformed into a sequence $X$ of $l$ tokens, where each token has a dimension of $d_{patch}$. These tokens are explicitly selected from components of the input, thus this layer does not require training. For simplicity, we assume $d = d_{patch} \times l$.

The input embedding layer $E_{in}: \mathbb{R}^{(d_{patch}+l) \times l} \rightarrow \mathbb{R}^{d_{model} \times l}$, incorporating position encoding, is a token-wise linear mapping:
\begin{align}
\label{eq: app 1}
Z_0 =E_{in}\left(\text{Concat}\left(
\begin{array}{c}
X \\
\mathbb{I}_l
\end{array} \right)\right)= A_{in} \left(
\begin{array}{c}
X \\
\mathbb{I}_l
\end{array} \right)
+\boldsymbol{b}_{in} \mathbbm{1}_l^{\top}
\end{align}
where $A_{in} \in \mathbb{R}^{d_{model} \times (d_{patch}+l)}$ and $\boldsymbol{b}_{in} \in \mathbb{R}^{d_{model}}$ represent the weight matrix and bias vector of the embedding layer, and $\mathbbm{1}_l$ denotes a vector of $l$ components, each of which is 1.

The multi-head attention layer $F^{(SA)}:\mathbb{R}^{d_{model} \times l} \rightarrow \mathbb{R}^{d_{model} \times l}$ represents the interaction among tokens:
\begin{align}
\label{eq: def sa}
F^{(SA)}(Z) = Z + \sum_{s=1}^h W_{O,s}(W_{V,s} Z)  \left[\left((W_{K,s} Z)^{\top}(W_{Q,s} Z)\right) \odot \sigma_H\left((W_{K,s} Z)^{\top}(W_{Q,s} Z)\right)\right]
\end{align}
where $h \in \mathbb{N}$ is the number of heads which we compute in parallel, $d_k \in \mathbb{N}$ is the dimension of the queries and keys, $d_v \in \mathbb{N}$ is the dimension of the values, $d_{model} = h \cdot d_v$, $W_{K,s}, W_{Q,s} \in \mathbb{R}^{d_k \times d_{model}}, W_{V,s} \in \mathbb{R}^{d_v \times d_{model}}$ and $W_{O,s} \in \mathbb{R}^{d_{model} \times d_v}$ are the weight matrices, 
and $\sigma_H$ is the hardmax operator. We include a skip-connection in the attention layer.

The token-wise feedforward neural network $F^{(FF)}:\mathbb{R}^{d_{model} \times l} \rightarrow \mathbb{R}^{d_{model} \times l}$ processes each token independently in parallel by applying two feedforward layers:
\begin{align*}
F^{(FF)}(Y) = Y + W_2\sigma (W_1 Y + \boldsymbol{b}_1 \mathbbm{1}_l^{\top}) + \boldsymbol{b}_2\mathbbm{1}_l^{\top}
\end{align*}
where $d_{ff} \in \mathbb{N}$ denotes the hidden layer size of the feedforward layer, $W_1\in \mathbb{R}^{d_{ff} \times d_{model}}, \boldsymbol{b}_1 \in \mathbb{R}^{d_{ff}}, W_2 \in \mathbb{R}^{d_{model} \times d_{ff}}$ and $\boldsymbol{b}_2 \in \mathbb{R}^{d_{model}}$ are parameters, and $\sigma$ is the ReLU activation function. The feedforward layer also includes a skip-connection.

The output embedding $E_{out}: \mathbb{R}^{d_{model} \times l} \rightarrow \mathbb{R}^{d^\prime}$,
\begin{align*}
E_{out}(Z) = A_{out} \boldsymbol{z}_1 + \boldsymbol{b}_{out}
\end{align*}
where $Z = (\boldsymbol{z}_1, \boldsymbol{z}_2, \ldots, \boldsymbol{z}_l)$, and $A_{out} \in \mathbb{R}^{d^\prime \times d_{model}}$ and $\boldsymbol{b}_{out} \in \mathbb{R}^{d^\prime}$ are the weight matrix and bias vector. It is important to highlight that only the first column of $Z$, specifically the first token, is used.

Based on the definitions provided, we configure the transformer networks as follows:
\begin{align}
\label{eq: def transformer}
\begin{aligned}
\mathcal{T}_{d,d^\prime} \left(N, h, d_k, d_v, d_{f f}, B, J, \gamma \right)  
= \bigg\{ & \boldsymbol{\phi}: \mathbb{R}^d \rightarrow \mathbb{R}^{d^\prime} : \boldsymbol{\phi} \text{ in the form of } (\ref{eq: app 10}), 
\sup_{\boldsymbol{x}}\|\boldsymbol{\phi}(\boldsymbol{x})\| \leq B, \\ 
& \left\|\boldsymbol{\phi}\left(\boldsymbol{x}_1\right)-\boldsymbol{\phi}\left(\boldsymbol{x}_2\right)\right\| \leq \gamma \left\|\boldsymbol{x}_1-\boldsymbol{x}_2\right\| \text{ for } \boldsymbol{x}_1, \boldsymbol{x}_2 \in [0,1]^d, \\
& \sum_{r=1}^N  \sum_{s=1}^h \left(\left\|W_{Q, r, s}\right\|_0+ \left\|W_{K, r, s}\right\|_0+\left\|W_{V, r, s}\right\|_0+ 
\left\|W_{O, r, s}\right\|_0\right) 
 \\
& +\sum_{r=1}^N \left(\left\|W_{r, 1}\right\|_0+\left\|\boldsymbol{b}_{r, 1}\right\|_0+\left\|W_{r, 2}\right\|_0+\left\|\boldsymbol{b}_{r, 2}\right\|_0\right) \\
& + \|A_{in}\|_0+\|\boldsymbol{b}_{in}\|_0 +  \|A_{out}\|_0+\|\boldsymbol{b}_{out}\|_0 \leq J \bigg\},
\end{aligned}
\end{align}
where $\|\cdot\|_0$ denotes the number of nonzero entries. In the absence of confusion, we write the defined transformer network class as $\mathcal{T}_{d, d^\prime}$ for brevity. In later sections, we will take networks based on (\ref{eq: def transformer}) with appropriate configuration parameters.

\begin{remark}
Compared to the classical transformer architecture \citep{vaswani2017attention, yun2019transformers, dosovitskiy2020image}, our transformer networks have a similar structure, with differences primarily in the multi-head attention layer. We would like to point out that our definition of the multi-head attention layer (\ref{eq: def sa}) is an equivalent reformulation of \citet{gurevych2022rate} and \citet{kohler2023rate}, where they concatenate attention heads. While there are similarities in framework, our approach and focus differ from the aforementioned work. We have incorporated control over the regularity of functions, which is necessary for subsequent analyses in flow matching and autoencoders. Some research concerns the properties of hard attention, which involves replacing the softmax function in the standard attention layer with a hardmax \citep{perez2021attention, hao2022formal}. Our attention layer, in comparison to hard attention, possesses optimization advantages, owing to the continuity and almost everywhere differentiability of the function $\boldsymbol{x} \odot \sigma_H(\boldsymbol{x})$.
\end{remark}

To measure the complexity of transformer network class from a learning theory perspective, we introduce the following notions for a real-valued function class.

\begin{definition}[Pseudo-dimension]
Let $\mathcal{H}$ be a class of real-valued functions defined on $\Omega$. The pseudo-dimension of $\mathcal{H}$, denoted by $\operatorname{Pdim}(\mathcal{H})$, is the largest integer $N$ for which there exist points $x_1, \ldots, x_N \in \Omega$ and constants $y_1, \ldots, y_N \in \mathbb{R}$ such that
\begin{align*}
|\{\operatorname{sgn}(h(x_1) - y_1), \ldots, \operatorname{sgn}(h(x_N)-y_N): h \in \mathcal{H}\}|=2^N.
\end{align*}
\end{definition}

\begin{definition}[Covering number]
Let $\rho$ be a pseudo-metric on $\mathcal{M}$ and $S \subseteq \mathcal{M}$. For any $\delta>0$, a set $A \subseteq \mathcal{M}$ is called a $\delta$-covering of $S$ if for any $x \in S$ there exists $y \in A$ such that $\rho(x, y) \leq \delta$. The $\delta$-covering number of $S$, denoted by $\mathcal{N}(\delta, S, \rho)$, is the minimum cardinality of any $\delta$-covering of $S$.
\end{definition}

Next, we introduce the notion of regularity for a function. For a multi-index $\boldsymbol{\alpha} = (\alpha_1, \ldots, \alpha_d)$, the monomial on $\boldsymbol{x} = (x_1, \ldots, x_d)$ is denoted by $\boldsymbol{x}^{\boldsymbol{\alpha}} := x_1^{\alpha_1} \cdots x_d^{\alpha_d}$, the $\boldsymbol{\alpha}$-derivative of a real-valued function $\phi$ is denoted by $\partial^{\boldsymbol{\alpha}} \phi := \partial^{\|\boldsymbol{\alpha}\|_1} \phi / \partial x_1^{\alpha_1} \cdots \partial x_d^{\alpha_d}$ with $\|\boldsymbol{\alpha}\|_1 = \sum_{i=1}^d \alpha_i$ as the usual 1-norm for vectors. We use the convention that $\partial^{\boldsymbol{\alpha}} \phi := \phi$ if $\|\boldsymbol{\alpha}\|_1=0$.

\begin{definition}[Lipschitz functions]
Let $\Omega \subseteq \mathbb{R}^d$ and $\boldsymbol{\phi}: \Omega \rightarrow \mathbb{R}^{d^\prime}$, the Lipschitz constant of $\boldsymbol{\phi}$ is denoted by
\begin{align*}
\operatorname{Lip} (\boldsymbol{\phi}) := \sup_{\boldsymbol{x}, \boldsymbol{y} \in \Omega, \boldsymbol{x} \neq \boldsymbol{y}} \frac{\|\boldsymbol{\phi}(\boldsymbol{x}) - \boldsymbol{\phi}(\boldsymbol{y})\|}{\|\boldsymbol{x} - \boldsymbol{y}\|}.
\end{align*}
\end{definition}

\begin{definition}[H\"older classes]
Let $\Omega \subseteq \mathbb{R}^d$ and $\beta > 0$. A function is said to possess $\beta$-H\"{o}lder smoothness if all its partial derivatives up to order $\lfloor\beta\rfloor$ exist and are bounded, and the partial derivatives of order $\lfloor\beta\rfloor$ are $\beta-\lfloor\beta\rfloor$ H\"{o}lder. For $d,d^\prime \in \mathbb{N}$, the H\"{o}lder class with smoothness index $\beta$ and norm constraint parameter $K$ is then defined as
\begin{align*}
\begin{aligned}
& \mathcal{H}_{d,d^\prime}^\beta (\Omega, K) = \Bigg\{ \boldsymbol{f} = (f_1, \ldots, f_{d^\prime})^\top: \Omega \rightarrow \mathbb{R}^{d^\prime}, \\ 
& \sum_{\boldsymbol{n}: \|\boldsymbol{n}\|_1 < \beta} \|\partial^{\boldsymbol{n}} f_k\|_{L^\infty(\Omega)} + \sum_{\boldsymbol{n}: \|\boldsymbol{n}\|_1 = \lfloor\beta\rfloor} \sup_{\boldsymbol{x}, \boldsymbol{y} \in \Omega,
\boldsymbol{x} \neq \boldsymbol{y}} \frac{|\partial^{\boldsymbol{n}} f_k(\boldsymbol{x})-\partial^{\boldsymbol{n}} f_k(\boldsymbol{y})|}{\|\boldsymbol{x} - \boldsymbol{y}\|^{\beta-\lfloor\beta\rfloor}} \leq K, \quad k=1, \ldots, d^\prime \Bigg\}.
\end{aligned}
\end{align*}
\end{definition}

\begin{definition}[Differentiability classes]
Let $\Omega \subseteq \mathbb{R}^d$ and $m \in \mathbb{N}$. For $d,d^\prime \in \mathbb{N}$, the differentiability class with smoothness index $m$ and norm constraint parameter $K$ is defined as
\begin{align*}
\begin{aligned}
\mathcal{C}_{d,d^\prime}^m(\Omega, K) = \Bigg\{ \boldsymbol{f} = (f_1, \ldots, f_{d^\prime})^\top: \Omega \rightarrow \mathbb{R}^{d^\prime},  \sum_{\boldsymbol{n}: \|\boldsymbol{n}\|_1 \leq m} \|\partial^{\boldsymbol{n}} f_k\|_{L^\infty(\Omega)} \leq K, \quad k=1, \ldots, d^\prime \Bigg\}.
\end{aligned}
\end{align*}
\end{definition}

\subsection{Pre-training}
For any measurable functions $\boldsymbol{E}: \mathbb{R}^D \rightarrow \mathbb{R}^d$ and $\boldsymbol{D}: \mathbb{R}^d \rightarrow \mathbb{R}^D$, we minimize the reconstruction loss w.r.t. the pre-trained data distribution $\widetilde{\gamma}_1$
\begin{align*}
(\boldsymbol{D}^*, \boldsymbol{E}^*) \in \argmin{\boldsymbol{D}, \boldsymbol{E} \text{ measurable}} \mathcal{R} (\boldsymbol{D}, \boldsymbol{E}) := \int_{\mathbb{R}^d} \| (\boldsymbol{D}\circ\boldsymbol{E})(\boldsymbol{y}) - \boldsymbol{y} \|^2 \mathrm{d} \widetilde{\gamma}_1.
\end{align*}
Our analysis on pre-training requires the following assumptions.

\begin{assumption}[Bounded support]
\label{ass: bounded support gammahat}
The pre-trained data distribution $\widetilde{\gamma}_1$ is supported on $[0,1]^D$.
\end{assumption}

\begin{assumption}[Compressibility]
\label{ass: compressibility}
There exist continuously differentiable functions $\boldsymbol{E}^*: [0,1]^D \rightarrow [0,1]^d$ and $\boldsymbol{D}^*: [0,1]^d \rightarrow \mathbb{R}^D$ such that $\mathcal{R} (\boldsymbol{D}, \boldsymbol{E})$ attains its minimum. The minimum value is denoted by $\varepsilon_{\widetilde{\gamma}_1} := \mathcal{R} (\boldsymbol{D}^*, \boldsymbol{E}^*)$. Furthermore, $\boldsymbol{E}^* \in \mathcal{C}^1_{D,d} ([0,1]^D, K_{\boldsymbol{E}}), \boldsymbol{D}^* \in \mathcal{C}^1_{d,D} ([0,1]^d, K_{\boldsymbol{D}})$.
\end{assumption}

\begin{remark}
The restriction on the range of the image of $\boldsymbol{E}^*$ in Assumption \ref{ass: compressibility} is not essential. Since $\boldsymbol{E}^*$ is a continuous function, there exists a constant $R>0$ such that $\boldsymbol{E}^*([0,1]^D) \subseteq [-R, R]^d$. Let $\widetilde{\boldsymbol{E}}^*(\boldsymbol{y}) := \frac{1}{2R} \boldsymbol{E}^*(\boldsymbol{y}) + \frac{1}{2} \mathbbm{1}_d$ and $\widetilde{\boldsymbol{D}}^*(\boldsymbol{y}) := \boldsymbol{D}^* (2R\boldsymbol{y} - R\mathbbm{1}_d)$. Then $\mathcal{R} (\boldsymbol{D}, \boldsymbol{E}) = \mathcal{R} (\widetilde{\boldsymbol{D}}^*, \widetilde{\boldsymbol{E}}^*)$ and $\widetilde{\boldsymbol{E}}^*([0,1]^D) \subseteq [0,1]^d$. Therefore, it is permissible to assume, without loss of generality, $\boldsymbol{E}^*([0,1]^D) \subseteq [0,1]^d$ and $\boldsymbol{D}^*: [0,1]^d \rightarrow \mathbb{R}^D$.
\end{remark}

\begin{assumption}[Bounded support]
\label{ass: bounded support gamma}
The target distribution $\gamma_1$ is supported on $[0,1]^D$.
\end{assumption}

We denote the domain shift between the pre-trained data distribution $\widetilde{\gamma}_1$ and the target distribution $\gamma_1$ in Wasserstein-2 distance as $\varepsilon_{\widetilde{\gamma}_1, \gamma_1} := W_2 (\widetilde{\gamma}_1, \gamma_1)$.  In our analysis, $\boldsymbol{E}^*$ and $\boldsymbol{D}^*$ are approximated by neural networks. By constraining the chosen encoder network $\widehat{\boldsymbol{E}}: [0,1]^D \rightarrow [0,1]^d$, Assumption \ref{ass: bounded support gamma} implies that the latent target distribution $\pi_1 := \widehat{\boldsymbol{E}}_{\#} \gamma_1$ is supported on $[0,1]^d$.

\subsection{Flow matching}
Given independent empirical observations of $X_0 \sim \pi_0$ and $X_1 \sim \pi_1$, we want to find an ODE on time $t \in [0,1]$,
\begin{align*}
\mathrm{d} Z_t = \boldsymbol{v}\left(Z_t, t\right) \mathrm{d} t,
\end{align*}
which converts $Z_0$ from $\pi_0$ to $Z_1$ following $\pi_1$. A line of research points out that the vector field can be found by solving a least square regression problem
\begin{align}
\label{eq: fm 1}
\min_{\boldsymbol{v}} \int_0^1 \mathbb{E}_{X_0, X_1}\left[\left\|\left(X_1-\frac{t}{\sqrt{1-t^2}} X_0\right)-\boldsymbol{v}\left(X_t, t\right)\right\|^2\right] \mathrm{d} t
\end{align}
with $X_t=t X_1+\sqrt{1-t^2} X_0$, where $X_0 \sim \pi_0, X_1 \sim \pi_1$, and $X_t$ is the interpolation between $X_0$ and $X_1$. The exact minimum of (\ref{eq: fm 1}) is achieved by
\begin{align}
\label{eq:gen 13}
\boldsymbol{v}^*(\boldsymbol{x}, t)=\mathbb{E}_{X_0, X_1}\left[X_1-\frac{t}{\sqrt{1-t^2}}X_0 \bigg| X_t=\boldsymbol{x}\right].
\end{align}

In practice, the velocity field $\boldsymbol{v}^*$ is approximated by neural networks. To avoid instability, we often clip the interval $[0,1]$ with $T$. Namely, we consider the truncated loss function
\begin{align}
\label{eq: fm 2}
\min_{\boldsymbol{v}} \mathcal{L}(\boldsymbol{v}) := \frac{1}{T} \int_0^T \mathbb{E}_{X_0, X_1}\left[\left\|\left(X_1-\frac{t}{\sqrt{1-t^2}} X_0\right)-\boldsymbol{v}\left(X_t, t\right)\right\|^2\right] \mathrm{d} t
\end{align}
with $X_t=t X_1+\sqrt{1-t^2} X_0$.

\section{Approximation}
\label{sec: approximation}
This section examines the approximation error involved in estimating the true velocity field. To begin with, we explore the approximation capabilities of transformer networks with constrained Lipschitz constants, as detailed in the following theorems.

\begin{theorem}
\label{theorem: app 3}
Let $0<\varepsilon<1$ and $\beta>0$. For any function $\boldsymbol{f} \in \mathcal{H}_{d,d^\prime}^\beta([0, 1]^{d}, K)$, there exists a transformer network $\boldsymbol{\phi} \in \mathcal{T}_{d, d^{\prime}}\left(N, h, d_k, d_v, d_{ff}, B, J, \gamma\right)$, where 
\begin{align*}
\begin{gathered}
N = \mathcal{O} \left( \log \left(\frac{K}{\varepsilon}\right)\right), \quad
h = \mathcal{O} \left(  \left( \frac{K}{\varepsilon} \right)^{d/\beta} \right), \quad
d_{ff} = 8 h, \quad 
d_k = \mathcal{O}(1), \quad d_v = \mathcal{O}(1)\\
B = \mathcal{O} \left(\|\boldsymbol{f}\|_{L^{\infty}([0, 1]^d)}\right), \quad J = \mathcal{O} \left(\left( \frac{K}{\varepsilon} \right)^{d/\beta} \log \left(\frac{K}{\varepsilon}\right) \right), 
\end{gathered}
\end{align*}
such that
\begin{align*}
\| \boldsymbol{\phi}(\boldsymbol{x}) - \boldsymbol{f}(\boldsymbol{x}) \|_{L^{\infty}([0, 1]^d)} \leq \varepsilon.
\end{align*}
Furthermore, if $\beta>1$, we may choose 
\begin{align*}
\gamma = \mathcal{O} (K).
\end{align*}
\end{theorem}

\begin{theorem}
\label{corollary: app 1}
Let $0<\varepsilon<1$ and $m\in\mathbb{N}$. For any function $\boldsymbol{f} \in \mathcal{C}_{d,d^\prime}^m([0, 1]^{d}, K)$, there exists a transformer network $\boldsymbol{\phi} \in \mathcal{T}_{d, d^{\prime}}\left(N, h, d_k, d_v, d_{ff}, B, J, \gamma\right)$, where 
\begin{align*}
\begin{gathered}
N = \mathcal{O} \left( \log \left(\frac{K}{\varepsilon}\right)\right), \quad
h = \mathcal{O} \left(  \left( \frac{K}{\varepsilon} \right)^{d/m} \right), \quad
d_{ff} = 8 h, \quad 
d_k = \mathcal{O}(1), \quad d_v = \mathcal{O}(1) \\
B = \mathcal{O} \left(\|\boldsymbol{f}\|_{L^{\infty}([0, 1]^d)}\right), \quad J = \mathcal{O} \left(\left( \frac{K}{\varepsilon} \right)^{d/m} \log \left(\frac{K}{\varepsilon}\right) \right), 
\end{gathered}
\end{align*}
such that
\begin{align*}
\left\| \boldsymbol{\phi}(\boldsymbol{x}) - \boldsymbol{f}(\boldsymbol{x}) \right\|_{L^{\infty}([0, 1]^d)} \leq \varepsilon.
\end{align*}
Furthermore, if $m\geq 1$, we may choose 
\begin{align*}
\gamma = \mathcal{O} (K).
\end{align*}
\end{theorem}
The proof of Theorem \ref{theorem: app 3} and Theorem \ref{corollary: app 1} can be found in Appendix \ref{appendix: app.3}.

\begin{remark}
Theorem \ref{theorem: app 3} and Theorem \ref{corollary: app 1} provide theoretical guarantees for the approximation capabilities of transformer networks with constrained Lipschitz constants. To effectively control the Lipschitz constants of networks in practical applications, various methods are employed, including spectral normalization \citep{miyato2018spectral}, batch normalization \citep{ioffe2015batch}, weight clipping \citep{arjovsky2017wasserstein}, gradient clipping \citep{lecun2015deep}, and gradient penalty \citep{gulrajani2017improved}. 
\end{remark}

\begin{remark}
Theorem \ref{theorem: app 3} and Theorem \ref{corollary: app 1} improve the approximation guarantee in \citet[Theorem 2]{gurevych2022rate} with additional Lipschitz continuity characterization. \citet{huang2022error} introduced control over Lipschitz continuity for ReLU neural networks. \citet[Lemma 10]{chen2020distribution} shows that ReLU neural networks can approximate Lipschitz continuous functions, while the Lipschitz continuity of the network remains independent of the approximation error. Their approach, contingent upon the structure of ReLU networks, is applicable solely to Lipschitz continuous target functions. We would like to highlight that we use a distinct proof method, enabling the Lipschitz continuity of the constructed transformer network to be independent of approximation error, while remaining applicable to target functions with higher regularity.
\end{remark}

We proceed to show that although the spatial input $\boldsymbol{x}$ in (\ref{eq:gen 13}) can be arbitrary in $\mathbb{R}^d$, by restricting $\boldsymbol{x}$ to a compact set, the true velocity field $\boldsymbol{v}^*$ can be effectively approximated. In our approach, we introduce time $t$ as an extra input dimension to the neural network and define the rescaled function space as 
\begin{align}
\label{eq: def tao}
\begin{aligned}
\mathcal{T} \left(N, h, d_k, d_v, d_{f f}, B, J, \gamma_{\boldsymbol{x}}, \gamma_{t}, R\right)  
= \bigg\{ & \boldsymbol{v}(\boldsymbol{x}, t) = \widetilde{\boldsymbol{v}} \left(\frac{1}{2R}(\operatorname{Proj}_{[-R,R]^d} (\boldsymbol{x}) +R\mathbbm{1}_d), \frac{1}{T} t\right): \\
& \widetilde{\boldsymbol{v}}(\boldsymbol{x}^\prime, t^\prime) \in \mathcal{T}_{d+1, d} \left(N, h, d_k, d_v, d_{f f}, B, J, \gamma \right), \\
& \gamma_{\boldsymbol{x}} = \frac{\gamma}{2R}, \gamma_{t} = \frac{\gamma}{T}  \bigg\},
\end{aligned}
\end{align}
where $\operatorname{Proj}_{\Omega} (\boldsymbol{x}) := \operatorname{arg\,min}_{\boldsymbol{y} \in \Omega} \|\boldsymbol{y} - \boldsymbol{x}\|$ denotes the projection operator onto the set $\Omega$. This definition ensures that $\boldsymbol{v} \in \mathcal{T} (N, h, d_k, d_v, d_{f f}, B, J, \gamma_{\boldsymbol{x}}, \gamma_{t}, R)$ is Lipschitz continuous over $\mathbb{R}^d \times [0,T]$, i.e. 
\begin{align*}
\|\boldsymbol{v}(\boldsymbol{x}_1, t) - \boldsymbol{v}(\boldsymbol{x}_2, t)\| &\leq \gamma_{\boldsymbol{x}} \|\boldsymbol{x}_1 - \boldsymbol{x}_2\| \text{ for any } t \in [0,T], \\
\|\boldsymbol{v}(\boldsymbol{x}, t_1) - \boldsymbol{v}(\boldsymbol{x}, t_2)\| &\leq \gamma_{t} \, |t_1 - t_2| \text{ for any } \boldsymbol{x} \in \mathbb{R}^d.
\end{align*}

\begin{corollary}
\label{coro: app true vd}
Suppose Assumption \ref{ass: bounded support gamma} holds. Let $\frac{1}{2} < T < 1$ and $R \geq 1$. Given an approximation error $0<\varepsilon<1$, for any velocity field $\boldsymbol{v}^*$, we choose the hypothesis class $\mathcal{T}$ with
\begin{align*}
\begin{gathered}
N = \mathcal{O} \left( \log \left(\frac{R}{(1-T)^3\varepsilon}\right)\right), \quad
h = \mathcal{O} \left(  \left( \frac{R}{(1-T)^3 \varepsilon} \right)^{d+1} \right), \quad
d_{ff} = 8 h, \\ 
d_k = \mathcal{O}(1), \quad d_v = \mathcal{O}(1), \quad
B = \mathcal{O} \left( \frac{R}{1-T} \right), \quad
J = \mathcal{O} \left(\left( \frac{R}{(1-T)^3\varepsilon} \right)^{d+1} \log \left(\frac{R}{(1-T)^3\varepsilon}\right) \right), \\
\gamma_{\boldsymbol{x}} = \mathcal{O} \left(\frac{1}{(1-T)^3}\right), \quad \gamma_{t} = \mathcal{O} \left(\frac{R}{(1-T)^3}\right).
\end{gathered}
\end{align*}
Then there exists a $\boldsymbol{v}(\boldsymbol{x}, t) \in \mathcal{T}$ such that 
\begin{align*}
\left\| \boldsymbol{v}(\boldsymbol{x}, t) - \boldsymbol{v}^*(\boldsymbol{x}, t) \right\|_{L^\infty ([-R, R]^d \times [0, T])} \leq \varepsilon.
\end{align*}
\end{corollary}
The proof can be found in Appendix \ref{appendix: app.4}.

\section{Generalization and Sampling}
\label{sec: generalization and sampling}
In this section, we consider the generalization error of estimating the velocity field and establish distribution recovery guarantees using the estimated velocity field. We begin with the following connection between the loss function $\mathcal{L}(\boldsymbol{v})$ and the $L^2$ approximation error $\left\|\boldsymbol{v}(\cdot, t)-\boldsymbol{v}^*(\cdot, t)\right\|_{L^2\left(\pi_t\right)}$.

\begin{lemma}
\label{lemma:gen 2}
For any velocity field $\boldsymbol{v}: \mathbb{R}^d \times [0, T] \rightarrow \mathbb{R}^d$, we have
\begin{align*}
\mathcal{L}(\boldsymbol{v})-\mathcal{L}(\boldsymbol{v}^*) = \frac{1}{T} \int_0^T \|\boldsymbol{v} (\cdot, t)-\boldsymbol{v}^*(\cdot, t)\|_{L^2(\pi_t)}^2 \mathrm{d} t.
\end{align*}
\end{lemma}
The proof can be found in Appendix \ref{appendix: gen.1}. According to Lemma \ref{lemma:gen 2}, minimizing (\ref{eq: fm 2}) is equivalent to minimizing the difference between the network and the true velocity field in $L^2 (\pi_t)$-norm.

In this paper, we choose the standard Gaussian distribution as the prior distribution, i.e., $\pi_0 = \mathcal{N}(0, I_d)$, where $d$ is the dimension of the latent space. Given $n$ independent and identically distributed (i.i.d.) samples $\{\boldsymbol{x}_{1, i}\}_{i=1}^n$ from $\pi_1$, and $n$ i.i.d. samples $\{\left(t_i, \boldsymbol{x}_{0, i}\right)\}_{i=1}^n$ from $\mathrm{Unif}[0, T]$ and $\pi_0$, which are cheap to generate, we denote the dataset as $\mathcal{X} := \{t_i, \boldsymbol{x}_{0,i}, \boldsymbol{x}_{1,i}\}_{i=1}^n$ and consider the empirical risk minimization
\begin{align}
\label{eq:gen 18}
\widehat{\boldsymbol{v}} \in \argmin{\boldsymbol{v} \in \mathcal{T}}~ \widehat{\mathcal{L}}(\boldsymbol{v}):=\frac{1}{n} \sum_{i=1}^n \left\|\boldsymbol{x}_{1,i} - \frac{t_i}{\sqrt{1-t_i^2}} \boldsymbol{x}_{0,i}-\boldsymbol{v}\left(t_i \boldsymbol{x}_{1,i} + \sqrt{1-t_i^2} \boldsymbol{x}_{0,i}, t_i\right)\right\|^2.
\end{align}
Our analysis gives the following generalization bound.

\begin{theorem}
\label{theorem: gen 1}
Suppose Assumption \ref{ass: bounded support gamma} holds. Let $\frac{1}{2}<T<1$. For any velocity field $\boldsymbol{v}^*$, given $n$ i.i.d. samples $\mathcal{X} = \{t_i, \boldsymbol{x}_{0,i}, \boldsymbol{x}_{1,i}\}_{i=1}^n$ from $\mathrm{Unif}[0, T]$, $\pi_0$ and $\pi_1$, we choose $\mathcal{T}$ as in Corollary \ref{coro: app true vd} with $\varepsilon = n^{-\frac{1}{d+3}}$ and $R = \mathcal{O}(\sqrt{\log n})$. Then it holds
\begin{align*}
\mathbb{E}_{\mathcal{X}} \left[\mathcal{L}(\widehat{\boldsymbol{v}}) - \mathcal{L}(\boldsymbol{v}^*)\right] = \mathbb{E}_{\mathcal{X}} \left[ \frac{1}{T} \int_0^T \|\widehat{\boldsymbol{v}}(\cdot, t) - \boldsymbol{v}^*(\cdot, t)\|_{L^2 (\pi_t)}^2 \mathrm{d} t \right] = \widetilde{\mathcal{O}}\left( \frac{1}{(1-T)^{3d+5}} n^{-\frac{2}{d+3}} \right),
\end{align*}
where we omit factors in $d, \log n, \log (1-T)$. 
\end{theorem}
The proof can be found in Appendix \ref{appendix: gen.1}. Theorem \ref{theorem: gen 1} becomes vacuous when $T$ tends to $1$ with fixed sample size $n$. This is a consequence of the blowup of the velocity field $\boldsymbol{v}^*(\boldsymbol{x}, t)$ as $t$ tends to $1$. Although a smaller early stopping time leads to better generalization error, stopping the sampling process at an early time results in a bad distribution recovery.

Given the estimated velocity field $\widehat{\boldsymbol{v}}$, we can generate samples from an approximation of the continuous flow ODE starting from the prior distribution:
\begin{align}
\label{eq: estimated sampling} 
\mathrm{d} \widetilde{X}_t (\boldsymbol{x}) = \widehat{\boldsymbol{v}} (\widetilde{X}_t(\boldsymbol{x}), t) \mathrm{d} t, \ \widetilde{X}_0 (\boldsymbol{x}) = \boldsymbol{x} \sim \pi_0, \  0 \leq t \leq T.
\end{align}
In practice, we need to use a discrete-time approximation for the sampling dynamics (\ref{eq: estimated sampling}). Let $0 = t_0 < t_1 < \cdots < t_N = T$ be the discretization points. We consider the explicit Euler discretization scheme:
\begin{align*}
\mathrm{d} \widehat{X}_t (\boldsymbol{x}) = \widehat{\boldsymbol{v}} (\widehat{X}_{t_k}(\boldsymbol{x}), t_k) \mathrm{d} t,\  t \in \left[t_k, t_{k+1}\right),
\end{align*}
for $k = 0, 1, \ldots, N-1$ and $\widehat{X}_0 (\boldsymbol{x}) = \boldsymbol{x} \sim \pi_0$. We denote the distribution of $\widehat{X}_T (\boldsymbol{x})$ by $\widehat{\pi}_T$.

\begin{theorem}
\label{theorem: consistency in latent space}
Suppose Assumption \ref{ass: bounded support gamma} holds. Given $n$ samples from latent target distribution $\pi_1$ and the networks as in Corollary \ref{coro: app true vd}, we use the estimated velocity field in (\ref{eq:gen 18}) to generate samples. By choosing the maximal step size $\max_{k=0,1 \ldots, N-1}\left|t_{k+1}-t_k\right|=\mathcal{O} (n^{-\frac{1}{d+3}})$ and the early stopping time $T(n)=1-(\log n)^{-1/6}$, we achieve
\begin{align*}
\mathbb{E}_{\mathcal{X}} [W_2(\widehat{\pi}_T, \pi_1)] \rightarrow 0.
\end{align*}
\end{theorem}
The proof can be found in Appendix \ref{appendix: dis.2}. Theorem \ref{theorem: consistency in latent space} demonstrates the consistency of flow matching in latent space. The consistency is mainly based on a mild assumption, i.e. boundedness, which justifies the use of continuous normalizing flows based on flow matching.

\section{End-to-end Error Analysis}
\label{sec: end-to-end error}
Data dimension reduction and restoration, serving as key steps in Stable Diffusion \citep{rombach2022high}, are accomplished by an autoencoder network. An autoencoder is a distinct neural network architecture designed to learn the low-dimensional features of data \citep{ballard1987modular, bourlard1988auto,  kramer1991nonlinear, hinton1993autoencoders}. The classical autoencoder typically consists of two subnetworks, an encoder and a decoder. The encoder processes high-dimensional input data into a low-dimensional latent representation, aiming to capture the intrinsic structure of the data in a compact form. The decoder then maps these latent features back to the high-dimensional space, striving to accurately reconstruct the original input. These processes can be summarized by the following model.

Given $m$ samples $\mathcal{Y} := \{\boldsymbol{y}_i\}_{i=1}^m$ drawn i.i.d. from the pre-trained data distribution $\widetilde{\gamma}_1$, we can devise an estimator through empirical risk minimization
\begin{align}
\label{eq: autoencoder erm}
(\widehat{\boldsymbol{D}}, \widehat{\boldsymbol{E}}) \in \argmin{\boldsymbol{D} \in \mathcal{D}, \boldsymbol{E} \in \mathcal{E}} \widehat{\mathcal{R}} (\boldsymbol{D}, \boldsymbol{E}) := \frac{1}{m} \sum_{i=1}^m \| (\boldsymbol{D}\circ\boldsymbol{E})(\boldsymbol{y}_i) - \boldsymbol{y}_i \|^2,
\end{align}
where we specify the encoder network architecture as
\begin{align}
\label{eq: def encoder network}
\mathcal{E} = \mathcal{T}_{D, d}(N_{\boldsymbol{E}}, h_{\boldsymbol{E}}, d_{\boldsymbol{E}, k}, d_{\boldsymbol{E}, v}, d_{\boldsymbol{E}, ff}, B_{\boldsymbol{E}}, J_{\boldsymbol{E}}, \gamma_{\boldsymbol{E}})
\end{align}
and the decoder network architecture as
\begin{align}
\label{eq: def decoder network}
\mathcal{D} = \mathcal{T}_{d, D}(N_{\boldsymbol{D}}, h_{\boldsymbol{D}}, d_{\boldsymbol{D}, k}, d_{\boldsymbol{D}, v}, d_{\boldsymbol{D}, ff}, B_{\boldsymbol{D}}, J_{\boldsymbol{D}}, \gamma_{\boldsymbol{D}}).
\end{align}

\begin{lemma}
\label{lemma: ae rate}
Suppose Assumption \ref{ass: bounded support gammahat} and \ref{ass: compressibility} hold. Given $m$ samples from pre-trained data distribution $\widetilde{\gamma}_1$, we choose
\begin{align*}
\begin{gathered}
N_{\boldsymbol{E}} = \mathcal{O} \left( \log \left(K_{\boldsymbol{E}} m\right)\right), \quad
h_{\boldsymbol{E}} = \mathcal{O} \left( K_{\boldsymbol{E}}^D m^{\frac{D}{D+2}} \right), \quad
d_{\boldsymbol{E}, ff} = 8 h, \quad 
d_{\boldsymbol{E}, k} = \mathcal{O}(1), \quad d_{\boldsymbol{E}, v} = \mathcal{O}(1) \\
B_{\boldsymbol{E}} = \mathcal{O} \left(K_{\boldsymbol{E}}\right), \quad J_{\boldsymbol{E}} = \mathcal{O} \left( K_{\boldsymbol{E}}^D m^{\frac{D}{D+2}} \log \left(K_{\boldsymbol{E}} m\right) \right), \quad \gamma_{\boldsymbol{E}} = \mathcal{O} \left(K_{\boldsymbol{E}}\right)
\end{gathered}
\end{align*}
for the encoder networks in (\ref{eq: def encoder network}) and 
\begin{align*}
\begin{gathered}
N_{\boldsymbol{D}} = \mathcal{O} \left( \log \left(K_{\boldsymbol{D}} m\right)\right), \quad
h_{\boldsymbol{D}} = \mathcal{O} \left( K_{\boldsymbol{D}}^d m^{\frac{d}{D+2}} \right), \quad
d_{\boldsymbol{D}, ff} = 8 h, \quad 
d_{\boldsymbol{D}, k} = \mathcal{O}(1), \quad d_{\boldsymbol{D}, v} = \mathcal{O}(1) \\
B_{\boldsymbol{D}} = \mathcal{O} \left(K_{\boldsymbol{D}}\right), \quad 
J_{\boldsymbol{D}} = \mathcal{O} \left( K_{\boldsymbol{D}}^d m^{\frac{d}{D+2}} \log \left(K_{\boldsymbol{D}} m\right) \right), \quad \gamma_{\boldsymbol{D}} = \mathcal{O} \left(K_{\boldsymbol{D}}\right)
\end{gathered}
\end{align*}
for the decoder networks in (\ref{eq: def decoder network}). Then it holds
\begin{align*}
\mathbb{E}_{\mathcal{Y}} [ \mathcal{R}(\widehat{\boldsymbol{D}}, \widehat{\boldsymbol{E}}) ] - \varepsilon_{\widetilde{\gamma}_1} = \mathcal{O} \left( m^{-\frac{1}{D+2}} \log^{5/2} m \right), 
\end{align*}
where $\mathcal{O}$ hides factors in $D, d, K_{\boldsymbol{E}}, K_{\boldsymbol{D}}$. 
\end{lemma}
The proof can be found in Appendix \ref{appendix: autoencoder}. Theorem \ref{lemma: ae rate} provides estimation guarantees for the pre-trained data distribution $\widetilde{\gamma}_1$. A line of research \citep{schonsheck2019chart, liu2024deep, liu2024generalization} relies on the assumption that the data follows $\boldsymbol{x} = \boldsymbol{D} \circ \boldsymbol{E}(\boldsymbol{x})$ when analyzing autoencoders. The Assumption \ref{ass: compressibility} represents a relaxation of the aforementioned assumption, reverting to the previous case when $\varepsilon_{\widetilde{\gamma}_1} = 0$.

With the pre-trained encoder and decoder, we can bridge the original high-dimensional space to the low-dimensional latent space, as illustrated in Figure \ref{fig: frame}. Denote the distribution of the generated samples as $\widehat{\gamma}_T := \widehat{\boldsymbol{D}}_{\#} \widehat{\pi}_T$. The following theorem provides estimation guarantees for the target distribution $\gamma_1$.

\begin{theorem}[Main result]
\label{theorem: main result}
Suppose Assumption \ref{ass: bounded support gammahat} - \ref{ass: bounded support gamma} hold. Given $m$ samples from the pre-trained distribution $\widetilde{\gamma}_1$, we use the pre-trained autoencoder in (\ref{eq: autoencoder erm}). Given $n$ samples from target distribution $\gamma_1$ and the networks as in Theorem \ref{theorem: gen 1}, we use the estimated velocity field in (\ref{eq:gen 18}) to generate samples. By choosing the maximal step size $\max_{k = 0,1 \ldots, N-1} |t_{k+1}-t_k| = \mathcal{O} (n^{-\frac{1}{d+3}})$ and the early stopping time $T(n) = 1-(\log n)^{-1 / 6}$, we have
\begin{align*}
\mathbb{E}_{\mathcal{Y}, \mathcal{X}} [W_2(\widehat{\gamma}_T, \gamma_1)] = \mathcal{O} (\sqrt{\varepsilon_{\widetilde{\gamma}_1}} + \varepsilon_{\widetilde{\gamma}_1,  \gamma_1})
\end{align*}
when $m, n$ are sufficiently large. Moreover, if $\varepsilon_{\widetilde{\gamma}_1} = 0$ and $\varepsilon_{\widetilde{\gamma}_1,  \gamma_1} = 0$, we have 
\begin{align*}
\mathbb{E}_{\mathcal{Y}, \mathcal{X}} [W_2(\widehat{\gamma}_T, \gamma_1)] \rightarrow 0
\end{align*}
as $m, n \rightarrow \infty$. 
\end{theorem}
The proof can be found in Appendix \ref{appendix: autoencoder}. Theorem \ref{theorem: main result} presents the first convergence analysis incorporating the transformer architecture and pre-training. Our findings indicate that the difference in the distribution of generated samples and the target distribution in Wasserstein-2 distance can be controlled by domain shift and the minimum reconstruction loss achievable by the encoder-decoder architecture. Furthermore, when domain shift is absent and the encoder-decoder architecture can perfectly reconstruct the distribution, the distribution of generated samples converges to the target distribution.

\section{Conclusion and Future Work}
In this work, we presents a statistical learning theory perspective on continuous normalizing flows based on flow matching. We demonstrate that a Lipschitz transformer network can approximate the true velocity field under $L^{\infty}$-norm and provide a sample complexity analysis for estimating the velocity field. Our analysis has identified the bias introduced by the encoder-decoder architecture in estimating the target distribution. Furthermore, we prove that under mild assumptions, the generated distribution based on flow matching converges to the target data distribution in Wasserstein-2 distance. To the best of our knowledge, we provide the first end-to-end error analysis that considers the transformer architecture and pre-training.

There are several natural directions for future research. Firstly, our analysis can be naturally extended to latent conditional SDE/ODE-based diffusion models, as in \citep{dao2023flow, peebles2023scalable}, where prompts or category information can be incorporated as conditions into the model. Furthermore, with stronger assumptions, there is hope to obtain convergence rates that depend only on $d$, which can help to overcome the curse of dimensionality. Finally, the existing analysis still cannot fully explain why various designs of attention layers are effective. This is a challenging task because the nonlinearities introduced by complex structures are difficult to handle and analyze. Developing a better understanding of transformer networks is an intriguing direction for future research.



\acks{The research of Y. Jiao is supported by the National Nature Science Foundation of China ( No.12371441) and supported by “the Fundamental Research Funds for the Central Universities” and by the research fund of KLATASDSMOE of China . The research of Y. Wang is supported by the HK RGC grant 16308518, the HK Innovation Technology Fund Grant ITS/044/18FX and the Guangdong-Hong Kong-Macao Joint Laboratory for Data Driven Fluid Dynamics and Engineering Applications (Project 2020B1212030001). We thank the
editor and reviewers for their feedback on our manuscript.}


\newpage
\appendix

\section{Approximation Error} \label{appendix: app}

In this section, we primarily follow the technical proof of the approximation capability of the transformer network by \citet{gurevych2022rate}. The fundamental idea is that, similar to ReLU networks, the transformer network can implement multiplication, and subsequently, polynomials. Polynomials can sufficiently approximate any continuous function, hence the constructed transformer network can also approximate continuous functions. In contrast to \citet{gurevych2022rate}, we need the network to maintain Lipschitz continuity. Therefore, we choose polynomials to approximate both the function and its derivative simultaneously. This ensures that the constructed transformer network and the target function not only have minimal differences in their functions, but also in their derivatives, which helps control the Lipschitz constant of the network.

\subsection{Input embedding}
We follow the construction in \citet{gurevych2022rate}. For $h\in\mathbb{N}$, we define $d_{model} = h \cdot (d_{patch}+l+4)$ and
\begin{align*}
\widetilde{A}_{in} =  
\begin{pmatrix}
\mathbb{I}_{d_{patch}}   & \mathbb{O}_{d_{patch},l}\\
\mathbb{O}_{1,d_{patch}}  & \mathbb{O}_{1,l}\\
\mathbb{O}_{l,d_{patch}}  & \mathbb{I}_l\\
\mathbb{O}_{3,d_{patch}}  & \mathbb{O}_{3,l}
\end{pmatrix}, \quad
\widetilde{\boldsymbol{b}}_{in} =  
\begin{pmatrix}
\mathbbm{0}_{d_{patch}} \\
1 \\
\mathbbm{0}_{l+1} \\
1 \\
0
\end{pmatrix}, \quad
A_{in} = 
\begin{pmatrix}
\widetilde{A}_{in} \\
\vdots \\
\widetilde{A}_{in}
\end{pmatrix}, \quad
\boldsymbol{b}_{in} = 
\begin{pmatrix}
\widetilde{\boldsymbol{b}}_{in} \\
\vdots \\
\widetilde{\boldsymbol{b}}_{in}
\end{pmatrix},
\end{align*}
where $\mathbb{O}_{m,n}$ denotes a zero matrix with $m$ rows and $n$ columns, and $\mathbbm{0}_{l}$ denotes a zero vector of length $l$, which satisfies $\|A_{in}\|_0 + \|\boldsymbol{b}_{in}\|_0 \leq h \cdot (d_{patch}+l+2)$. $Z_0$ defined in (\ref{eq: app 1}) then satisfies
\begin{align*}
z_{0, j}^{((k-1)(d_{patch}+l+4)+s)}= \begin{cases}
x_j^{(s)} & \text { if } s \in\{1, \ldots, d_{patch}\} \\ 
1 & \text { if } s=d_{patch}+1 \\ 
\delta_{s-d_{patch}-1, j} & \text { if } s \in\{d_{patch}+2, \ldots, d_{patch}+l+1\} \\ 
1 & \text { if } s=d_{patch}+l+3 \\ 
0 & \text { if } s \in\{d_{patch}+l+2, d_{patch}+l+4\}\end{cases}
\end{align*}
for $k \in \{1, \ldots, h\}, j \in \{1, \ldots, l\}$, where $x_j^{(s)}$ denotes the $s$-th component of the $j$-th token of $X$, or in other words, it is the element at the $s$-th row and $j$-th column of matrix $X$. The attention layer and feedforward layer iteratively compute the representations
\begin{align*}
Y_r = F_r^{(SA)}(Z_{r-1}), \quad Z_r = F_r^{(FF)}(Y_r)
\end{align*}
for $r=1, \ldots, N$.

\subsection{Approximation of polynomials with single-head attention}
In the case of a transformer network with single-head attention, the embedded input sequence is represented by
\begin{align*}
Z_0=(\boldsymbol{z}_{0,1}, \ldots, \boldsymbol{z}_{0, l}) \in \mathbb{R}^{d_{model} \times l},
\end{align*}
where $d_{model} = d_{patch}+l+4$ and
\begin{align}
\label{eq: app 2}
z_{0, j}^{(s)} = \begin{cases}
x_j^{(s)} & \text { if } s \in\{1, \ldots, d_{patch}\} \\ 
1 & \text { if } s=d_{patch}+1 \\ 
\delta_{s-d_{patch}-1, j} & \text { if } s \in\{d_{patch}+2, \ldots, d_{patch}+l+1\} \\ 
1 & \text { if } s=d_{patch}+l+3 \\ 
0 & \text { if } s \in\{d_{patch}+l+2, d_{patch}+l+4\}
\end{cases}.
\end{align}

The first lemma shows that single-head attention layer can be used to compute linear functions in one variable.

\begin{lemma}
\label{lemma: app 1}
Let $\boldsymbol{x}_j \in \mathbb{R}^{d_{patch}}$ and $b_j \in \mathbb{R}\, (j=1, \ldots, l)$. Let $Z \in \mathbb{R}^{d_{model} \times l}$ be given by
\begin{align*}
z_{j}^{(s)} = 
\begin{cases}
x_j^{(s)} & \text { if } s \in \{1, \ldots, d_{patch}\} \\ 
1 & \text { if } s=d_{patch}+1 \\ 
\delta_{s-d_{patch}-1, j} & \text { if } s \in\{d_{patch}+2, \ldots, d_{patch}+l+1\} \\ 
b_j & \text { if } s=d_{patch}+l+3 \\ 
0 & \text { if } s \in\{d_{patch}+l+2, d_{patch}+l+4\}
\end{cases}.
\end{align*}
Let $j \in \{1, \ldots, l\}, k \in \{1, \ldots, d_{patch}\}$ and $u \in \mathbb{R}$ be arbitrary. Let
\begin{align*}
B > 2 \cdot \max_{s=1, \ldots, d_{patch}, j=1, \ldots, l} |x_j^{(s)}|.
\end{align*}
Then there exist matrices $W_Q, W_K \in \mathbb{R}^{2 \times d_{model}}$ and $W_V, W_O \in \mathbb{R}^{d_{model} \times d_{model}}$, where 
\begin{align*}
\|W_Q\|_0 + \|W_K\|_0 + \|W_V\|_0 + \|W_O\|_0 \leq d_{model}+6,
\end{align*}
such that 
\begin{align*}
Y = F^{(SA)}(Z) \in \mathbb{R}^{d_{model} \times l}
\end{align*}
satisfies
\begin{align*}
y_1^{(s)}= \begin{cases}
x_1^{(s)} & \text { if } s \in\{1, \ldots, d_{patch}\} \\
1 & \text { if } s=d_{patch}+1 \\
\delta_{s-d_{patch}-1, j} & \text { if } s \in\{d_{patch}+2, \ldots, d_{patch}+l+1\} \\
x_j^{(k)}-u & \text { if } s=d_{patch}+l+2 \\
b_1 & \text { if } s=d_{patch}+l+3 \\
0 & \text { if } s=d_{patch}+l+4\end{cases}
\end{align*}
and
\begin{align*}
\boldsymbol{y}_j=\boldsymbol{z}_{j} \quad \text { for } j \in \{2, \ldots, l\}. 
\end{align*}
\end{lemma}

\begin{proof}
See also \citet[Lemma 1]{gurevych2022rate}. Let $W_O = \mathbb{I}_{d_{model}}$,
\begin{align*}
W_Q=\left(\begin{array}{ccccccc}
0 & \ldots & 0 & 1 & 0 & \ldots & 0 \\
0 & \ldots & 0 & B & 0 & \ldots & 0
\end{array}\right)
\end{align*}
where all columns are zero except for column number $d_{patch}+2$,
\begin{align*}
W_K=\left(\begin{array}{ccccccccccccccc}
0 & \ldots & 0 & 1 & 0 & \ldots & 0 & -u-B & 0 & \ldots & 0 & 0 & 0 & \ldots & 0 \\
0 & \ldots & 0 & 0 & 0 & \ldots & 0 & 0 & 0 & \ldots & 0 & 1 & 0 & \ldots & 0
\end{array}\right)
\end{align*}
where all columns are zero except for column number $k$, column number $d_{patch}+1$ and column number $d_{patch}+1+j$, and
\begin{align*}
W_V=\left(\begin{array}{ccccccc}
0 & \ldots & 0 & 0 & 0 & \ldots & 0 \\
\vdots &  & \vdots & \vdots & \vdots &  & \vdots \\
0 & \ldots & 0 & 0 & 0 & \ldots & 0 \\
0 & \ldots & 0 & 1 & 0 & \ldots & 0 \\
0 & \ldots & 0 & 0 & 0 & \ldots & 0 \\
\vdots &  & \vdots & \vdots & \vdots &  & \vdots \\
0 & \ldots & 0 & 0 & 0 & \ldots & 0
\end{array}\right)
\end{align*}
where all rows and all columns are zero except for row number $d_{patch}+l+2$ and column number $d_{patch}+1$. Direct calculations give that 
\begin{align*}
&W_{O}(W_{V} Z)  \left[\left((W_{K} Z)^{\top}(W_{Q} Z)\right) \odot \sigma_H\left((W_{K} Z)^{\top}(W_{Q} Z)\right)\right] 
= \left(\begin{array}{cccc}
0 & 0 & \ldots & 0 \\
\vdots & \vdots & & \vdots \\
0 & 0 & \ldots & 0 \\
x_j^{(k)}-u & 0 & \ldots & 0 \\
0 & 0 & \ldots & 0 \\
0 & 0 & \ldots & 0 \\
\end{array}\right), 
\end{align*}
where all rows and all columns are zero except for row number $d_{patch}+l+2$ and column number $1$, which completes the proof.
\end{proof}

The next lemma shows that single-head attention layer can be used to compute products.

\begin{lemma}
\label{lemma: app 2}
Let $\boldsymbol{x}_s \in \mathbb{R}^{d_{patch}}$ and $a_s, b_s \in \mathbb{R}\, (s=1, \ldots, l)$. Let $Z \in \mathbb{R}^{d_{model} \times l}$ be given by
\begin{align*}
z_{j}^{(s)}= \begin{cases}
x_j^{(s)} & \text { if } s \in\{1, \ldots, d_{patch}\} \\ 
1 & \text { if } s=d_{patch}+1 \\ 
\delta_{s-d_{patch}-1, j} & \text { if } s \in\{d_{patch}+2, \ldots, d_{patch}+l+1\} \\ 
a_j & \text { if } s=d_{patch}+l+2 \\ 
b_j & \text { if } s=d_{patch}+l+3 \\ 
0 & \text { if } s=d_{patch}+l+4\end{cases}.
\end{align*}
Let $j \in \{1, \ldots, l\}$. Let 
\begin{align*}
B>2|b_1| \cdot \max_{r=1, \ldots, l} |a_r|.
\end{align*}
Then there exist matrices $W_Q, W_K \in \mathbb{R}^{2 \times d_{model}}$ and $W_V, W_O \in \mathbb{R}^{d_{model} \times d_{model}}$, where
\begin{align*}
\|W_Q\|_0 + \|W_K\|_0 + \|W_V\|_0 + \|W_O\|_0 \leq d_{model}+5,
\end{align*}
such that 
\begin{align*}
Y = F^{(SA)}(Z) \in \mathbb{R}^{d_{model} \times l}
\end{align*}
satisfies
\begin{align*}
y_1^{(s)}= \begin{cases}x_1^{(s)} & \text { if } s \in\{1, \ldots, d_{patch}\} \\ 1 & \text { if } s=d_{patch}+1 \\ \delta_{s-d_{patch}-1, j} & \text { if } s \in\{d_{patch}+2, \ldots, d_{patch}+l+1\} \\ a_1 & \text { if } s=d_{patch}+l+2 \\ b_1 & \text { if } s=d_{patch}+l+3 \\ b_1 \cdot a_j+B & \text { if } s=d_{patch}+l+4\end{cases}
\end{align*}
and
\begin{align*}
y_s^{(i)} = z_{s}^{(i)} \quad \text { for } i \in\{1, \ldots, d_{patch}+l+3\}, s \in\{1, \ldots, l\}.
\end{align*}
\end{lemma}

\begin{proof}
See also \citet[Lemma 2]{gurevych2022rate}. Define $W_V$ as in the proof of Lemma \ref{lemma: app 1} such that all rows and all columns are zero except for row number $d_{patch}+l+4$ and column number $d_{patch}+1$. Let $W_O = \mathbb{I}_{d_{model}}$, 
\begin{align*}
W_Q = \left(\begin{array}{ccccccccccc}
0 & \ldots & 0 & 0 & 0 & \ldots & 0 & 1 & 0 & \ldots & 0 \\
0 & \ldots & 0 & B & 0 & \ldots & 0 & 0 & 0 & \ldots & 0
\end{array}\right)
\end{align*}
where all columns are zero except for column number $d_{patch}+2$ and column number $d_{patch}+l+3$, and
\begin{align*}
W_K = \left(\begin{array}{ccccccccccc}
0 & \ldots & 0 & 0 & 0 & \ldots & 0 & 1 & 0 & \ldots & 0 \\
0 & \ldots & 0 & 1 & 0 & \ldots & 0 & 0 & 0 & \ldots & 0
\end{array}\right)
\end{align*}
where all columns are zero except for column number $d_{patch}+1+j$ and column number $d_{patch}+l+2$. Direct calculations give that 
\begin{align*}
&W_{O}(W_{V} Z)  \left[\left((W_{K} Z)^{\top}(W_{Q} Z)\right) \odot \sigma_H\left((W_{K} Z)^{\top}(W_{Q} Z)\right)\right] 
= \left(\begin{array}{cccc}
0 & 0 & \ldots & 0 \\
\vdots & \vdots & & \vdots \\
0 & 0 & \ldots & 0 \\
b_1 \cdot a_j+B & * & \ldots & * \\
\end{array}\right), 
\end{align*}
where all rows are zero except for row number $d_{patch}+l+4$ and the symbol $*$ indicates the specific values are not of concern. 
\end{proof}

The following lemma illustrates that single-head attention layer can be used to compute squares.

\begin{lemma}
\label{lemma: app 9}
Let $a \in \mathbb{R}$, $\boldsymbol{x}_s \in \mathbb{R}^{d_{patch}}$ and $b_s \in \mathbb{R}\, (s=1, \ldots, l)$. Let $Z \in \mathbb{R}^{d_{model} \times l}$ be given by
\begin{align*}
z_{j}^{(s)}= \begin{cases}
x_j^{(s)} & \text { if } s \in\{1, \ldots, d_{patch}\} \\ 
1 & \text { if } s=d_{patch}+1 \\ 
\delta_{s-d_{patch}-1, j} & \text { if } s \in\{d_{patch}+2, \ldots, d_{patch}+l+1\} \\ 
a\cdot\delta_{j,1} & \text { if } s=d_{patch}+l+2 \\ 
b_j & \text { if } s=d_{patch}+l+3 \\ 
0 & \text { if } s=d_{patch}+l+4
\end{cases}.
\end{align*}
Then there exist matrices $W_Q, W_K \in \mathbb{R}^{1 \times d_{model}}$ and $W_V, W_O \in \mathbb{R}^{d_{model} \times d_{model}}$, where
\begin{align*}
\|W_Q\|_0 + \|W_K\|_0 + \|W_V\|_0 + \|W_O\|_0 \leq d_{model}+3,
\end{align*}
such that 
\begin{align*}
Y = F^{(SA)}(Z) \in \mathbb{R}^{d_{model} \times l}
\end{align*}
satisfies
\begin{align*}
y_1^{(s)}= \begin{cases}
x_1^{(s)} & \text { if } s \in\{1, \ldots, d_{patch}\} \\ 
1 & \text { if } s=d_{patch}+1 \\ 
\delta_{s-d_{patch}-1, j} & \text { if } s \in\{d_{patch}+2, \ldots, d_{patch}+l+1\} \\ 
a & \text { if } s=d_{patch}+l+2 \\ 
b_1 & \text { if } s=d_{patch}+l+3 \\ 
a^2 & \text { if } s=d_{patch}+l+4\end{cases}
\end{align*}
and
\begin{align*}
\boldsymbol{y}_j = \boldsymbol{z}_j \quad \text { for } j \in \{2, \ldots, l\}. 
\end{align*}
\end{lemma}

\begin{proof}
Define $W_V$ as in the proof of Lemma \ref{lemma: app 1} such that all rows and all columns are zero except for row number $d_{patch}+l+4$ and column number $d_{patch}+1$. Let $W_O = \mathbb{I}_{d_{model}}$, 
\begin{align*}
W_Q=\left(\begin{array}{cccccc}
0 & \ldots & 0 & 1 & 0 & 0 
\end{array}\right)
\end{align*}
where all columns are zero except for column number $d_{patch}+l+2$ and $W_K = W_Q$. Direct calculations give that 
\begin{align*}
&W_{O}(W_{V} Z)  \left[\left((W_{K} Z)^{\top}(W_{Q} Z)\right) \odot \sigma_H\left((W_{K} Z)^{\top}(W_{Q} Z)\right)\right] 
= \left(\begin{array}{cccc}
0 & 0 & \ldots & 0 \\
\vdots & \vdots & & \vdots \\
0 & 0 & \ldots & 0 \\
a^2 & 0 & \ldots & 0 \\
\end{array}\right), 
\end{align*}
where all rows and all columns are zero except for row number $d_{patch}+l+4$ and column number $1$, which completes the proof. 
\end{proof}

The following lemma introduces a special token-wise feedforward layer that applies the function $x \mapsto \alpha \cdot (x - B)$ to the $d+l+4$-th component of each token and writes the result into the $d+l+3$-th component. This layer is usually positioned after the attention layer defined in Lemma \ref{lemma: app 2}.

\begin{lemma}
\label{lemma: app 3}
Let $Y = (\boldsymbol{y}_1, \ldots, \boldsymbol{y}_l) \in \mathbb{R}^{d_{model} \times l}$. Let $d_{ff} \geq 8$ and let $\alpha, B \in \mathbb{R}$. Then there exist matrices and vectors
\begin{align*}
W_1 \in \mathbb{R}^{d_{f f} \times d_{model}}, \quad \boldsymbol{b}_1 \in \mathbb{R}^{d_{f f}}, \quad  W_2 \in \mathbb{R}^{d_{model} \times d_{f f}}, \quad \boldsymbol{b}_2 \in \mathbb{R}^{d_{model}},
\end{align*}
where
\begin{align*}
\|W_1\|_0 + \|\boldsymbol{b}_1\|_0 + \|W_2\|_0 + \|\boldsymbol{b}_2\|_0 \leq 18,
\end{align*}
such that
\begin{align*}
Z = F^{(FF)}(Y)
\end{align*}
satisfies
\begin{align*}
z_s^{(i)}= \begin{cases}
y_s^{(i)} & \text { if } i \in\{1, \ldots, d_{patch}+l+1\} \\ 
\alpha \cdot\left(y_s^{(d_{patch}+l+4)}-By_s^{(d_{patch}+1)}\right) & \text { if } i=d_{patch}+l+3 \\ 
0 & \text { if } i \in\{d_{patch}+l+2, d_{patch}+l+4\}\end{cases}
\end{align*}
for $s \in \{1, \ldots, l\}$.
\end{lemma}

\begin{proof}
See also \citet[Lemma 3]{gurevych2022rate}. We choose $\boldsymbol{b}_1=\mathbbm{0}_{d_{ff}}, \boldsymbol{b}_2=\mathbbm{0}_{d_{model}}$,
\begin{align*}
W_1=\left(\begin{array}{cccccccccc}
0 & \ldots & 0 & 0 & 0 & \ldots & 0 & 1 & 0 & 0 \\
0 & \ldots & 0 & 0 & 0 & \ldots & 0 & -1 & 0 & 0 \\
0 & \ldots & 0 & 0 & 0 & \ldots & 0 & 0 & 1 & 0 \\
0 & \ldots & 0 & 0 & 0 & \ldots & 0 & 0 & -1 & 0 \\
0 & \ldots & 0 & -B & 0 & \ldots & 0 & 0 & 0 & 1 \\
0 & \ldots & 0 & B & 0 & \ldots & 0 & 0 & 0 & -1 \\
0 & \ldots & 0 & 0 & 0 & \ldots & 0 & 0 & 0 & 1 \\
0 & \ldots & 0 & 0 & 0 & \ldots & 0 & 0 & 0 & -1
\end{array}\right)
\end{align*}
where all columns except column number $d_{patch}+1, d_{patch}+l+2, d_{patch}+l+3$ and $d_{patch}+l+4$ are zero, and
\begin{align*}
W_2=\left(\begin{array}{cccccccc}
0 & 0 & 0 & 0 & 0 & 0 & 0 & 0 \\
\vdots & \vdots & \vdots & \vdots & \vdots & \vdots & \vdots & \vdots \\
0 & 0 & 0 & 0 & 0 & 0 & 0 & 0 \\
-1 & 1 & 0 & 0 & 0 & 0 & 0 & 0 \\
0 & 0 & -1 & 1 & \alpha & -\alpha & 0 & 0 \\
0 & 0 & 0 & 0 & 0 & 0 & -1 & 1
\end{array}\right)
\end{align*}
where all rows except row number $d_{patch}+l+2, d_{patch}+l+3$ and $d_{patch}+l+4$ are zero. Then we have
\begin{align*}
W_2 \cdot \sigma (W_1 \cdot \boldsymbol{y}_s + \boldsymbol{b}_1) + \boldsymbol{b}_2 = \left(\begin{array}{c}
0 \\
\vdots \\
0 \\
-\sigma\left(y_s^{(d_{patch}+l+2)}\right)+\sigma\left(-y_s^{(d_{patch}+l+2)}\right) \\
A \\
-\sigma\left(y_s^{(d_{patch}+l+4)}\right)+\sigma\left(-y_s^{(d_{patch}+l+4)}\right)
\end{array}\right),
\end{align*}
where
\begin{align*}
A = &-\sigma\left(y_s^{(d_{patch}+l+3)}\right)+\sigma\left(-y_s^{(d_{patch}+l+3)}\right) \\
&+\alpha \cdot \sigma\left(y_s^{(d_{patch}+l+4)}-By_s^{(d_{patch}+1)}\right) -\alpha \cdot \sigma\left(By_s^{(d_{patch}+1)}-y_s^{(d_{patch}+l+4)}\right).
\end{align*}
Using the fact $\sigma(u)-\sigma(-u)=u$ for any $u \in \mathbb{R}$ completes the proof.
\end{proof}

\begin{remark}
\label{remark: app 1}
It follows from the proof of Lemma \ref{lemma: app 3} that we can modify $W_1$ and $W_2$ such that 
\begin{align*}
z_s^{(i)}= \begin{cases}
y_s^{(i)} & \text { if } i \in\{1, \ldots, d_{patch}+l+1, d_{patch}+l+2\} \\ 
\alpha \cdot\left(y_s^{(d_{patch}+l+4)}-By_s^{(d_{patch}+1)}\right) & \text { if } i=d_{patch}+l+3 \\ 
0 & \text { if } i =d_{patch}+l+4\end{cases}
\end{align*}
for $s \in\{1, \ldots, l\}$.
\end{remark}

\begin{remark}
\label{remark: app 2}
It follows from the proof of Lemma \ref{lemma: app 3} that we can modify $W_1$ and $W_2$ such that 
\begin{align*}
z_s^{(i)}= \begin{cases}
y_s^{(i)} & \text { if } i \in\{1, \ldots, d_{patch}+l+1, d_{patch}+l+3\} \\ 
y_s^{(d_{patch}+l+4)} & \text { if } i=d_{patch}+l+2 \\ 
0 & \text { if } i=d_{patch}+l+4 \end{cases}
\end{align*}
for $s \in\{1, \ldots, l\}$.
\end{remark}

\begin{remark}
\label{remark: app 3}
It follows from the proof of Lemma \ref{lemma: app 3} that we can modify $W_1$ and $W_2$ such that 
\begin{align*}
z_s^{(i)}= \begin{cases}
y_s^{(i)} & \text { if } i \in\{1, \ldots, d_{patch}+l+1, d_{patch}+l+3\} \\ 
0 & \text { if } i \in\{d_{patch}+l+2, d_{patch}+l+4\} \end{cases}
\end{align*}
for $s \in\{1, \ldots, l\}$.
\end{remark}

The following lemma shows that transformer networks with single-head attention can be used to compute monomials. 

\begin{lemma}
\label{lemma: app 10}
Assume $X \in [0,1]^{d_{patch} \times l}$. Let $M_\varepsilon \geq 2$. For any multi-index $n \in \mathbb{N}^{d_{patch} \times l}$ with $\|n\|_1 \leq M_\varepsilon$, we define
\begin{align}
\label{eq: eta x}
\eta_{n}(X) = \prod_{k=1}^l \prod_{s=1}^{d_{patch}} \left(x_k^{(s)}\right)^{n_{s,k}}.
\end{align}
Then there exists a transformer network consisting of $N$ pairs of layers, where in each pair the first layer is a single-head attention layer and the second layer is a token-wise feedforward layer, and 
\begin{align*}
\begin{gathered}
N \leq 2\, l\, d_{patch}\, (\log_2 M_\varepsilon + 1), \\
\begin{aligned}
&\sum_{r=1}^N \left(\|W_{Q, r}\|_0 + \|W_{K, r}\|_0 + \|W_{V, r}\|_0 +\|W_{O, r}\|_0\right) + \sum_{r=1}^N  \left(\|W_{r, 1}\|_0 + \|\boldsymbol{b}_{r, 1}\|_0 + \|W_{r, 2}\|_0 + \|\boldsymbol{b}_{r, 2}\|_0\right) \\
&\hspace{9.5cm} \leq (d_{patch}+l+28)\, N,
\end{aligned} 
\end{gathered}
\end{align*}
which takes $Z_0$ as the input  defined in (\ref{eq: app 2}) and generates $Z_N$ as the output, where
\begin{align*}
z_{N,1}^{(d_{patch}+l+3)} = \eta_{n} (X).
\end{align*}
\end{lemma}

\begin{proof}
For each $n_{s,k}$, we consider its binary representation $n_{s,k} = \sum_{p=0}^P a_{s,k,p} 2^p$, where $a_{s,k,p}\in \{0, 1\}$ and $P\leq\log_2 (M_\varepsilon)$. Let 
\begin{align*}
Z_0 = 
\begin{pmatrix}
\boldsymbol{x}_1  & \boldsymbol{x}_2 & \cdots & \boldsymbol{x}_l\\
1  & 1 & \cdots & 1\\
\boldsymbol{e}_1  & \boldsymbol{e}_2 & \cdots & \boldsymbol{e}_l \\
0  & 0 & \cdots & 0 \\
b_1  & * & \cdots & * \\
0  & 0 & \cdots & 0
\end{pmatrix},
\end{align*}
then application of Lemma \ref{lemma: app 1} yields
\begin{align*}
Y_1 = 
\begin{pmatrix}
\vdots  & \vdots &  & \vdots\\
x_k^{(s)}  & 0 & \cdots & 0\\
b_1  & * & \cdots & *\\
0  & 0 & \cdots & 0
\end{pmatrix},
\end{align*}
where the first $d_{patch}+l+1$ rows of $Y_1$ are identical to $Z_0$ and we focus on the last three rows. We set all components of $W_1, \boldsymbol{b}_1, W_2$ and $\boldsymbol{b}_2$ to zero and the feedforward layer yields
\begin{align*}
Z_1 = 
\begin{pmatrix}
\vdots  & \vdots &  & \vdots\\
x_k^{(s)}  & 0 & \cdots & 0\\
b_1  & * & \cdots & *\\
0  & 0 & \cdots & 0
\end{pmatrix}.
\end{align*}
By Lemma \ref{lemma: app 2}, we obtain
\begin{align*}
Y_2 = 
\begin{pmatrix}
\vdots  & \vdots &  & \vdots\\
x_k^{(s)}  & 0 & \cdots & 0\\
b_1  & * & \cdots & *\\
b_1\, x_k^{(s)}+B  & * & \cdots & *
\end{pmatrix}.
\end{align*}
By Remark \ref{remark: app 1}, we obtain 
\begin{align*}
Z_2 = 
\begin{pmatrix}
\vdots  & \vdots &  & \vdots\\
x_k^{(s)}  & 0 & \cdots & 0\\
b_1\, x_k^{(s)}  & * & \cdots & *\\
0  & 0 & \cdots & 0
\end{pmatrix}.
\end{align*}

Let $r \in \{1, \ldots, P\}$. Assume that we already have  
\begin{align*}
Z_m = 
\begin{pmatrix}
\vdots  & \vdots &  & \vdots\\
\left(x_k^{(s)}\right)^{2^{r-1}}  & 0 & \cdots & 0\\
b_1\cdot \left(x_k^{(s)}\right)^{\sum_{p=0}^{r-1} a_{s,k,p} 2^p}  & * & \cdots & *\\
0  & 0 & \cdots & 0
\end{pmatrix}
\end{align*}
for some $m \in \mathbb{N}$. Then we apply Lemma \ref{lemma: app 9} to obtain
\begin{align*}
Y_{m+1} = 
\begin{pmatrix}
\vdots  & \vdots &  & \vdots\\
\left(x_k^{(s)}\right)^{2^{r-1}}  & 0 & \cdots & 0\\
b_1\cdot \left(x_k^{(s)}\right)^{\sum_{p=0}^{r-1} a_{s,k,p} 2^p}  & * & \cdots & *\\
\left(x_k^{(s)}\right)^{2^{r}}  & 0 & \cdots & 0
\end{pmatrix}.
\end{align*}
Using Remark \ref{remark: app 2}, we have 
\begin{align*}
Z_{m+1} = 
\begin{pmatrix}
\vdots  & \vdots &  & \vdots\\
\left(x_k^{(s)}\right)^{2^{r}}  & 0 & \cdots & 0\\
b_1\cdot \left(x_k^{(s)}\right)^{\sum_{p=0}^{r-1} a_{s,k,p} 2^p}  & * & \cdots & *\\
0 & 0 & \cdots & 0
\end{pmatrix}.
\end{align*}
By consecutively employing Lemma \ref{lemma: app 2} and Remark \ref{remark: app 1}, we obtain
\begin{align*}
Y_{m+2} = 
\begin{pmatrix}
\vdots  & \vdots &  & \vdots\\
\left(x_k^{(s)}\right)^{2^{r}}  & 0 & \cdots & 0\\
b_1\cdot \left(x_k^{(s)}\right)^{\sum_{p=0}^{r-1} a_{s,k,p} 2^p}  & * & \cdots & *\\
b_1\cdot\left(x_k^{(s)}\right)^{\sum_{p=0}^{r-1} a_{s,k,p} 2^p  + 2^{r}} + B & 0 & \cdots & 0
\end{pmatrix}
\end{align*}
and
\begin{align*}
Z_{m+2} = 
\begin{pmatrix}
\vdots  & \vdots &  & \vdots\\
\left(x_k^{(s)}\right)^{2^{r}}  & 0 & \cdots & 0\\
b_1\cdot\left(x_k^{(s)}\right)^{\sum_{p=0}^{r-1} a_{s,k,p} 2^p  + 2^{r}}  & * & \cdots & *\\
0 & 0 & \cdots & 0
\end{pmatrix}.
\end{align*}
If $a_{s,k,r}=0$, $Z_{m+1}$ meets the next induction hypothesis. Otherwise, if $a_{s,k,r}=1$, $Z_{m+2}$ becomes the next induction hypothesis.

In the last step, we make the following modifications: if $a_{s,k,P}=0$, we substitute Remark \ref{remark: app 2} with Remark \ref{remark: app 3}; if $a_{s,k,P}=1$, we replace Remark \ref{remark: app 1} with Lemma \ref{lemma: app 3}. We combine a total of $\sum_{p=0}^{P} (a_{s,k,p}+1)$ pairs of attention and feedforward layers to achieve
\begin{align*}
Z_{\sum_{p=0}^{P} (a_{s,k,p}+1)} = 
\begin{pmatrix}
\vdots  & \vdots &  & \vdots\\
0 & 0 & \cdots & 0\\
b_1\cdot\left(x_k^{(s)}\right)^{\sum_{p=0}^{P} a_{s,k,p} 2^p}  & * & \cdots & *\\
0 & 0 & \cdots & 0
\end{pmatrix}
=
\begin{pmatrix}
\vdots  & \vdots &  & \vdots\\
0 & 0 & \cdots & 0\\
b_1\cdot\left(x_k^{(s)}\right)^{n_{s,k}}  & * & \cdots & *\\
0 & 0 & \cdots & 0
\end{pmatrix}.
\end{align*}
Iterating over all components of $n$, we obtain a transformer network, which takes $Z_0$ as the input, as defined in (\ref{eq: app 2}), and generates $Z_N$ as the output, where
\begin{align*}
Z_{N,1}^{(d_{patch}+l+3)} = \eta_{n}(x)
\end{align*}
and 
\begin{align*}
N = \sum_{k=1}^{l} \sum_{s=1}^{d_{patch}} \sum_{p=0}^{P}\, (a_{s,k,p}+1) \leq \sum_{k=1}^{l} \sum_{s=1}^{d_{patch}} \sum_{p=0}^{P} 2 \leq 2\, l\, d_{patch} \, (\log_2 M_\varepsilon+1).
\end{align*}
\end{proof}

\subsection{Approximation of polynomials with multi-head attention} \label{appendix: app.3}

In this subsection, we generalize the results from the previous subsection to transformer networks with multi-head attention. The basic idea is to use each head to compute a monomial as in Lemma \ref{lemma: app 10}, and employ a linear combination of these monomials to approximate arbitrary function.

As in input embedding, we represent the input sequence by
\begin{align*}
Z_0 = (\boldsymbol{z}_{0,1}, \ldots, \boldsymbol{z}_{0, l}) \in \mathbb{R}^{d_{model}\times l}
\end{align*}
where $d_{model} = h \cdot (d_{patch}+l+4)$ and
\begin{align}
\label{eq: app 3}
z_{0, j}^{((k-1) \cdot(d_{patch}+l+4)+s)}= \begin{cases}x_j^{(s)} & \text { if } s \in\{1, \ldots, d_{patch}\} \\ 1 & \text { if } s=d_{patch}+1 \\ \delta_{s-d_{patch}-1, j} & \text { if } s \in\{d_{patch}+2, \ldots, d_{patch}+l+1\} \\ 1 & \text { if } s=d_{patch}+l+3 \\ 0 & \text { if } s \in\{d_{patch}+l+2, d_{patch}+l+4\}\end{cases}
\end{align}
for $k \in \{1, \ldots, h\},\, j \in \{1, \ldots, l\}$.

\begin{lemma}
\label{lemma: app 8}
Let $h \in \mathbb{N}$ and let $s: \{ n\in\mathbb{N}^{d_{patch} \times l} : \|n\|_1 \leq M_\varepsilon \} \rightarrow \{1, \ldots, h\}, n\mapsto s(n)$ be an injection. Then there exists a transformer network consisting of $N$ pairs of layers, where in each pair the first layer is a multi-head attention layer with $h$ heads and the second layer is a token-wise feedforward
layer, and
\begin{align*}
\begin{gathered}
N \leq 2 \, l \, d_{patch} \, \left(\log _2 M_{\varepsilon}+1\right), \quad d_k = 2, \quad d_v = d_{patch}+l+4, \\
h \leq 
\left(\begin{array}{c}
l\cdot d_{patch}+M_\varepsilon \\
l\cdot d_{patch}
\end{array}\right), \quad 
d_{ff} = 8h, \\
\begin{aligned}
& \sum_{r=1}^N \sum_{s=1}^h   \left(\left\|W_{Q, r, s}\right\|_0+ \left\|W_{K, r, s}\right\|_0+\left\|W_{V, r, s}\right\|_0+ 
\left\|W_{O, r, s}\right\|_0\right) \\
&+\sum_{r=1}^N \left(\left\|W_{r, 1}\right\|_0+\left\|b_{r, 1}\right\|_0+\left\|W_{r, 2}\right\|_0+\left\|b_{r, 2}\right\|_0\right) \leq \left(d_{patch}+l+28\right) \cdot N \cdot h,
\end{aligned} 
\end{gathered}
\end{align*}
which gets as input $Z_0$ defined in (\ref{eq: app 3}) and produces as output $Z_N$ which satisfies
\begin{align*}
z_{N, 1}^{((s-1) \cdot(d_{patch}+l+4)+(d_{patch}+l+3))} = \eta_{n}(X).
\end{align*}
\end{lemma}

\begin{proof}
The result is a straightforward extension of the proof of Lemma \ref{lemma: app 10}. To begin with, we choose 
\begin{align*}
\begin{gathered}
W_{O,s} = 
\begin{pmatrix}
\mathbb{O} \\
\vdots \\
\mathbb{O} \\
\widetilde{W}_{O,s}  \\
\mathbb{O} \\
\vdots \\
\mathbb{O} \\
\end{pmatrix}, \quad
W_{K,s} = \left(\mathbb{O}, \ldots, \mathbb{O}, \widetilde{W}_{K,s},\mathbb{O}, \ldots, \mathbb{O} \right),  \\
W_{Q,s} = \left(\mathbb{O}, \ldots, \mathbb{O}, \widetilde{W}_{Q,s},\mathbb{O}, \ldots, \mathbb{O} \right),  \quad
W_{V,s} = \left(\mathbb{O}, \ldots, \mathbb{O}, \widetilde{W}_{V,s},\mathbb{O}, \ldots, \mathbb{O} \right),  
\end{gathered}
\end{align*}
where in each matrix all blocks are designated as zero matrices except for the $s$-th block, and $\widetilde{W}_{K,s}, \widetilde{W}_{Q,s} \in \mathbb{R}^{d_k\times (d_{patch}+l+4)}, \widetilde{W}_{V,s}, \widetilde{W}_{O,s} \in \mathbb{R}^{(d_{patch}+l+4)\times (d_{patch}+l+4)}$. We format the input sequence $Z$ into identical blocks
\begin{align*}
Z = \begin{pmatrix}
Z_1 \\
Z_2 \\
\vdots \\
Z_h
\end{pmatrix}
\end{align*}
with each $Z_s \in \mathbb{R}^{(d_{patch}+l+4)\times l}$ for $s \in \{ 1, \ldots, h\}$. The application of a multi-head attention layer to \(Z\) can be interpreted as the individual application of a single-head attention layer to each \(Z_s\), that is, 
\begin{align*}
F^{(SA)}(Z) = 
\begin{pmatrix}
\widetilde{F}_1^{(SA)}(Z_1)\\
\vdots \\
\widetilde{F}_h^{(SA)}(Z_h)
\end{pmatrix},
\end{align*}
where $\widetilde{F}_s^{(SA)}(Z_s)$ is a single-head attention layer applied only to $Z_s$. Similarly, we define
\begin{align*}
W_1 = 
\begin{pmatrix}
\widetilde{W}_{1,1}  & \mathbb{O} & \ldots & \mathbb{O} \\
\mathbb{O} & \widetilde{W}_{1,2} & \ldots & \mathbb{O} \\
\vdots  & \vdots & \ddots  & \vdots \\
\mathbb{O} & \mathbb{O} & \ldots & \widetilde{W}_{1,h}
\end{pmatrix}, \quad
\boldsymbol{b}_1 = 
\begin{pmatrix}
\widetilde{\boldsymbol{b}}_{1,1} \\
\widetilde{\boldsymbol{b}}_{1,2} \\
\vdots \\
\widetilde{\boldsymbol{b}}_{1,h} \\
\end{pmatrix}, \\
W_2 = 
\begin{pmatrix}
\widetilde{W}_{2,1}  & \mathbb{O} & \ldots & \mathbb{O} \\
\mathbb{O} & \widetilde{W}_{2,2} & \ldots & \mathbb{O} \\
\vdots  & \vdots & \ddots  & \vdots \\
\mathbb{O} & \mathbb{O} & \ldots & \widetilde{W}_{2,h}
\end{pmatrix}, \quad
\boldsymbol{b}_2 = 
\begin{pmatrix}
\widetilde{\boldsymbol{b}}_{2,1} \\
\widetilde{\boldsymbol{b}}_{2,2} \\
\vdots \\
\widetilde{\boldsymbol{b}}_{2,h} \\
\end{pmatrix}, 
\end{align*}
where $\widetilde{W}_{1,s}\in \mathbb{R}^{8 \times (d_{patch}+l+4)}, \widetilde{\boldsymbol{b}}_{1,s}\in \mathbb{R}^{8}, \widetilde{W}_{2,s}\in \mathbb{R}^{(d_{patch}+l+4) \times 8}, \widetilde{\boldsymbol{b}}_{1,s}\in \mathbb{R}^{d_{patch}+l+4}$ for $s \in \{ 1, \ldots, h\}$, and let 
\begin{align*}
Y = \begin{pmatrix}
Y_1 \\
\vdots \\
Y_h
\end{pmatrix}
\end{align*}
with each $Y_s \in \mathbb{R}^{(d_{patch}+l+4)\times l}$. Then 
\begin{align*}
F^{(FF)}(Y) = 
\begin{pmatrix}
\widetilde{F}_1^{(FF)}(Y_1)\\
\vdots \\
\widetilde{F}_h^{(FF)}(Y_h)
\end{pmatrix},
\end{align*}
where $\widetilde{F}_s^{(FF)}(Y_s)$ is a token-wise feedforward layer that operates only on $Y_s$. Utilizing the method in Lemma \ref{lemma: app 10}, we compute  $\eta_{n}(X)$ for each block independently.
\end{proof}

\begin{proof}[Proof of Theorem \ref{theorem: app 3}]
We first consider the case where $f$ is a real-valued function. Let $\beta>0$ and $f \in \mathcal{H}_{d,1}^\beta([0, 1]^{d_{patch}\times l}, K)$. According to Lemma \ref{lemma: app simultaneous approximation}, for each $M_\varepsilon \in \mathbb{N}$, there exists a polynomial $P_{M_\varepsilon}^\beta f$ of degree at most $M_\varepsilon$ such that for all $X \in [0,1]^{d_{patch} \times l}$ and each multi-index $\boldsymbol{\alpha}$ with $\|\boldsymbol{\alpha}\|_1 \leq \min \{\lfloor\beta\rfloor, M_\varepsilon\}$ we have
\begin{align*}
\left|\partial^{\boldsymbol{\alpha}} \left(f(X)-P_{M_\varepsilon}^\beta f(X)\right)\right| \leq \frac{c_1 K}{M_\varepsilon^{\beta-\|\boldsymbol{\alpha}\|_1}},
\end{align*}
where the constant $c_1$ is independent of $M_\varepsilon$ and $K$. Since
\begin{align*}
P_{M_\varepsilon}^\beta f(X) = \sum_{\|n\|_1 \leq M_\varepsilon} a_n \eta_{n}(X),
\end{align*}
where $n \in \mathbb{N}^{d_{patch}\times l}$, $a_n \in \mathbb{R}$, and $\eta_{n}(X)$ is a monomial defined in (\ref{eq: eta x}) for each $n$, we can implement this linear combination using the output embedding and Lemma \ref{lemma: app 8}. Namely, we define $b_{out}=0 \in \mathbb{R}$ and $A_{out} \in \mathbb{R}^{1 \times d_{model}}$ a zero matrix except for the $((s-1) \cdot(d_{patch}+l+4)+(d_{patch}+l+3))$-th position where it takes the value of $a_n$ for all $n$ with $\|n\|_1 \leq M_\varepsilon$, then
\begin{align*}
\|A_{out}\|_0 + \|b_{out}\|_0 \leq h
\end{align*}
and 
\begin{align*}
A_{out} \cdot \boldsymbol{z}_{N,1} + b_{out} = P_{M_\varepsilon}^\beta f(X),
\end{align*}
where $Z_{N} = \{\boldsymbol{z}_{N,1}, \ldots, \boldsymbol{z}_{N,l}\}$ is taken from Lemma \ref{lemma: app 8}.

We then consider the case where $\boldsymbol{f}$ is a $\mathbb{R}^{d^\prime}$-valued function. For function $\boldsymbol{f} = (f_1, \ldots, f_{d^\prime})^\top \in \mathcal{H}_{d,d^\prime}^\beta ([0, 1]^{d}, K)$, each component of $\boldsymbol{f}$ is a real-valued function. As previously mentioned, we have $d = l \cdot d_{patch}$, and we do not differentiate between functions with respect to either $\boldsymbol{x}$ or $X$ as independent variables. Define
\begin{align}
\label{eq: app 9}
\widetilde{P}_{M_\varepsilon}^\beta \boldsymbol{f}(\boldsymbol{x}) = 
\begin{pmatrix}
P_{M_\varepsilon}^\beta f_1(\boldsymbol{x}) \\
P_{M_\varepsilon}^\beta f_2(\boldsymbol{x}) \\
\vdots \\
P_{M_\varepsilon}^\beta f_{d^\prime}(\boldsymbol{x}) \\
\end{pmatrix},
\end{align}
then for all $\boldsymbol{x} \in [0,1]^{d}$ we have
\begin{align*}
\left\| \boldsymbol{f}(\boldsymbol{x}) - \widetilde{P}_{M_\varepsilon}^\beta \boldsymbol{f}(\boldsymbol{x}) \right\|_{L^\infty ([0,1]^d)} \leq \frac{c_2 K}{M_\varepsilon^{\beta}}
\end{align*}
and there exist $A_{out} \in \mathbb{R}^{d^\prime \times d_{model}}$ and $\boldsymbol{b}_{out} \in  \mathbb{R}^{d^\prime}$ such that
\begin{align*}
\|A_{out}\|_0+\|\boldsymbol{b}_{out}\|_0 \leq d^\prime \cdot h
\end{align*}
and 
\begin{align*}
A_{out}\cdot \boldsymbol{z}_{N,1} + \boldsymbol{b}_{out} = \widetilde{P}_{M_\varepsilon}^\beta \boldsymbol{f}(\boldsymbol{x}),
\end{align*}
where $Z_{N}$ is taken from Lemma \ref{lemma: app 8}.

We put things together. Let
\begin{align}
M_\varepsilon = 
\left\lceil \left( \frac{c_2 K}{\varepsilon} \right)^{1/\beta} \right\rceil,
\end{align}
then 
\begin{align*}
\left\| \boldsymbol{f}(\boldsymbol{x}) - \widetilde{P}_{M_\varepsilon}^\beta \boldsymbol{f}(\boldsymbol{x}) \right\|_{L^\infty ([0,1]^d)} \leq \varepsilon.
\end{align*}
Combining input embedding, output embedding and Lemma \ref{lemma: app 8}, there exists a transformer network $\boldsymbol{\phi}$ consisting of $N$ pairs of layers, where 
\begin{align*}
\begin{gathered}
N \leq 2 \, l \, d_{patch} \, \left(\log _2 M_{\varepsilon}+1\right) \leq 4 d \log_2 \left( \left\lceil \left( \frac{c_2 K}{\varepsilon} \right)^{1/\beta} \right\rceil \right) = \mathcal{O} \left( \log \left(\frac{K}{\varepsilon}\right)\right), \\
h \leq \left(\begin{array}{c}
d+\left\lceil \left( \frac{c_2 K}{\varepsilon} \right)^{1/\beta} \right\rceil \\
d
\end{array}\right)
= \mathcal{O} \left(  \left( \frac{K}{\varepsilon} \right)^{d/\beta} \right), \quad
d_{ff} = 8 h, \\
d_k = 2 = \mathcal{O}(1), \quad d_v = d_{patch}+l+4 = \mathcal{O}(1), \\
\begin{aligned}
J &= \sum_{r=1}^N \sum_{s=1}^h \left(\left\|W_{Q, r, s}\right\|_0+ \left\|W_{K, r, s}\right\|_0+\left\|W_{V, r, s}\right\|_0+ 
\left\|W_{O, r, s}\right\|_0\right) \\
&~~~ +\sum_{r=1}^N \left(\left\|W_{r, 1}\right\|_0+\left\|b_{r, 1}\right\|_0+\left\|W_{r, 2}\right\|_0+\left\|b_{r, 2}\right\|_0\right) \\
&~~~ + \|A_{in}\|_0 + \|b_{in}\|_0 + \|A_{out}\|_0 + \|b_{out}\|_0 \\
&\leq \left(d_{patch}+l+28\right) \cdot N \cdot h + \left(d_{patch}+l+2\right) \cdot h + d^\prime \cdot h  \\
&\leq \left(8d^2+114d+2+d^\prime\right) \cdot \log_2 \left( \left\lceil \left( \frac{c_2 K}{\varepsilon} \right)^{1/\beta} \right\rceil \right) \cdot
\left(\begin{array}{c}
d+\left\lceil \left( \frac{c_2 K}{\varepsilon} \right)^{1/\beta} \right\rceil \\
d
\end{array}\right) \\
&= \mathcal{O} \left(\left( \frac{K}{\varepsilon} \right)^{d/\beta} \log \left(\frac{K}{\varepsilon}\right) \right)
\end{aligned}
\end{gathered}
\end{align*}
such that
\begin{align}
\boldsymbol{\phi}(\boldsymbol{x}) = \widetilde{P}_{M_\varepsilon}^\beta \boldsymbol{f}(\boldsymbol{x})
\end{align}
for all $\boldsymbol{x} \in [0, 1]^d$. Since
\begin{align*}
\|\boldsymbol{\phi}(\boldsymbol{x})\| \leq \|\boldsymbol{f}(\boldsymbol{x})\| + \|\boldsymbol{\phi}(\boldsymbol{x}) - \boldsymbol{f}(\boldsymbol{x})\| \leq \sup_{\boldsymbol{x} \in [0, 1]^d} \|\boldsymbol{f}(\boldsymbol{x})\| + \varepsilon,
\end{align*}
we may choose $B = \mathcal{O} (\|\boldsymbol{f}\|_{L^{\infty}([0, 1]^d)})$.

Additionally, we examine the Lipschitz constant for $\boldsymbol{\phi}(\boldsymbol{x})$. When 
$\beta>1$, since 
\begin{align*}
\left|\frac{\partial}{\partial x_k^{(s)}}\left(P_{M_\varepsilon}^\beta f(\boldsymbol{x})\right)\right| &\leq \left|\frac{\partial}{\partial x_k^{(s)}}\left(f(\boldsymbol{x})\right)\right| + \left|\frac{\partial}{\partial x_k^{(s)}}\left(P_{M_\varepsilon}^\beta f(\boldsymbol{x}) - f(\boldsymbol{x})\right)\right|\\
&\leq K + \frac{c_2 K}{M_\varepsilon^{\beta-1}} \\
&\leq (1 + c_2) K,
\end{align*}
mean value theorem yields
\begin{align*}
\left|P_{M_\varepsilon}^\beta f(\boldsymbol{x}) - P_{M_\varepsilon}^\beta f(\boldsymbol{y})\right| \leq d (1 + c_2) K \|\boldsymbol{x}-\boldsymbol{y}\|
\end{align*}
for all $\boldsymbol{x}, \boldsymbol{y} \in [0, 1]^d$. Thus, by the definition of $\widetilde{P}_{M_\varepsilon}^\beta \boldsymbol{f}(\boldsymbol{x})$ in (\ref{eq: app 9}), we obtain
\begin{align*}
\left\| \boldsymbol{\phi}(\boldsymbol{x}) -  \boldsymbol{\phi}(\boldsymbol{y}) \right\| & = \left\| \widetilde{P}_{M_\varepsilon}^\beta \boldsymbol{f}(\boldsymbol{x}) - \widetilde{P}_{M_\varepsilon}^\beta \boldsymbol{f}(\boldsymbol{y}) \right\| \\
&\leq c_3 K \|\boldsymbol{x}-\boldsymbol{y}\|,
\end{align*}
which indicates the transformer network $\boldsymbol{\phi}$ is $\mathcal{O}(K)$-Lipschitz when $\beta>1$.
\end{proof}

\begin{proof}[Proof of Theorem \ref{corollary: app 1}]
We only need to modify the proof in Theorem \ref{theorem: app 3} by replacing Lemma \ref{lemma: app simultaneous approximation} with Lemma \ref{lemma: app simultaneous approximation 3}.
\end{proof}

\begin{lemma}
\label{lemma: app simultaneous approximation}
For each $f \in \mathcal{H}_{d,1}^\beta([0, 1]^{d}, K)$ with $\beta>0$ and positive integer $N$, there is a polynomial $p_N$ of degree at most $N$ on $\mathbb{R}^d$, such that for each multi-index $\boldsymbol{\alpha}$ with $\|\boldsymbol{\alpha}\|_1 \leq \min \{\lfloor\beta\rfloor, N\}$ we have
\begin{align*}
\sup_{[0, 1]^{d}} \left|\partial^{\boldsymbol{\alpha}}\left(f-p_N\right)\right| \leq \frac{c_4 K}{N^{\beta-\|\boldsymbol{\alpha}\|_1}},
\end{align*}
where $c_4$ is a positive constant depending only on $d, \beta$.
\end{lemma}

\begin{proof}
Let $f$ be a function of compact support on $\mathbb{R}^d$, of class $\mathcal{C}^m$ where $0 \leq$ $m<\infty$, and let $K$ be a compact subset of $\mathbb{R}^d$ which contains the support of $f$. Then \citet[Theorem 1]{bagby2002multivariate} gives that for each positive integer $N$ there is a polynomial $p_N$ of degree at most $N$ on $\mathbb{R}^d$ with the following property: for each multi-index $\boldsymbol{\alpha}$ with $\|\boldsymbol{\alpha}\|_1 \leq \min \{m, N\}$ we have
\begin{align*}
\sup_K \left|\partial^{\boldsymbol{\alpha}} \left(f-p_N\right)\right| \leq \frac{c_5}{N^{m-\|\boldsymbol{\alpha}\|_1}} \omega_{f, m}\left(\frac{1}{N}\right),
\end{align*}
where $c_5$ is a positive constant depending only on $d, m$ and the diameter of $K$, and
\begin{align*}
\omega_{f, m}(\delta) = \sup_{\|\boldsymbol{\gamma}\|_1=m} \left(\sup_{\|\boldsymbol{x}-\boldsymbol{y}\| \leq \delta} \left|\partial^{\boldsymbol{\gamma}} f(\boldsymbol{x})-\partial^{\boldsymbol{\gamma}} f(\boldsymbol{y})\right|\right).
\end{align*}

We now consider $f \in \mathcal{H}_{d,1}^\beta([0, 1]^{d}, K)$ with $\beta>1$ and $K>0$. The Whitney extension theorem provides an extension of $f$ to all of $\mathbb{R}^d$ (See \citet[Theorem 4]{stein1970singular}, \citet[Theorem 2.3.6]{hormander2015analysis}). In more detail, there exists a function $F$ of class $\mathcal{H}_{d,1}^\beta(\mathbb{R}^{d}, c_6 K)$ on $\mathbb{R}^d$ such that for each multi-index $\boldsymbol{\alpha}$ with $\|\boldsymbol{\alpha}\|_1 \leq \lfloor\beta\rfloor$ we have $\partial^{\boldsymbol{\alpha}} F = \partial^{\boldsymbol{\alpha}} f$ on $[0, 1]^{d}$, and
\begin{align*}
\sum_{\boldsymbol{\alpha}: \|\boldsymbol{\alpha}\|_1<\beta} \sup_{\mathbb{R}^d} \left|\partial^{\boldsymbol{\alpha}} F\right| + \sum_{\boldsymbol{\alpha}: \|\boldsymbol{\alpha}\|_1 = \lfloor\beta\rfloor} \sup_{\substack{\boldsymbol{x}, \boldsymbol{y} \in \mathbb{R}^{d} \\ \boldsymbol{x} \neq \boldsymbol{y}}} \frac{\left|\partial^{\boldsymbol{\alpha}} F(\boldsymbol{x})-\partial^{\boldsymbol{\alpha}} F(\boldsymbol{y})\right|}{\|\boldsymbol{x}-\boldsymbol{y}\|^{\beta-\lfloor\beta\rfloor}} \leq c_6 K,
\end{align*}
where $c_6 = c_6(\beta)$. We fix a test function $\Psi \in \mathcal{C}_0^{\infty} (\mathbb{R}^d)$ with compact support and $\Psi \equiv 1$ in the vicinity of $[0, 1]^{d}$. Since $\Psi F$ has compact support and belongs to the class $\mathcal{C}^{\lfloor\beta\rfloor}$, we may apply the aforementioned multivariate simultaneous approximation theorem to $\Psi F$ and find a polynomial $p_N$ of degree at most $N$, which satisfies
\begin{align}
\label{eq: approximate holder function 1}
\begin{aligned}
\sup_{[0, 1]^{d}} \left|\partial^{\boldsymbol{\alpha}} \left(f-p_N\right)\right| &= \sup_{[0, 1]^{d}} \left|\partial^{\boldsymbol{\alpha}} \left(\Psi F-p_N\right)\right| \\
&\leq  \sup_{\mathbb{R}^d} \left|\partial^{\boldsymbol{\alpha}} \left(\Psi F-p_N\right)\right| \\
&\leq \frac{c_7}{N^{\lfloor\beta\rfloor-\|\boldsymbol{\alpha}\|_1}} \omega_{\Psi F, \lfloor\beta\rfloor}\left(\frac{1}{N}\right).
\end{aligned}
\end{align}
For each multi-index $\boldsymbol{\alpha}$ with $\|\boldsymbol{\alpha}\|_1 = \lfloor\beta\rfloor$ and $\boldsymbol{x}, \boldsymbol{y}\in \mathbb{R}^d$ with $\|\boldsymbol{x}-\boldsymbol{y}\| \leq 1$, we have
\begin{align*}
& \left| \partial^{\boldsymbol{\alpha}} (\Psi F)(\boldsymbol{x}) - \partial^{\boldsymbol{\alpha}} (\Psi F)(\boldsymbol{y}) \right| \\
&= \left| \sum_{\boldsymbol{0} \leq \boldsymbol{k}\leq \boldsymbol{\boldsymbol{\alpha}}} \binom{\alpha}{\boldsymbol{k}}\left( \partial^{\boldsymbol{\alpha}-\boldsymbol{k}}\Psi(\boldsymbol{x})\cdot \partial^{\boldsymbol{k}} F(\boldsymbol{x}) - \partial^{\boldsymbol{\alpha}-\boldsymbol{k}}\Psi(\boldsymbol{y})\cdot \partial^{\boldsymbol{k}} F(\boldsymbol{y}) \right) \right| \\
&\leq \sum_{\boldsymbol{0} \leq \boldsymbol{k} \leq \boldsymbol{\alpha}} \binom{\boldsymbol{\alpha}}{\boldsymbol{k}} \left| \partial^{\boldsymbol{\alpha}-\boldsymbol{k}}\Psi(\boldsymbol{x})\cdot \partial^{\boldsymbol{k}} F(\boldsymbol{x}) - \partial^{\boldsymbol{\alpha}-\boldsymbol{k}}\Psi(\boldsymbol{y})\cdot \partial^{\boldsymbol{k}} F(\boldsymbol{y}) \right| \\
&\leq \left|\Psi(\boldsymbol{x})\right| \cdot\left|\partial^{\boldsymbol{\alpha}} F(\boldsymbol{x}) - \partial^{\boldsymbol{\alpha}} F(\boldsymbol{y})\right|+\left|\Psi(\boldsymbol{x})-\Psi(\boldsymbol{y})\right|\cdot \left|\partial^{\boldsymbol{\alpha}} F(\boldsymbol{y})\right| \\
&~~~ + \sum_{\substack{\boldsymbol{0} \leq \boldsymbol{k} \leq \boldsymbol{\alpha} \\ \boldsymbol{k} \neq \boldsymbol{\alpha}}} \binom{\boldsymbol{\alpha}}{\boldsymbol{k}} \left| \partial^{\boldsymbol{\alpha}-\boldsymbol{k}}\Psi(\boldsymbol{x})\cdot \partial^{\boldsymbol{k}} F(\boldsymbol{x}) - \partial^{\boldsymbol{\alpha}-\boldsymbol{k}}\Psi(\boldsymbol{y})\cdot \partial^{\boldsymbol{k}} F(\boldsymbol{y}) \right| \\ 
&\leq \left|\Psi(\boldsymbol{x})\right| \cdot c_6 K \|\boldsymbol{x}-\boldsymbol{y}\|^{\beta-\lfloor\beta\rfloor} + \left\|\nabla \Psi(\xi)\right\| \|\boldsymbol{x}-\boldsymbol{y}\| \cdot \left|\partial^{\boldsymbol{\alpha}} F(\boldsymbol{y})\right| \\
&~~~ + \sum_{\substack{\boldsymbol{0} \leq \boldsymbol{k} \leq \boldsymbol{\alpha} \\ \boldsymbol{k} \neq \boldsymbol{\alpha}}} \binom{\boldsymbol{\alpha}}{\boldsymbol{k}} \left\|\nabla \left(\partial^{\boldsymbol{\alpha}-\boldsymbol{k}}\Psi\cdot \partial^{\boldsymbol{k}} F\right)(\xi_k)\right\|  \|\boldsymbol{x}-\boldsymbol{y}\| \\
&\leq \left\|\Psi\right\|_{L^\infty(\mathbb{R}^d)} \cdot c_6 K \|\boldsymbol{x}-\boldsymbol{y}\|^{\beta-\lfloor\beta\rfloor} + \left\|\nabla\Psi\right\|_{L^\infty(\mathbb{R}^d)} \|\boldsymbol{x}-\boldsymbol{y}\| \cdot c_6 K \\
&~~~ + \sum_{\substack{\boldsymbol{0} \leq \boldsymbol{k} \leq \boldsymbol{\alpha} \\ \boldsymbol{k} \neq \boldsymbol{\alpha}}} \binom{\boldsymbol{\alpha}}{\boldsymbol{k}} \left(\left\|\nabla\left(\partial^{\boldsymbol{\alpha}-\boldsymbol{k}}\Psi\right)\right\|_{L^\infty(\mathbb{R}^d)} + \left\|\partial^{\boldsymbol{\alpha}-\boldsymbol{k}}\Psi\right\|_{L^\infty(\mathbb{R}^d)}\right) c_6 K \|\boldsymbol{x}-\boldsymbol{y}\| \\
&\leq c_8 (\Psi, c_6(\beta), d, \beta) K \cdot \|\boldsymbol{x}-\boldsymbol{y}\|^{\beta-\lfloor\beta\rfloor},
\end{align*}
where in the last inequality we use $0<\beta-\lfloor\beta\rfloor\leq 1$ and $\|\boldsymbol{x}-\boldsymbol{y}\|\leq 1$, which indicates
\begin{align}
\label{eq: approximate holder function 2}
\omega_{\Psi F, \lfloor\beta\rfloor}\left(\frac{1}{N}\right) \leq \frac{c_8 K}{N^{\beta-\lfloor\beta\rfloor}}.
\end{align}
Combining (\ref{eq: approximate holder function 1}) and (\ref{eq: approximate holder function 2}), we complete the proof for the case when $\beta > 1$.

For $f \in \mathcal{H}_{d,1}^\beta([0, 1]^{d}, K)$ with $0<\beta\leq 1$, there is a completely analogous proof. 
\end{proof}

\begin{lemma}[\cite{bagby2002multivariate}, Theorem 2]
\label{lemma: app simultaneous approximation 3}
For each $f \in \mathcal{C}_{d,1}^m ([0, 1]^{d}, K)$ with $m\in\mathbb{N}$ and positive integer $N$, there is a polynomial $p_N$ of degree at most $N$ on $\mathbb{R}^d$, such that for each multi-index $\boldsymbol{\alpha}$ with $\|\boldsymbol{\alpha}\|_1 \leq \min \{m, N\}$ we have
\begin{align*}
\sup_{[0, 1]^{d}} \left|\partial^{\boldsymbol{\alpha}}\left(f-p_N\right)\right| \leq \frac{c_9 K}{N^{m-\|\boldsymbol{\alpha}\|_1}},
\end{align*}
where $c_9$ is a positive constant depending only on $d, m$.
\end{lemma}

\subsection{Approximation of velocity field} \label{appendix: app.4}

Lemma \ref{lemma: true 3}, Lemma \ref{lemma: true 4}, and Lemma \ref{lemma: true 5} demonstrate that the true velocity field has certain regularities. As a result, the previous approximation results can be naturally applied to the approximation of the velocity field.

\begin{proof}[Proof of Corollary \ref{coro: app true vd}]
We define $\widetilde{\boldsymbol{v}}^*(\boldsymbol{x}, t) := \boldsymbol{v}^* (2 R \boldsymbol{x} - R \mathbbm{1}, T t)$ and restrict the input space of $\widetilde{\boldsymbol{v}}^*$ to $[0,1]^d \times [0,1]$. According to Lemma \ref{lemma: true 3}, Lemma \ref{lemma: true 4} and Lemma \ref{lemma: true 5}, we have $\widetilde{\boldsymbol{v}}^*(\boldsymbol{x}, t) \in \mathcal{C}_{d+1, d}^1 ([0,1]^d \times [0,1], K)$ with $K = \mathcal{O} (\frac{R}{(1-T)^3})$. It follows from Theorem \ref{corollary: app 1} that for any $\varepsilon \in (0,1)$ there exists a transformer network $\widetilde{\boldsymbol{v}}(\boldsymbol{x}, t) \in \mathcal{T}_{d+1, d}\left(N, h, d_k, d_v, d_{ff}, B, J, \gamma\right)$ with configuration
\begin{align*}
\begin{gathered}
N = \mathcal{O} \left( \log \left(\frac{R}{(1-T)^3\varepsilon}\right)\right), \quad
h = \mathcal{O} \left(  \left( \frac{R}{(1-T)^3 \varepsilon} \right)^{d+1} \right), \\
d_{ff} = 8 h, \quad 
d_k = \mathcal{O}(1), \quad 
d_v = \mathcal{O}(1),\\
B = \mathcal{O} \left( \frac{R}{1-T} \right), \quad
J = \mathcal{O} \left(\left( \frac{R}{(1-T)^3\varepsilon} \right)^{d+1} \log \left(\frac{R}{(1-T)^3 \varepsilon}\right) \right), \quad
\gamma = \mathcal{O} \left(\frac{R}{(1-T)^3}\right), 
\end{gathered}
\end{align*}
such that
\begin{align*}
\|\widetilde{\boldsymbol{v}}(\boldsymbol{x}, t) - \widetilde{\boldsymbol{v}}^*(\boldsymbol{x}, t)\|_{L^\infty ([0,1]^d \times [0,1])} \leq \varepsilon,
\end{align*}
which implies
\begin{align*}
\left\|\widetilde{\boldsymbol{v}}\left(\frac{1}{2R}(\boldsymbol{x}+R\mathbbm{1}), \frac{1}{T} t\right) - \boldsymbol{v}^*(\boldsymbol{x}, t)\right\|_{L^\infty ([-R, R]^d \times [0, T])} \leq \varepsilon.
\end{align*}
Let $\boldsymbol{v}(\boldsymbol{x}, t) := \widetilde{\boldsymbol{v}} (\frac{1}{2R}(\operatorname{Proj}_{[-R,R]^d} (\boldsymbol{x}) +R\mathbbm{1}), \frac{1}{T} t)$ for $\boldsymbol{x} \in \mathbb{R}^d$, then $\boldsymbol{v}(\boldsymbol{x}, t) \in \mathcal{T}$ and  
\begin{align*}
\left\|\boldsymbol{v}(\boldsymbol{x}, t) - \boldsymbol{v}^*(\boldsymbol{x}, t)\right\|_{L^\infty ([-R, R]^d \times [0, T])} \leq \varepsilon.
\end{align*}
By the definition of $\mathcal{T}$ as defined in (\ref{eq: def tao}), we choose
\begin{align*}
\gamma_{\boldsymbol{x}} = \frac{\gamma}{2R} = \mathcal{O} \left(\frac{1}{(1-T)^3}\right), \quad \gamma_{t} = \frac{\gamma}{T} = \mathcal{O} \left(\frac{R}{(1-T)^3}\right).
\end{align*}
\end{proof}

\section{Generalization Error} \label{appendix: gen}
In this section, we provide the proof for the generalization error by estimating the covering number of the transformer network function class.

\subsection{Proof of Lemma \ref{lemma:gen 2} and Theorem \ref{theorem: gen 1}} \label{appendix: gen.1}
\begin{proof}[Proof of Lemma \ref{lemma:gen 2}]
By performing some calculations, we have
\begin{align}
\begin{aligned}
\label{eq:gen 1}
& \mathbb{E} \left[\left\|X_1 - \frac{t}{\sqrt{1-t^2}} X_0-\boldsymbol{v}\left(X_t, t\right)\right\|^2\right] \\
&= \mathbb{E}\left[\left\|X_1 - \frac{t}{\sqrt{1-t^2}} X_0 -\boldsymbol{v}^*\left(X_t, t\right)+\boldsymbol{v}^*\left(X_t, t\right)-\boldsymbol{v}\left(X_t, t\right)\right\|^2\right] \\
&= \mathbb{E}\left[\left\|X_1 - \frac{t}{\sqrt{1-t^2}} X_0-\boldsymbol{v}^*\left(X_t, t\right)\right\|^2\right]+\left\|\boldsymbol{v}(\cdot, t)-\boldsymbol{v}^*(\cdot, t)\right\|_{L^2\left(\pi_t\right)}^2 \\
&~~~ +2 \mathbb{E}\left[\left\langle X_1-\frac{t}{\sqrt{1-t^2}}X_0-\boldsymbol{v}^*\left(X_t, t\right), \boldsymbol{v}^*\left(X_t, t\right)-\boldsymbol{v}\left(X_t, t\right) \right\rangle\right].
\end{aligned}
\end{align}
By taking expectation conditioned on $X_t$, we have
\begin{align*}
& \mathbb{E}_{X_0, X_1} \left[\left\langle X_1 - \frac{t}{\sqrt{1-t^2}} X_0-\boldsymbol{v}^*\left(X_t, t\right), \boldsymbol{v}^*\left(X_t, t\right)-\boldsymbol{v}\left(X_t, t\right)\right\rangle\right] \\
&= \mathbb{E}_{X_t}\left[\mathbb{E}_{X_0, X_1} \left[\left\langle X_1 - \frac{t}{\sqrt{1-t^2}}X_0-\boldsymbol{v}^*\left(X_t, t\right), \boldsymbol{v}^*\left(X_t, t\right)-\boldsymbol{v}\left(X_t, t\right)\right\rangle \bigg| X_t\right]\right] \\
&= \mathbb{E}_{X_t} \left[\left\langle\mathbb{E}_{X_0, X_1} \left[X_1 - \frac{t}{\sqrt{1-t^2}} X_0 \bigg| X_t\right]-\boldsymbol{v}^*\left(X_t, t\right), \boldsymbol{v}^*\left(X_t, t\right)-\boldsymbol{v}\left(X_t, t\right)\right\rangle\right] \\
&= \mathbb{E}_{X_t}\left[\left\langle\boldsymbol{v}^*\left(X_t, t\right)-\boldsymbol{v}^*\left(X_t, t\right), \boldsymbol{v}^*\left(X_t, t\right)-\boldsymbol{v}\left(X_t, t\right)\right\rangle\right] = 0.
\end{align*}
Substituting the aforementioned identity into (\ref{eq:gen 1}) and integrating over the interval $[0, T]$ w.r.t. $t$, we obtain
\begin{align*}
\mathcal{L}(\boldsymbol{v})=\mathcal{L}\left(\boldsymbol{v}^*\right)+\frac{1}{T} \int_0^T\left\|\boldsymbol{v}(\cdot, t)-\boldsymbol{v}^*(\cdot, t)\right\|_{L^2\left(\pi_t\right)}^2 \mathrm{d} t,
\end{align*}
which concludes the proof.
\end{proof}

\begin{proof}[Proof of Theorem \ref{theorem: gen 1}]
Let $R>0$ be determined later. Let $\boldsymbol{x}_{t,i} := t_i \boldsymbol{x}_{1,i}+\sqrt{1-t_i^2} \boldsymbol{x}_{0,i}$. 

$\bullet$ \textbf{Error decomposition}. We consider the error decomposition in an asymmetric form
\begin{align}
\label{eq:gen 6}
\begin{aligned}
& \mathcal{L}(\widehat{\boldsymbol{v}}) - \mathcal{L}(\boldsymbol{v}^*) \\
=& \frac{1}{T} \int_0^T\left\|\widehat{\boldsymbol{v}}(\cdot, t)-\boldsymbol{v}^*(\cdot, t)\right\|_{L^2\left(\pi_t\right)}^2 \mathrm{d} t \\
=& \mathbb{E}_{t, X_t} \left[ \left\|\widehat{\boldsymbol{v}}(X_t, t)-\boldsymbol{v}^*(X_t, t)\right\|^2 \right] \\
=& \mathbb{E}_{t, X_t} \left[ \left\|\widehat{\boldsymbol{v}}(X_t, t)-\boldsymbol{v}^*(X_t, t)\right\|^2 \mathbbm{1}\{\|X_t\|_{\infty} \leq R\} \right] \\
& + \mathbb{E}_{t, X_t} \left[ \left\|\widehat{\boldsymbol{v}}(X_t, t)-\boldsymbol{v}^*(X_t, t)\right\|^2 \mathbbm{1}\{\|X_t\|_{\infty} > R\} \right] \\
=& \mathbb{E}_{t, X_t} \left[ \left\|\widehat{\boldsymbol{v}}(X_t, t)-\boldsymbol{v}^*(X_t, t)\right\|^2 \mathbbm{1}\{\|X_t\|_{\infty} \leq R\} \right] \\
& - \frac{3}{n}\sum_{i=1}^n \left\|\widehat{\boldsymbol{v}}(\boldsymbol{x}_{t,i}, t_i)-\boldsymbol{v}^*(\boldsymbol{x}_{t,i}, t_i)\right\|^2 \mathbbm{1}\{\|\boldsymbol{x}_{t,i}\|_{\infty} \leq R\} \\
& + \frac{3}{n}\sum_{i=1}^n \left\|\widehat{\boldsymbol{v}}(\boldsymbol{x}_{t,i}, t_i)-\boldsymbol{v}^*(\boldsymbol{x}_{t,i}, t_i)\right\|^2 \mathbbm{1}\{\|\boldsymbol{x}_{t,i}\|_{\infty} \leq R\} \\
& + \mathbb{E}_{t, X_t} \left[ \left\|\widehat{\boldsymbol{v}}(X_t, t)-\boldsymbol{v}^*(X_t, t)\right\|^2 \mathbbm{1}\{\|X_t\|_{\infty} > R\} \right].
\end{aligned}
\end{align}
For simplicity, we denote 
\begin{align*}
\Rmnum{1} &= \mathbb{E}_{t, X_t} \left[ \left\|\widehat{\boldsymbol{v}}(X_t, t)-\boldsymbol{v}^*(X_t, t)\right\|^2 \mathbbm{1}\{\|X_t\|_{\infty} \leq R\} \right] \\
&~~~ - \frac{3}{n}\sum_{i=1}^n \left\|\widehat{\boldsymbol{v}}(\boldsymbol{x}_{t,i}, t_i)-\boldsymbol{v}^*(\boldsymbol{x}_{t,i}, t_i)\right\|^2 \mathbbm{1}\{\|\boldsymbol{x}_{t,i}\|_{\infty} \leq R\} \\
\Rmnum{2} &= \frac{1}{n}\sum_{i=1}^n \left\|\widehat{\boldsymbol{v}}(\boldsymbol{x}_{t,i}, t_i)-\boldsymbol{v}^*(\boldsymbol{x}_{t,i}, t_i)\right\|^2 \mathbbm{1}\{\|\boldsymbol{x}_{t,i}\|_{\infty} \leq R\} \\
\Rmnum{3} &= \mathbb{E}_{t, X_t} \left[ \left\|\widehat{\boldsymbol{v}}(X_t, t)-\boldsymbol{v}^*(X_t, t)\right\|^2 \mathbbm{1}\{\|X_t\|_{\infty} > R\} \right].
\end{align*}

$\bullet$ \textbf{Bound of $\E_{\mathcal{X}} \left[\Rmnum{1}\right]$.}
Let $\mathcal{H}=\{ h(\boldsymbol{x}_t, t) := \left\|\boldsymbol{v}(\boldsymbol{x}_t, t)-\boldsymbol{v}^*(\boldsymbol{x}_t, t)\right\|^2 \mathbbm{1}\{\|\boldsymbol{x}_t\|_{\infty} \leq R\} : \boldsymbol{v} \in \mathcal{T} \}$. For any $h\in \mathcal{H}$, we have
\begin{align*}
h(\boldsymbol{x}_t, t) &\leq 2\left\|\boldsymbol{v}(\boldsymbol{x}_t, t)\right\|^2 \mathbbm{1}\{\|\boldsymbol{x}_t\|_{\infty} \leq R\} + 2\left\|\boldsymbol{v}^*(\boldsymbol{x}_t, t)\right\|^2 \mathbbm{1}\{\|\boldsymbol{x}_t\|_{\infty} \leq R\} \\
&\leq 2B^2 + \frac{2d(1+R)^2}{(1-T^2)^2},
\end{align*}
where in the second equality we use Lemma \ref{lemma: true 3}, and subsequently
\begin{align*}
\E_{\mathcal{X}} \left[\Rmnum{1}\right] & \leq \mathbb{E}_{\mathcal{X}} \left[\sup _{h \in \mathcal{H}} 2 \mathbb{E}[h]-\mathbb{E}[h]-\frac{2}{n} \sum_{i=1}^n h\left(\boldsymbol{x}_{t,i}, t_i\right)-\frac{1}{n} \sum_{i=1}^n h\left(\boldsymbol{x}_{t,i}, t_i\right)\right] \\
& \leq \mathbb{E}_{\mathcal{X}} \left[\sup_{h \in \mathcal{H}} 2 \mathbb{E}[h]-\frac{1}{2B^2 + \frac{2d(1+R)^2}{(1-T^2)^2}} \mathbb{E}[h^2]-\frac{2}{n} \sum_{i=1}^n h\left(\boldsymbol{x}_{t,i}, t_i\right) \right. \\
&\hspace{1.5cm} \left. -\frac{1}{\left(2B^2 + \frac{2d(1+R)^2}{(1-T^2)^2}\right) n} \sum_{i=1}^n h^2\left(\boldsymbol{x}_{t,i}, t_i\right)\right].
\end{align*}
We can then use the symmetrization technique to bound it by Rademacher complexity. We introduce a ghost dataset $\mathcal{X}^{\prime}=\left\{t_i^{\prime}, \boldsymbol{x}_{0,i}^{\prime}, \boldsymbol{x}_{1,i}^{\prime}\right\}_{i=1}^n$ drawn i.i.d. from $\text {Unif}[0, T], \pi_0$ and $\pi_1$, and let $\tau=\left\{\tau_i\right\}_{i=1}^n$ be a sequence of i.i.d. Rademacher variables independent of both $\mathcal{X}$ and $\mathcal{X}^{\prime}$. Then,
\begin{align*}
& \mathbb{E}_\mathcal{X}\left[\sup _{h \in \mathcal{H}} 2 \mathbb{E}[h]-\dfrac{1}{2B^2 + \frac{2d(1+R)^2}{(1-T^2)^2}} \mathbb{E}[h^2]-\dfrac{2}{n} \sum_{i=1}^n h\left(\boldsymbol{x}_{t,i}, t_i\right) \right. \\
&\hspace{1.5cm} \left. -\dfrac{1}{\left(2B^2 + \frac{2d(1+R)^2}{(1-T^2)^2}\right) n} \sum_{i=1}^n h^2\left(\boldsymbol{x}_{t,i}, t_i\right)\right] \\
= & \mathbb{E}_\mathcal{X}\left[\sup _{h \in \mathcal{H}} \mathbb{E}_{\mathcal{X}^{\prime}}\left[\dfrac{2}{n} \sum_{i=1}^n h\left(\boldsymbol{x}_{t,i}^{\prime}, t_i^{\prime}\right)-\dfrac{1}{\left(2B^2 + \frac{2d(1+R)^2}{(1-T^2)^2}\right) n} \sum_{i=1}^n h^2\left(\boldsymbol{x}_{t,i}^{\prime}, t_i^{\prime}\right)\right]\right. \\
&\hspace{1.5cm}    \left.-\dfrac{2}{n} \sum_{i=1}^n h\left(\boldsymbol{x}_{t,i}, t_i\right)-\dfrac{1}{\left(2B^2 + \frac{2d(1+R)^2}{(1-T^2)^2}\right) n} \sum_{i=1}^n h^2\left(\boldsymbol{x}_{t,i}, t_i\right)\right] \\
\leq & \mathbb{E}_{\mathcal{X}, \mathcal{X}^{\prime}}\left[\sup _{h \in \mathcal{H}} \frac{2}{n} \sum_{i=1}^n h\left(\boldsymbol{x}_{t,i}^{\prime}, t_i^{\prime}\right)-\frac{2}{n} \sum_{i=1}^n h\left(\boldsymbol{x}_{t,i}, t_i\right)-\frac{1}{\left(2B^2 + \frac{2d(1+R)^2}{(1-T^2)^2}\right) n} \sum_{i=1}^n h^2\left(\boldsymbol{x}_{t,i}^{\prime}, t_i^{\prime}\right)\right. \\
&\hspace{1.5cm}     \left.-\frac{1}{\left(2B^2 + \frac{2d(1+R)^2}{(1-T^2)^2}\right) n} \sum_{i=1}^n h^2\left(\boldsymbol{x}_{t,i}, t_i\right)\right] \\
= & \mathbb{E}_{\mathcal{X}, \mathcal{X}^{\prime}, \tau} \left[ \sup_{h \in \mathcal{H}} \frac{2}{n} \sum_{i=1}^n \tau_i\left(h\left(\boldsymbol{x}_{t,i}^{\prime}, t_i^{\prime}\right)-h\left(\boldsymbol{x}_{t,i}, t_i\right)\right) \right. \\
&\hspace{1.5cm} \left. -\frac{1}{\left(2B^2 + \frac{2d(1+R)^2}{(1-T^2)^2}\right) n} \sum_{i=1}^n\left(h^2\left(\boldsymbol{x}_{t,i}^{\prime}, t_i^{\prime}\right)+h^2\left(\boldsymbol{x}_{t,i}, t_i\right)\right)\right] \\
\leq & \mathbb{E}_{\mathcal{X}, \mathcal{X}^{\prime}, \tau}\left[\sup _{h \in \mathcal{H}}\left(\frac{2}{n} \sum_{i=1}^n \tau_i h\left(\boldsymbol{x}_{t,i}^{\prime}, t_i^{\prime}\right)-\frac{1}{\left(2B^2 + \frac{2d(1+R)^2}{(1-T^2)^2}\right) n} \sum_{i=1}^n h^2\left(\boldsymbol{x}_{t,i}^{\prime}, t_i^{\prime}\right)\right)\right. \\
&\hspace{1.5cm}      \left.+\sup _{h \in \mathcal{H}}\left(\frac{2}{n} \sum_{i=1}^n\left(-\tau_i\right) h\left(\boldsymbol{x}_{t,i}, t_i\right)-\frac{1}{\left(2B^2 + \frac{2d(1+R)^2}{(1-T^2)^2}\right) n} \sum_{i=1}^n h^2\left(\boldsymbol{x}_{t,i}, t_i\right)\right)\right] \\
= & \mathbb{E}_{\mathcal{X}, \tau}\left[\sup _{h \in \mathcal{H}} \frac{4}{n} \sum_{i=1}^n \tau_i h\left(\boldsymbol{x}_{t,i}, t_i\right)-\frac{1}{\left(B^2 + \frac{d(1+R)^2}{(1-T^2)^2}\right) n} \sum_{i=1}^n h^2\left(\boldsymbol{x}_{t,i}, t_i\right)\right],
\end{align*}
where the second equality is due to the fact that randomly interchange of the corresponding components of $\mathcal{X}$ and $\mathcal{X}^{\prime}$ doesn't affect the joint distribution of $\mathcal{X}, \mathcal{X}^{\prime}$ and the summation $\sum_{i=1}^n(h^2 ( \boldsymbol{x}_{t,i}^{\prime}, t_i^{\prime}) + h^2\left( \boldsymbol{x}_{t,i}, t_i\right))$, and the last equality is because $\left( \boldsymbol{x}_{t,i}, t_i\right)$ and $(\boldsymbol{x}_{t,i}^{\prime}, t_i^{\prime})$ have the same distribution and the $\tau_i$ and $-\tau_i$ have the same distribution.

For any fixed $\mathcal{X}$, we discretize $\mathcal{H}$ with respect to the metric 
\begin{align*}
d_{\mathcal{X}, 1}\left(h, h^{\prime}\right) := \frac{1}{n} \sum_{i=1}^n\left|h\left(\boldsymbol{x}_{t,i}, t_i\right)-h^{\prime}\left(\boldsymbol{x}_{t,i}, t_i\right)\right|.
\end{align*}
Let $\mathcal{H}_\delta(\mathcal{X})$ be a $\delta$-cover of $\mathcal{H}$ with minimal cardinality under the distance $d_{\mathcal{X}, 1}$, then for any $h \in \mathcal{H}$, there exists $g \in \mathcal{H}_\delta(\mathcal{X})$ such that $\frac{1}{n} \sum_{i=1}^n |h\left(\boldsymbol{x}_{t,i}, t_i\right)-g\left(\boldsymbol{x}_{t,i}, t_i\right)| \leq \delta$. Therefore,
\begin{align*}
\begin{aligned}
\frac{1}{n} \sum_{i=1}^n \tau_i h\left(\boldsymbol{x}_{t,i}, t_i\right) & \leq \frac{1}{n} \sum_{i=1}^n \tau_i g\left(\boldsymbol{x}_{t,i}, t_i\right)+\frac{1}{n} \sum_{i=1}^n\left|\tau_i\right| \left|h\left(\boldsymbol{x}_{t,i}, t_i\right)-g\left(\boldsymbol{x}_{t,i}, t_i\right)\right| \\
& \leq \frac{1}{n} \sum_{i=1}^n \tau_i g\left(\boldsymbol{x}_{t,i}, t_i\right)+\delta
\end{aligned}
\end{align*}
and since $|h\left(\boldsymbol{x}_{t,i}, t_i\right)|,|g\left(\boldsymbol{x}_{t,i}, t_i\right)| \leq 2B^2 + \frac{2d(1+R)^2}{(1-T^2)^2}$,
\begin{align*}
\begin{aligned}
\frac{1}{n} \sum_{i=1}^n h^2\left(\boldsymbol{x}_{t,i}, t_i\right) & =\frac{1}{n} \sum_{i=1}^n g^2\left(\boldsymbol{x}_{t,i}, t_i\right)+\frac{1}{n} \sum_{i=1}^n\left(h\left(\boldsymbol{x}_{t,i}, t_i\right)+g\left(\boldsymbol{x}_{t,i}, t_i\right)\right)\left(h\left(\boldsymbol{x}_{t,i}, t_i\right)-g\left(\boldsymbol{x}_{t,i}, t_i\right)\right) \\
& \geq \frac{1}{n} \sum_{i=1}^n g^2\left(\boldsymbol{x}_{t,i}, t_i\right) - 2\left(2B^2 + \frac{2d(1+R)^2}{(1-T^2)^2}\right) \frac{1}{n} \sum_{i=1}^n\left|h\left(\boldsymbol{x}_{t,i}, t_i\right)-g\left(\boldsymbol{x}_{t,i}, t_i\right)\right| \\
& \geq \frac{1}{n} \sum_{i=1}^n g^2\left(\boldsymbol{x}_{t,i}, t_i\right)-2\left(2B^2 + \frac{2d(1+R)^2}{(1-T^2)^2}\right)  \delta.
\end{aligned}
\end{align*}
It follows that 
\begin{align*}
\E_{\mathcal{X}} \left[\Rmnum{1}\right] &\leq \mathbb{E}_{\mathcal{X}, \tau}\left[\sup _{h \in \mathcal{H}} \frac{4}{n} \sum_{i=1}^n \tau_i h\left(\boldsymbol{x}_{t,i}, t_i\right)-\frac{1}{\left(B^2 + \frac{d(1+R)^2}{(1-T^2)^2}\right) n} \sum_{i=1}^n h^2\left(\boldsymbol{x}_{t,i}, t_i\right)\right] \\
&\leq 4\mathbb{E}_{\mathcal{X}, \tau}\left[\sup _{g \in \mathcal{H}_\delta(\mathcal{X})} \frac{1}{n} \sum_{i=1}^n \tau_i g\left(\boldsymbol{x}_{t,i}, t_i\right)-\frac{1}{\left(4B^2 + \frac{4d(1+R)^2}{(1-T^2)^2}\right) n} \sum_{i=1}^n g^2\left(\boldsymbol{x}_{t,i}, t_i\right)\right]+8 \delta.
\end{align*}
For fixed $\mathcal{X}$ and any $\lambda>0$, we have
\begin{align*}
\begin{aligned}
& \exp \left(\lambda \mathbb{E}_\tau\left[\sup _{g \in \mathcal{H}_\delta(\mathcal{X})} \frac{1}{n} \sum_{i=1}^n \tau_i g\left(\boldsymbol{x}_{t,i}, t_i\right)-\frac{1}{\left(4B^2 + \frac{4d(1+R)^2}{(1-T^2)^2}\right) n} \sum_{i=1}^n g^2\left(\boldsymbol{x}_{t,i}, t_i\right)\right]\right) \\
\leq & \mathbb{E}_\tau\left[\exp \left(\lambda \sup _{g \in \mathcal{H}_\delta(\mathcal{X})} \frac{1}{n} \sum_{i=1}^n \tau_i g\left(\boldsymbol{x}_{t,i}, t_i\right)-\frac{1}{\left(4B^2 + \frac{4d(1+R)^2}{(1-T^2)^2}\right) n} \sum_{i=1}^n g^2\left(\boldsymbol{x}_{t,i}, t_i\right)\right)\right] \\
\leq & \sum_{g \in \mathcal{H}_\delta(\mathcal{X})} \mathbb{E}_\tau\left[\exp \left(\frac{\lambda}{n} \sum_{i=1}^n \tau_i g\left(\boldsymbol{x}_{t,i}, t_i\right)-\frac{\lambda}{\left(4B^2 + \frac{4d(1+R)^2}{(1-T^2)^2}\right) n} \sum_{i=1}^n g^2\left(\boldsymbol{x}_{t,i}, t_i\right)\right)\right] \\
= & \sum_{g \in \mathcal{H}_\delta(\mathcal{X})} \prod_{i=1}^n \mathbb{E}_{\tau_i}\left[\exp \left(\frac{\lambda}{n} \tau_i g\left(\boldsymbol{x}_{t,i}, t_i\right)-\frac{\lambda}{\left(4B^2 + \frac{4d(1+R)^2}{(1-T^2)^2}\right) n} g^2\left(\boldsymbol{x}_{t,i}, t_i\right)\right)\right] \\
\leq & \sum_{g \in \mathcal{H}_\delta(\mathcal{X})} \prod_{i=1}^n \exp \left(\frac{\lambda^2}{2 n^2} g^2\left(\boldsymbol{x}_{t,i}, t_i\right)-\frac{\lambda}{\left(4B^2 + \frac{4d(1+R)^2}{(1-T^2)^2}\right) n} g^2\left(\boldsymbol{x}_{t,i}, t_i\right)\right),
\end{aligned}
\end{align*}
where we use Lemma \ref{lemma:gen 1} in the last inequality. If we take $\lambda=n/(2B^2 + \frac{2d(1+R)^2}{(1-T^2)^2})$, then
\begin{align*}
&\mathbb{E}_\tau\left[\sup _{g \in \mathcal{H}_\delta(\mathcal{X})} \frac{1}{n} \sum_{i=1}^n \tau_i g\left(\boldsymbol{x}_{t,i}, t_i\right)-\frac{1}{\left(4B^2 + \frac{4d(1+R)^2}{(1-T^2)^2}\right) n} \sum_{i=1}^n g^2\left(\boldsymbol{x}_{t,i}, t_i\right)\right] \\
\leq & \frac{1}{\lambda} \log \left|\mathcal{H}_\delta(\mathcal{X})\right|=\frac{2}{n} \left(B^2 + \frac{d(1+R)^2}{(1-T^2)^2}\right)\log \mathcal{N}\left(\delta, \mathcal{H}, d_{\mathcal{X}, 1}\right).
\end{align*}
As a consequence,
\begin{align}
\label{eq:gen 14}
\E_{\mathcal{X}} \left[\Rmnum{1}\right] \leq \frac{8}{n}\left(B^2 + \frac{d(1+R)^2}{(1-T^2)^2}\right) \mathbb{E}_\mathcal{X} \log \mathcal{N}\left(\delta, \mathcal{H}, d_{\mathcal{X}, 1}\right)+8 \delta.
\end{align}

$\bullet$ \textbf{Bound of $\E_{\mathcal{X}} \left[\Rmnum{2}\right]$.} 
For any $\boldsymbol{v}\in \mathcal{T}$, we have 
\begin{align*}
&\E_{\mathcal{X}} \left[ \widehat{\mathcal{L}}(\boldsymbol{v}) - \widehat{\mathcal{L}}(\boldsymbol{v}^*) \right] \\
=& \E_{\mathcal{X}} \left[\frac{1}{n} \sum_{i=1}^n \left\| \boldsymbol{v}\left(\boldsymbol{x}_{t,i}, t_i\right) - \boldsymbol{v}^*\left(\boldsymbol{x}_{t,i}, t_i\right) \right\|^2 \right. \\
&\hspace{1cm}  + \frac{2}{n} \sum_{i=1}^n \left\langle \boldsymbol{x}_{1,i} - \frac{t_i}{\sqrt{1-t_i^2}}\boldsymbol{x}_{0,i} -  \boldsymbol{v}^*\left(\boldsymbol{x}_{t,i}, t_i\right), \boldsymbol{v}^*\left(\boldsymbol{x}_{t,i}, t_i\right) - \boldsymbol{v}\left(\boldsymbol{x}_{t,i}, t_i\right) \right\rangle \Bigg] \\
=& \E_{\mathcal{X}} \left[\frac{1}{n} \sum_{i=1}^n \left\| \boldsymbol{v}\left(\boldsymbol{x}_{t,i}, t_i\right) - \boldsymbol{v}^*\left(\boldsymbol{x}_{t,i}, t_i\right) \right\|^2 \right],
\end{align*}
where in the last equality we use the identity from Lemma \ref{lemma:gen 2}. Recall that $\widehat{\boldsymbol{v}} \in \arg \min_{\boldsymbol{v} \in \mathcal{T}} \widehat{\mathcal{L}}(\boldsymbol{v})$, we have
\begin{align*}
& \E_{\mathcal{X}} \left[\frac{1}{n} \sum_{i=1}^n \left(\left\| \widehat{\boldsymbol{v}}\left(\boldsymbol{x}_{t,i}, t_i\right) - \boldsymbol{v}^*\left(\boldsymbol{x}_{t,i}, t_i\right) \right\|^2 - 
\left\| \boldsymbol{v}\left(\boldsymbol{x}_{t,i}, t_i\right) - \boldsymbol{v}^*\left(\boldsymbol{x}_{t,i}, t_i\right) \right\|^2 \right) \right] \\ 
&= \E_{\mathcal{X}} \left[ \widehat{\mathcal{L}}(\widehat{\boldsymbol{v}}) - \widehat{\mathcal{L}}(\boldsymbol{v}) \right] \leq 0,
\end{align*}
and subsequently
\begin{align}
\label{eq:gen 2}
\begin{aligned}
& \E_{\mathcal{X}} \left[\Rmnum{2}\right] \\
&= \E_{\mathcal{X}} \left[ \frac{1}{n}\sum_{i=1}^n \left\|\widehat{\boldsymbol{v}}(\boldsymbol{x}_{t,i}, t_i)-\boldsymbol{v}^*(\boldsymbol{x}_{t,i}, t_i)\right\|^2 \mathbbm{1}\{\|\boldsymbol{x}_{t,i}\|_{\infty} \leq R\} \right] \\
&= \E_{\mathcal{X}} \left[ \frac{1}{n}\sum_{i=1}^n \left\|\boldsymbol{v}(\boldsymbol{x}_{t,i}, t_i)-\boldsymbol{v}^*(\boldsymbol{x}_{t,i}, t_i)\right\|^2 \mathbbm{1}\{\|\boldsymbol{x}_{t,i}\|_{\infty} \leq R\} \right] \\
&~~~ + \E_{\mathcal{X}} \left[\frac{1}{n} \sum_{i=1}^n \left(\left\| \widehat{\boldsymbol{v}}\left(\boldsymbol{x}_{t,i}, t_i\right) - \boldsymbol{v}^*\left(\boldsymbol{x}_{t,i}, t_i\right) \right\|^2 - 
\left\| \boldsymbol{v}\left(\boldsymbol{x}_{t,i}, t_i\right) - \boldsymbol{v}^*\left(\boldsymbol{x}_{t,i}, t_i\right) \right\|^2 \right)\mathbbm{1}\{\|\boldsymbol{x}_{t,i}\|_{\infty} \leq R\} \right] \\
&\leq \E_{t,X_t} \left[ \left\|\boldsymbol{v}(X_t, t)-\boldsymbol{v}^*(X_t, t)\right\|^2 \mathbbm{1}\{\|X_t\|_{\infty} \leq R\} \right] \\
&~~~ + \E_{\mathcal{X}} \left[\frac{1}{n} \sum_{i=1}^n \left(\left\| \boldsymbol{v}\left(\boldsymbol{x}_{t,i}, t_i\right) - \boldsymbol{v}^*\left(\boldsymbol{x}_{t,i}, t_i\right) \right\|^2 - \left\| \widehat{\boldsymbol{v}}\left(\boldsymbol{x}_{t,i}, t_i\right) - \boldsymbol{v}^*\left(\boldsymbol{x}_{t,i}, t_i\right) \right\|^2\right)\mathbbm{1}\{\|\boldsymbol{x}_{t,i}\|_{\infty} > R\} \right].
\end{aligned}
\end{align}
On the one hand, for any $\boldsymbol{v}\in \mathcal{T}$, we have 
\begin{align}
\label{eq:gen 3}
\begin{aligned}
&\left\| \boldsymbol{v}\left(\boldsymbol{x}_{t,i}, t_i\right) - \boldsymbol{v}^*\left(\boldsymbol{x}_{t,i}, t_i\right) \right\|^2 \\
&\leq 2\left\| \boldsymbol{v}\left(\boldsymbol{x}_{t,i}, t_i\right)\right\|^2 + 2\left\|\boldsymbol{v}^*\left(\boldsymbol{x}_{t,i}, t_i\right) \right\|^2 \\
&\leq 2B^2 + \frac{4d}{(1-t_i^2)^2} + \frac{4t_i^2}{(1-t_i^2)^2}\left\| \boldsymbol{x}_{t,i} \right\|^2 \\
&= 2B^2 + \frac{4d}{(1-t_i^2)^2} + \frac{4t_i^2}{(1-t_i^2)^2}\left\| t_i \boldsymbol{x}_{1,i} + \sqrt{1-t_i^2} \boldsymbol{x}_{0,i} \right\|^2 \\
&\leq 2B^2 + \frac{4d}{(1-t_i^2)^2} + \frac{4t_i^2}{(1-t_i^2)^2}\left( 2t_i^2 \left\| \boldsymbol{x}_{1,i} \right\|^2 + 2(1-t_i^2) \left\| \boldsymbol{x}_{0,i} \right\|^2 \right) \\
&\leq 2B^2 + \frac{4d}{(1-T^2)^2} + \frac{4}{(1-T^2)^2}\left( 2d + 2\left\| \boldsymbol{x}_{0,i} \right\|^2 \right) \\
&= 2B^2 + \frac{12d}{(1-T^2)^2} + \frac{8}{(1-T^2)^2}\left\| \boldsymbol{x}_{0,i} \right\|^2.
\end{aligned}
\end{align}
On the other hand, since $\left\| \boldsymbol{x}_{t,i} \right\|_{\infty} \leq t_i \left\| \boldsymbol{x}_{1,i} \right\|_{\infty} + \sqrt{1-t_i^2} \left\| \boldsymbol{x}_{0,i} \right\|_{\infty} \leq 1 + \left\| \boldsymbol{x}_{0,i} \right\|_{\infty}$, we have
\begin{align}
\label{eq:gen 4}
\begin{aligned}
\{ \left\| \boldsymbol{x}_{t,i} \right\|_{\infty} > R \} \subseteq \{ \left\| \boldsymbol{x}_{0,i} \right\|_{\infty} > R-1 \}.
\end{aligned}
\end{align}
Denote the $k$-coordinate of $\boldsymbol{x}_{0,i}$ by $x_{0,i}^{(k)}$, we have the following upper bound for the tail probability:
\begin{align}
\label{eq:gen 5}
\begin{aligned}
\mathbb{P}\left(\left\|\boldsymbol{x}_{0,i}\right\|_{\infty}>R-1\right) & =\mathbb{P}\left(\max _{k=1, \ldots, d}\left|x_{0,i}^{(k)}\right|>R-1\right) \\
& =\mathbb{P}\left(\bigcup_{k=1}^d\left\{\left|x_{0,i}^{(k)}\right|>R-1\right\}\right) \\
& \leq \sum_{k=1}^d \mathbb{P}\left(\left|x_{0,i}^{(k)}\right|>R-1\right) \\
& \leq 2 d \exp \left(-\frac{(R-1)^2}{2}\right).
\end{aligned}
\end{align}
Combining (\ref{eq:gen 3}), (\ref{eq:gen 4}), (\ref{eq:gen 5}) and the Cauchy-Schwartz inequality, we obtain
\begin{align*}
&\E_{\mathcal{X}} \left[\frac{1}{n} \sum_{i=1}^n \left(\left\| \boldsymbol{v}\left(\boldsymbol{x}_{t,i}, t_i\right) - \boldsymbol{v}^*\left(\boldsymbol{x}_{t,i}, t_i\right) \right\|^2 - \left\| \widehat{\boldsymbol{v}}\left(\boldsymbol{x}_{t,i}, t_i\right) - \boldsymbol{v}^*\left(\boldsymbol{x}_{t,i}, t_i\right) \right\|^2\right)\mathbbm{1}\{\|\boldsymbol{x}_{t,i}\|_{\infty} > R\} \right] \\
\leq & \E_{\mathcal{X}} \left[\frac{1}{n} \sum_{i=1}^n \left( 4B^2 + \frac{24d}{(1-T^2)^2} + \frac{16}{(1-T^2)^2}\left\| \boldsymbol{x}_{0,i} \right\|^2\right) \mathbbm{1}\{\left\| \boldsymbol{x}_{0,i} \right\|_{\infty} > R-1\} \right] \\
= & \frac{1}{n} \sum_{i=1}^n \left( 4B^2 + \frac{24d}{(1-T^2)^2} \right)\E_{\boldsymbol{x}_{0,i}} \left[  \mathbbm{1}\{\left\| \boldsymbol{x}_{0,i} \right\|_{\infty} > R-1\} \right] \\
& + \frac{1}{n} \sum_{i=1}^n \frac{16}{(1-T^2)^2} \E_{\boldsymbol{x}_{0,i}} \left[ \left\| \boldsymbol{x}_{0,i} \right\|^2 \mathbbm{1}\{\left\| \boldsymbol{x}_{0,i} \right\|_{\infty} > R-1\} \right] \\
\leq & \frac{1}{n} \sum_{i=1}^n \left( 4B^2 + \frac{24d}{(1-T^2)^2} \right) \mathbb{P} \left( \left\| \boldsymbol{x}_{0,i} \right\|_{\infty} > R-1 \right) \\
& + \frac{1}{n} \sum_{i=1}^n \frac{16}{(1-T^2)^2} \E_{\boldsymbol{x}_{0,i}} \left[ \left\| \boldsymbol{x}_{0,i} \right\|^4 \right]^{1/2} \mathbb{P} \left( \left\| \boldsymbol{x}_{0,i} \right\|_{\infty} > R-1 \right)^{1/2} \\
\leq & \left( 8dB^2 + \frac{48d^2}{(1-T^2)^2} \right)  \exp \left(-\frac{\left(R-1\right)^2}{2}\right) + \frac{48d^2}{(1-T^2)^2}  \exp\left(-\frac{\left(R-1\right)^2}{4}\right) \\
\leq & \left( 8dB^2 + \frac{96d^2}{(1-T^2)^2} \right)  \exp \left(-\frac{\left(R-1\right)^2}{4}\right).
\end{align*}
Substituting into (\ref{eq:gen 2}) and taking the infimum over all $\boldsymbol{v} \in \mathcal{T}$, we have
\begin{align}
\label{eq:gen 8}
\begin{aligned}
\E_{\mathcal{X}} \left[\Rmnum{2}\right] \leq & \inf_{\boldsymbol{v} \in \mathcal{T}} \E_{t,X_t} \left[ \left\|\boldsymbol{v}(X_t, t)-\boldsymbol{v}^*(X_t, t)\right\|^2 \mathbbm{1}\{\|X_t\|_{\infty} \leq R\} \right] \\
& + \left( 8dB^2 + \frac{96d^2}{(1-T^2)^2} \right)  \exp \left(-\frac{\left(R-1\right)^2}{4}\right).
\end{aligned}
\end{align}

$\bullet$ \textbf{Bound of $\Rmnum{3}$.} Similarly, we have
\begin{align}
\label{eq:gen 9}
\begin{aligned}
\Rmnum{3} &= \mathbb{E}_{t, X_t} \left[ \left\|\widehat{\boldsymbol{v}}(X_t, t)-\boldsymbol{v}^*(X_t, t)\right\|^2 \mathbbm{1}\{\|X_t\|_{\infty} > R\} \right] \\
&\leq \mathbb{E}_{X_0} \left[ \left(2B^2 + \frac{12d}{(1-T^2)^2}  + \frac{8}{(1-T^2)^2}\left\| X_0 \right\|^2\right) \mathbbm{1}\{\|X_0\|_{\infty} > R-1\} \right] \\
&\leq \left(2B^2 + \frac{12d}{(1-T^2)^2} \right) \mathbb{P}\left( \|X_0\|_{\infty} > R-1 \right) \\
&~~~ + \frac{8}{(1-T^2)^2} \mathbb{E}_{X_0}[\left\| X_0 \right\|^4]^{1/2} \mathbb{P}\left( \|X_0\|_{\infty} > R-1 \right)^{1/2} \\
&\leq \left(4dB^2 + \frac{24d^2}{(1-T^2)^2} \right) \exp \left(-\frac{(R-1)^2}{2}\right) + \frac{24d^2}{(1-T^2)^2}  \exp \left(-\frac{(R-1)^2}{4}\right) \\
&\leq \left(4dB^2 + \frac{48d^2}{(1-T^2)^2} \right) \exp \left(-\frac{(R-1)^2}{4}\right).
\end{aligned}
\end{align}

Combining Corollary \ref{coro: app true vd}, (\ref{eq:gen 14}), (\ref{eq:gen 8}), (\ref{eq:gen 9}) and Lemma \ref{lemma:gen 6}, we get
\begin{align*}
&\E_{\mathcal{X}} \left[\mathcal{L}(\widehat{\boldsymbol{v}}) - \mathcal{L}(\boldsymbol{v}^*)\right] \\
&= \E_{\mathcal{X}} \left[\Rmnum{1}\right] + 3\E_{\mathcal{X}} \left[\Rmnum{2}\right] + \Rmnum{3} \\
&\leq \frac{8}{n}\left(B^2 + \frac{d(1+R)^2}{(1-T^2)^2}\right) \mathbb{E}_\mathcal{X} \log \mathcal{N}\left(\delta, \mathcal{H}, d_{\mathcal{X}, 1}\right)+8 \delta \\
&~~~ + 3\inf_{\boldsymbol{v} \in \mathcal{T}} \E_{t,X_t} \left[ \left\|\boldsymbol{v}(X_t, t)-\boldsymbol{v}^*(X_t, t)\right\|^2 \mathbbm{1}\{\|X_t\|_{\infty} \leq R\} \right] \\
&~~~ + \left( 28dB^2 + \frac{336d^2}{(1-T^2)^2} \right)  \exp \left(-\frac{\left(R-1\right)^2}{4}\right) \\
&= \mathcal{O} \Bigg( \frac{1}{n}\left(B^2 + \frac{(1+R)^2}{(1-T)^2}\right) N^2 J \log \left( \max \left\{N, h, d_k, d_v, d_{f f}\right\} \right) \log \frac{ n(B^2 + \frac{(1+R) B}{1-T})}{\delta} \\
&\hspace{1cm}  + \delta + \varepsilon^2 + \left(B^2 + \frac{1}{(1-T)^2}\right) \exp \left(-\frac{\left(R-1\right)^2}{4}\right) \Bigg).
\end{align*}

$\bullet$ \textbf{Balancing error terms.} Based on our choice of $\mathcal{T}$ in Corollary \ref{coro: app true vd}, setting $\varepsilon = n^{-\frac{1}{d+3}}, \delta = \varepsilon^2$ and $R = \sqrt{\frac{8}{d+3} \log n}+1$ gives rise to
\begin{align*}
\E_{\mathcal{X}} \left[\mathcal{L}(\widehat{\boldsymbol{v}}) - \mathcal{L}(\boldsymbol{v}^*)\right] = \widetilde{\mathcal{O}}\left( \frac{1}{(1-T)^{3d+5}} n^{-\frac{2}{d+3}} \right),
\end{align*}
where we omit factors in $d, \log n, \log (1-T)$.
\end{proof}

\subsection{Covering number evaluation}

\begin{lemma}
\label{lemma:gen 6}
$\sup_{\mathcal{X}} \log \mathcal{N}\left(\delta, \mathcal{H}, d_{\mathcal{X}, 1}\right) = \mathcal{O}(N^2 J \log \left( \max \left\{N, h, d_k, d_v, d_{f f}\right\} \right) \log \frac{ n(B^2 + \frac{(1+R) B}{1-T})}{\delta})$, where $d_{\mathcal{X}, 1} (h, h^\prime) = \frac{1}{n} \sum_{i=1}^n |h(\boldsymbol{x}_{t,i}, t_i)-h^\prime (\boldsymbol{x}_{t,i}, t_i)|$.
\end{lemma}

\begin{proof}
Denote the $i$-coordinate of $\boldsymbol{v}$ by $v^{(i)}$ and define $\mathcal{T}^{(i)} := \{ v^{(i)}: \boldsymbol{v} \in \mathcal{T} \}$. Let $\delta>0$. Denote by $\{v^{(i)}_k\}_{k=1}^{\mathcal{N}^{(i)}}$ the centers of a minimal $\delta / (2d(d+1)B+ \frac{4d^{3/2}(1+R)}{1-T^2}  )$-covering of $\mathcal{T}^{(i)}$ with respect to the metric
\begin{align*}
d_{\mathcal{X}, 1} \left(v^{(i)}, v^{(i)\prime}\right)=\frac{1}{n} \sum_{i=1}^n\left|v^{(i)}\left(\boldsymbol{x}_{t, i}, t_i\right)-v^{(i)\prime}\left(\boldsymbol{x}_{t, i}, t_i\right)\right|.
\end{align*}
Triangle inequality gives that for each $v^{(i)}_k$ there exists a $\widetilde{v}^{(i)}_k \in \mathcal{T}^{(i)}$ such that $\{\widetilde{v}^{(i)}_k\}_{k=1}^{\mathcal{N}^{(i)}}$ is an \textit{interior} $\delta / (d(d+1)B+ \frac{2d^{3/2}(1+R)}{1-T^2})$-cover of $\mathcal{T}^{(i)}$. For any $h \in \mathcal{H}$ with $h(\boldsymbol{x}_t, t) = \| \boldsymbol{v}(\boldsymbol{x}_t, t)-\boldsymbol{v}^*(\boldsymbol{x}_t, t) \|^2 \mathbbm{1}\{\|\boldsymbol{x}_t\|_{\infty} \leq R\}$, where $\boldsymbol{v} \in \mathcal{T}$, by the cover property of each component, there exists a $k_i \in \{1, \ldots, \mathcal{N}^{(i)} \}$ such that $d_{\mathcal{X}, 1}( v^{(i)}, \widetilde{v}^{(i)}_{k_i} ) \leq \delta / (d(d+1)B+ \frac{2d^{3/2}(1+R)}{1-T^2})$. We denote $h_k$ as determined by $\boldsymbol{v}_k := ( \widetilde{v}^{(1)}_{k_1}, \ldots, \widetilde{v}^{(d)}_{k_d} )$. Then we have
\begin{align*}
\begin{aligned}
& d_{\mathcal{X}, 1}\left(h, h_k\right) \\
& =\frac{1}{n} \sum_{i=1}^n\left|h\left(\boldsymbol{x}_{t,i}, t_i\right)-h_k\left(\boldsymbol{x}_{t,i}, t_i\right)\right| \\
& =\frac{1}{n} \sum_{i=1}^n\left|\left\langle \boldsymbol{v}\left(\boldsymbol{x}_{t,i}, t_i\right) + \boldsymbol{v}_k\left(\boldsymbol{x}_{t,i}, t_i\right) - 2\boldsymbol{v}^*\left(\boldsymbol{x}_{t,i}, t_i\right), \boldsymbol{v}\left(\boldsymbol{x}_{t,i}, t_i\right) - \boldsymbol{v}_k\left(\boldsymbol{x}_{t,i}, t_i\right) \right\rangle\right| \mathbbm{1}\{\|\boldsymbol{x}_{t,i}\|_{\infty} \leq R\} \\
&\leq \frac{1}{n} \sum_{i=1}^n\left\| \boldsymbol{v}\left(\boldsymbol{x}_{t,i}, t_i\right) + \boldsymbol{v}_k\left(\boldsymbol{x}_{t,i}, t_i\right) - 2\boldsymbol{v}^*\left(\boldsymbol{x}_{t,i}, t_i\right)\right\| \cdot \left\| \boldsymbol{v}\left(\boldsymbol{x}_{t,i}, t_i\right) - \boldsymbol{v}_k\left(\boldsymbol{x}_{t,i}, t_i\right) \right\| \mathbbm{1}\{\|\boldsymbol{x}_{t,i}\|_{\infty} \leq R\} \\
&\leq \left( (d+1)B+ \frac{2\sqrt{d}(1+R)}{1-T^2} \right)  \frac{1}{n} \sum_{i=1}^n \left\| \boldsymbol{v}\left(\boldsymbol{x}_{t,i}, t_i\right) - \boldsymbol{v}_k\left(\boldsymbol{x}_{t,i}, t_i\right) \right\| \\ 
&\leq \left( (d+1)B+ \frac{2\sqrt{d}(1+R)}{1-T^2}  \right)  \frac{1}{n} \sum_{i=1}^n \left\| \boldsymbol{v}\left(\boldsymbol{x}_{t,i}, t_i\right) - \boldsymbol{v}_k\left(\boldsymbol{x}_{t,i}, t_i\right) \right\|_1 \\
&= \left( (d+1)B+ \frac{2\sqrt{d}(1+R)}{1-T^2}  \right) \sum_{i=1}^d d_{\mathcal{X}, 1}\left( v^{(i)}, \widetilde{v}^{(i)}_{k_i} \right) \\
&\leq \delta,
\end{aligned}
\end{align*}
where in the first inequality we use the Cauchy-Schwartz inequality, in the third inequality we use $\|\cdot\| \leq \|\cdot\|_1$, and in the third equality we rearrange the order of summation. It follows that $\{h_k\}$ forms a $\delta$-covering of $\mathcal{H}$. Hence, 
\begin{align*}
\mathcal{N}\left(\delta, \mathcal{H}, d_{\mathcal{X}, 1}\right) &\leq \prod_{i=1}^d \mathcal{N}\left(\frac{\delta}{2d(d+1)B+ \frac{4d^{3/2}(1+R)}{1-T^2}}, \mathcal{T}^{(i)}, d_{\mathcal{X}, 1}\right). 
\end{align*}
Using the fact that $\mathcal{T}^{(i)} \subseteq \mathcal{T}_{d+1,1}\left(N, h, d_k, d_v, d_{ff}, B, J, \gamma\right)$ and $d_{\mathcal{X}, 1}(f, f^\prime) \leq d_{\mathcal{X}, \infty}(f, f^\prime)$ for any $f, f^\prime$, we obtain
\begin{align*}
& \log \mathcal{N}\left(\delta, \mathcal{H}, d_{\mathcal{X}, 1}\right) \\
& \leq \sum_{i=1}^d \log \mathcal{N}\left(\frac{\delta}{2d(d+1)B+ \frac{4d^{3/2}(1+R)}{1-T^2}}, \mathcal{T}^{(i)}, d_{\mathcal{X}, 1}\right) \\
&\leq \sum_{i=1}^d \log \mathcal{N}\left(\frac{\delta}{2d(d+1)B+ \frac{4d^{3/2}(1+R)}{1-T^2}}, \mathcal{T}_{d+1,1}\left(N, h, d_k, d_v, d_{ff}, B, J, \gamma\right), d_{\mathcal{X}, \infty}\right) \\
&\leq c_{10} N^2 J \log \left( \max \left\{N, h, d_k, d_v, d_{f f}\right\} \right) \log \frac{ n\left(B^2 + \frac{(1+R) B}{1-T}\right)}{\delta},
\end{align*}
where in the last inequality we use Lemma \ref{lemma:gen 5}.
\end{proof}

\begin{lemma}
\label{lemma:gen 5}
Let $\mathcal{X} = \{\boldsymbol{x}_1, \ldots, \boldsymbol{x}_n\}$. Then
$\sup_{\mathcal{X}} \log \mathcal{N}\left(\delta, \mathcal{T}_{d,1}\left(N, h, d_k, d_v, d_{ff}, B, J, \gamma\right), d_{\mathcal{X}, \infty}\right)
= \mathcal{O}(N^2 J \log \left( \max \left\{N, h, d_k, d_v, d_{ff}\right\} \right) \log\frac{B n}{\delta})$, where $d_{\mathcal{X}, \infty} \left(\phi, \phi^{\prime}\right) = \max_{i = 1, \ldots, n} |\phi (\boldsymbol{x}_i) - \phi^{\prime} (\boldsymbol{x}_i)|$.
\end{lemma}

\begin{proof}
For simplicity, let $\mathcal{T}_{d,1}$ denote $\mathcal{T}_{d,1}(N, h, d_k, d_v, d_{ff}, B, J, \gamma)$. We initially establish an upper bound on the pseudo-dimension of subsets of $\mathcal{T}_{d,1}$, specifically focusing on those subsets where the nonzero components are fixed in position. Subsequently, leveraging known results enables us to control the covering number via the pseudo-dimension. The proof is a modification of the proof of Lemma 8 in \cite{gurevych2022rate} and Theorem 6 in \cite{bartlett2019nearly}.

Based on the definition of $\mathcal{T}_{d,1}$, in all parameters determining a function $v \in \mathcal{T}_{d,1}$, only at most $J$ components are permitted to be nonzero. We fix the positions of these nonzero parameters and denote by $\boldsymbol{\theta}$ the vector in $\mathbb{R}^J$ comprising all potential nonzero parameter values. Then define
\begin{align*}
\mathcal{V}=\left\{v(\cdot, \boldsymbol{\theta}): \mathbb{R}^{d} \rightarrow \mathbb{R}: \boldsymbol{\theta} \in \mathbb{R}^J\right\}.
\end{align*} 
To estimate the pseudo-dimension of $\mathcal{V}$, denoted as $\operatorname{Pdim}(\mathcal{V})$, we consider a set of points $(\boldsymbol{x}_1, y_1), \ldots, (\boldsymbol{x}_m, y_m) \in \mathbb{R}^{d} \times \mathbb{R}$ that satisfy 
\begin{align*}
\left|\left\{\left(\operatorname{sgn}\left(v(\boldsymbol{x}_1, \boldsymbol{\theta}) - y_1\right), \ldots, \operatorname{sgn}\left(v(\boldsymbol{x}_m, \boldsymbol{\theta}) - y_m\right)\right): v \in \mathcal{V} \right\}\right|=2^m.
\end{align*}
It suffices to bound $m$.

We view the parameters as the variables for any network $v \in \mathcal{V}$ when input is fixed. In the following, we construct a sequence of partitions $\mathcal{P}_{0}, \mathcal{P}_{1}, \ldots, \mathcal{P}_{N+1}$ of $\mathbb{R}^J$ by successive refinement such that in the last partition for all $S \in \mathcal{P}_{N+1}$ we have
\begin{align*}
v\left(\boldsymbol{x}_1, \boldsymbol{\theta}\right), \ldots, v\left(\boldsymbol{x}_m, \boldsymbol{\theta}\right)
\end{align*}
are polynomials as functions of $\boldsymbol{\theta}$ of degree at most $9^{N+2}$ for $\boldsymbol{\theta} \in S$. Then application of Lemma \ref{lemma:gen 3} and Lemma \ref{lemma:gen 4} yields an upper bound for $m$.

To begin with, set $\mathcal{P}_0 = \{\mathbb{R}^J\}$. As defined in (\ref{eq: app 1}), all components of $Z_0$ are polynomials as functions of $\boldsymbol{\theta}$ of degree at most $1 \leq 9$ in $\boldsymbol{\theta}$ for $\boldsymbol{\theta} \in \mathbb{R}^J$. Let $r \in\{1, \ldots, N\}$ and assume that for all $S \in \mathcal{P}_{r-1}$ all components in $Z_{r-1}$ are polynomials as functions of $\boldsymbol{\theta}$ of degree at most $9^{r}$ in $\boldsymbol{\theta}$ for $\boldsymbol{\theta} \in S$. Then all components in 
\begin{align*}
\boldsymbol{q}_{r-1, s, i} := W_{Q, r, s} \boldsymbol{z}_{r-1, i}\quad \text{ and } \quad \boldsymbol{k}_{r-1, s, i} := W_{K, r, s} \boldsymbol{z}_{r-1, i}
\end{align*}
are polynomials of degree at most $9^{r}+1$ on each set $S \in \mathcal{P}_{r-1}$. Consequently, for $S \in \mathcal{P}_{r-1}$, 
\begin{align*}
<\boldsymbol{q}_{r-1, s, i}, \boldsymbol{k}_{r-1, s, j}>
\end{align*}
is a polynomial of degree at most $2 \cdot 9^{r}+2$ for $\boldsymbol{\theta} \in S$. Application of Lemma \ref{lemma:gen 3} yields that
\begin{align*}
<\boldsymbol{q}_{r-1, s, i}, \boldsymbol{k}_{r-1, s, j_1}>-<\boldsymbol{q}_{r-1, s, i}, \boldsymbol{k}_{r-1, s, j_2}> \quad \left(s \in\{1, \ldots, h\}, i, j_1, j_2 \in\{1, \ldots, l\}\right)
\end{align*}
has at most
\begin{align*}
\Delta_1 = 2 \left(\frac{2e  h l^3 \left(2 \cdot 9^{r}+2\right)}{J}\right)^J
\end{align*}
difference sign patterns. If we partition in each set in $\mathcal{P}_{r-1}$ according to these sign patterns in $\Delta_1$ subsets such that all these polynomials have the same signs within each refined region, then on each set in the new partition all elements in
\begin{align*} 
\left((W_{K, r, s} z_{r-1})^{\top}(W_{Q, r, s} z_{r-1})\right) \odot \sigma_H \left((W_{K, r, s} z_{r-1})^{\top}(W_{Q, r, s} z_{r-1})\right)
\end{align*}
are fixed polynomials of degree at most $2 \cdot 9^{r}+2$, indicating that all components in
\begin{align*}
Y_r = F^{(SA)}(Z_{r-1})
\end{align*}
are polynomials of degree at most $3 \cdot 9^{r}+4$ on each refined region. On each set within the new partition every component of
\begin{align*}
W_{r, 1} Y_r + \boldsymbol{b}_{r, 1} \mathbbm{1}_l^{\top}
\end{align*}
is a polynomial of degree at most $3 \cdot 9^{r}+5$. By applying Lemma \ref{lemma:gen 3} once again, we can refine each set in this partition into
\begin{align*}
\Delta_2 = 2 \left(\frac{2e d_{ff} \left(3 \cdot 9^{r}+5\right)}{J}\right)^J
\end{align*}
sets such that all components in $W_{r, 1} Y_r + \boldsymbol{b}_{r, 1} \mathbbm{1}_l^{\top}$ have the same sign patterns within the refined partition. We refer to the partition obtained by refining $\mathcal{P}_{r-1}$ twice as $\mathcal{P}_r$. Since on each set of $\mathcal{P}_r$ the sign of all components does not change, we can conclude that all components in
\begin{align*}
\sigma\left(W_{r, 1} Y_r + \boldsymbol{b}_{r, 1} \mathbbm{1}_l^{\top}\right)
\end{align*}
are either equal to zero or equal to a polynomial of degree at most $3 \cdot 9^{r}+5$. Consequently we have on each set in $\mathcal{P}_r$ all components of
\begin{align*}
Z_r = F^{(FF)}(Y_r)
\end{align*}
are equal to a polynomial of degree at most $3 \cdot 9^{r}+6 \leq 9^{r+1}$.

Proceeding in this way we obtain a partition $\mathcal{P}_{N}$ of $\mathbb{R}^J$ such that on each set $S \in \mathcal{P}_{N}$ all components of
\begin{align*}
Z_N
\end{align*}
are fixed polynomials of $\boldsymbol{\theta} \in S$ of degree no more than $9^{N+1}$, and hence for all $k \in \{1, \ldots, m\}$
\begin{align*}
v\left(\boldsymbol{x}_k,  \boldsymbol{\theta}\right) - y_k
\end{align*}
are polynomials of degree at most $9^{N+1}+1 \leq 9^{N+2}$ in $\boldsymbol{\theta}$ for $\boldsymbol{\theta} \in S$.

According to the refinement process, we have
\begin{align*}
\left|\mathcal{P}_{N}\right|=\prod_{r=1}^N \frac{\left|\mathcal{P}_r\right|}{\left|\mathcal{P}_{r-1}\right|} \leq \prod_{r=1}^N 2 \left(\frac{2e  h l^3 \left(2 \cdot 9^{r}+2\right)}{J}\right)^J \cdot 2 \left(\frac{2e d_{ff} \left(3 \cdot 9^{r}+5\right)}{J}\right)^J.
\end{align*}
Using that
\begin{align*}
\begin{aligned}
& \left|\left\{\left(\operatorname{sgn}\left(v(\boldsymbol{x}_1, \boldsymbol{\theta}) - y_1\right), \ldots, \operatorname{sgn}\left(v(\boldsymbol{x}_m, \boldsymbol{\theta}) - y_m\right)\right): v \in \mathcal{V} \right\}\right| \\
& \leq \sum_{S \in \mathcal{P}_{N}}\left|\left\{\left(\operatorname{sgn}\left(v\left(\boldsymbol{x}_1, \boldsymbol{\theta}\right) - y_1\right), \ldots, \operatorname{sgn}\left(v\left(\boldsymbol{x}_m, \boldsymbol{\theta}\right) - y_m\right)\right): \boldsymbol{\theta} \in S\right\}\right|,
\end{aligned}
\end{align*}
we apply Lemma \ref{lemma:gen 3} and obtain
\begin{align*}
\begin{aligned}
2^m & = \left|\left\{\left(\operatorname{sgn}\left(v(\boldsymbol{x}_1, \boldsymbol{\theta}) - y_1\right), \ldots, \operatorname{sgn}\left(v(\boldsymbol{x}_m, \boldsymbol{\theta}) - y_m\right)\right): v \in \mathcal{V} \right\}\right| \\
& \leq \left|\mathcal{P}_{N}\right| \cdot 2 \left(\frac{2e m 9^{N+2}}{J}\right)^J \\
& \leq 2 \left(\frac{2e m 9^{N+2}}{J}\right)^J \cdot \prod_{r=1}^N 2 \left(\frac{2e  h l^3 \left(2 \cdot 9^{r}+2\right)}{J}\right)^J \cdot 2 \left(\frac{2e d_{ff} \left(3 \cdot 9^{r}+5\right)}{J}\right)^J \\
& \leq 2^{2 N+1} \left(\frac{m 6e  \max \left\{h l^3, d_{f f}\right\} 9^{N+2}}{(2 N+1) J}\right)^{(2 N+1) J}.
\end{aligned}
\end{align*}
Assume $m \geq (2N+1) J$. Due to Lemma \ref{lemma:gen 4}, we have
\begin{align*}
m &\leq (2 N+1) + (2 N+1)J \log_2 \left(12e  \max \left\{h l^3, d_{f f}\right\} 9^{N+2} \log_2 \left(6e  \max \left\{h l^3, d_{f f}\right\} 9^{N+2}\right) \right) \\
&\leq c_{11} N^2 J \log \left( \max \left\{h, d_{f f}\right\} \right),
\end{align*}
which implies 
\begin{align*}
\operatorname{Pdim}(\mathcal{V}) \leq  c_{11} N^2 J \log \left( \max \left\{h, d_{f f}\right\} \right).
\end{align*}

For any fixed $\mathcal{X} = \{\boldsymbol{y}_i\}_{i=1}^n$, since $\{ (v\left(\boldsymbol{y}_{1}\right), \ldots, v\left(\boldsymbol{y}_{n}\right)): v \in \mathcal{V} \} \subseteq \{ \boldsymbol{x} \in \mathbb{R}^n: \|\boldsymbol{x}\|_{\infty} \leq B \}$ can be covered by at most $\lceil\frac{2 B}{\delta}\rceil^n$ balls with radius $\delta$ in $\|\cdot\|_{\infty}$ distance, we always have $\mathcal{N}(\delta, \mathcal{V}, d_{\mathcal{X}, \infty}) \leq \lceil\frac{2 B}{\delta}\rceil^n$. By Theorem 12.2 in \cite{anthony1999neural}, the covering number $\mathcal{N}(\delta, \mathcal{V}, d_{\mathcal{X}, \infty})$ can be bounded by the pseudo-dimension $\operatorname{Pdim}(\mathcal{V})$ through
\begin{align*}
\mathcal{N}(\delta, \mathcal{V}, d_{\mathcal{X}, \infty}) \leq \left(\frac{2 e B n}{\delta \operatorname{Pdim}(\mathcal{V})}\right)^{\operatorname{Pdim}(\mathcal{V})}
\end{align*}
if $n \geq \operatorname{Pdim}(\mathcal{V})$. In any cases,
\begin{align*}
\log \mathcal{N}(\delta, \mathcal{V}, d_{\mathcal{X}, \infty}) \leq \operatorname{Pdim}(\mathcal{V}) \log \frac{2 e B n}{\delta}.
\end{align*}

The functions in the function set $\mathcal{T}_{d, 1}$ depend on at most $c_{12} N h^2 (\max \{d_k, d_v, d_{f f}\})^3$ many parameters, and of these parameters at most $J$ are allowed to be nonzero. It follows that the number of possible ways to select these positions is given by
\begin{align*}
\left(\begin{array}{c}
c_{12} N h^2 \left(\max \left\{d_k, d_v, d_{f f}\right\}\right)^3 \\
J
\end{array}\right) \leq \left(c_{12} N h^2 \left(\max \left\{d_k, d_v, d_{f f}\right\}\right)^3\right)^J.
\end{align*}
Fixing these positions delineates a specific function space $\mathcal{V}$. From this we can conclude
\begin{align*}
& \log \mathcal{N}(\delta, \mathcal{T}_{d, 1}, d_{\mathcal{X}, \infty}) \\
&\leq J \log\left( c_{12} N h^2 \left(\max \left\{d_k, d_v, d_{f f}\right\}\right)^3 \right) + c_{11} N^2 J \log \left( \max \left\{h, d_{f f}\right\} \right) \log \frac{2 e B n}{\delta} \\
&\leq c_{13} N^2 J \log \left( \max \left\{N, h, d_k, d_v, d_{f f}\right\} \right) \log\frac{B n}{\delta}.
\end{align*}
Taking the supremum over $\mathcal{X}$ completes the proof.
\end{proof}

\subsection{Auxiliary lemma}

\begin{lemma}
\label{lemma:gen 1}
Given a Rademacher random variable $\sigma$ takes the values $\{-1,1\}$ equiprobably. We have, for any $\lambda \in \mathbb{R}, \mathbb{E}_\sigma [e^{\lambda \sigma}] \leq e^{\lambda^2 / 2}$.
\end{lemma}

\begin{proof}
By taking expectations and using the power-series expansion for the exponential, we obtain
\begin{align*}
\begin{aligned}
\mathbb{E}_\sigma [e^{\lambda \sigma}]=\frac{1}{2} (e^{-\lambda}+e^\lambda) & =\frac{1}{2}\left(\sum_{k=0}^{\infty} \frac{(-\lambda)^k}{k !}+\sum_{k=0}^{\infty} \frac{(\lambda)^k}{k !}\right) \\
& =\sum_{k=0}^{\infty} \frac{\lambda^{2 k}}{(2 k) !} \\
& \leq 1+\sum_{k=1}^{\infty} \frac{\lambda^{2 k}}{2^k k !} \\
& =e^{\lambda^2 / 2}.
\end{aligned}
\end{align*}
It concludes the proof.
\end{proof}

\begin{lemma}[\cite{anthony1999neural}, Theorem 8.3]
\label{lemma:gen 3}
Suppose $W \leq m$ and let $f_1, \ldots, f_m$ be polynomials of degree at most $D$ in $W$ variables. Define
\begin{align*}
K:=\left|\left\{\left(\operatorname{sgn}\left(f_1(\boldsymbol{x})\right), \ldots, \operatorname{sgn}\left(f_m(\boldsymbol{x})\right)\right): \boldsymbol{x} \in \mathbb{R}^W\right\}\right|.
\end{align*}
Then 
\begin{align*}
K \leq 2 \left(\frac{2 e m D}{W}\right)^W.
\end{align*}
\end{lemma}

\begin{lemma}[\cite{bartlett2019nearly}, Lemma 16]
\label{lemma:gen 4}
Suppose that $2^m \leq 2^L (m R / w)^w$ for some $R \geq 16$ and $m \geq w \geq L \geq 0$. Then,
\begin{align*}
m \leq L + w \log_2 \left(2 R \log_2 R\right).
\end{align*}
\end{lemma}

\section{Discretization Analysis}
\subsection{Estimation Error}
Consider the target continuous flow
\begin{align}
\label{eq: dis 1}
\mathrm{d} X_t(\boldsymbol{x})=\boldsymbol{v}^*\left(X_t(\boldsymbol{x}), t\right) \mathrm{d} t,\  X_0(\boldsymbol{x})=\boldsymbol{x} \sim \pi_0,\  0 \leq t \leq T,
\end{align}
and the estimated continuous flow
\begin{align}
\label{eq: dis 2}
\mathrm{d} \widetilde{X}_t(\boldsymbol{x})=\widehat{\boldsymbol{v}}(\widetilde{X}_t(\boldsymbol{x}), t) \mathrm{d} t,\  \widetilde{X}_0(\boldsymbol{x})=\boldsymbol{x} \sim \pi_0,\  0 \leq t \leq T.
\end{align}
Denote the distribution of $X_t(\boldsymbol{x})$ and $\widetilde{X}_t(\boldsymbol{x})$ by $\pi_t$ and $\widetilde{\pi}_t$, respectively. We have the following estimate of the Wasserstein-2 distance $W_2\left(\pi_T, \widetilde{\pi}_T\right)$.

\begin{proposition}
\label{pro: dis 1}
Suppose Assumption \ref{ass: bounded support gamma} holds. Given $n$ samples $\mathcal{X} = \{\boldsymbol{x}_{1, i}, \boldsymbol{x}_{0, i}, t_i\}_{i=1}^n$ from $\pi_1$, $\pi_0$ and $\mathrm{Unif}[0, T]$, we choose neural network as in Theorem \ref{theorem: gen 1}. Then
\begin{align*}
\mathbb{E}_{\mathcal{X}}\left[ W_2\left(\pi_T, \widetilde{\pi}_T\right) \right]=\widetilde{\mathcal{O}}\left(e^{\gamma_{\boldsymbol{x}}} (1-T)^{-\frac{3d+5}{2}} n^{-\frac{1}{d+3}}\right).
\end{align*}
\end{proposition}

\begin{proof}
Since $X_t(\boldsymbol{x})$ and $\widetilde{X}_t(\boldsymbol{x})$ form a coupling of $\pi_t$ and $\widetilde{\pi}_t$, by the definition of Wasserstein-2 distance, we have
\begin{align*}
W_2^2\left(\pi_t, \widetilde{\pi}_t\right) \leq \int_{R^d}\left\|X_t(\boldsymbol{x})-\widetilde{X}_t(\boldsymbol{x})\right\|^2 \pi_0(\boldsymbol{x}) \mathrm{d} \boldsymbol{x},
\end{align*}
where $X_t$ is the flow map solution of (\ref{eq: dis 1}) with the exact $\boldsymbol{v}^*$ defined in (\ref{eq:gen 13}) and $\widetilde{X}_t$ is the flow map solution of (\ref{eq: dis 2}). Now, we consider the evolution of
\begin{align*}
R_t:=\int_{\mathbb{R}^d}\left\|X_t(\boldsymbol{x})-\widetilde{X}_t(\boldsymbol{x})\right\|^2 \pi_0(\boldsymbol{x}) \mathrm{d} \boldsymbol{x}.
\end{align*}
Differentiating on both sides, we get
\begin{align}
\label{eq: dis 3}
\begin{aligned}
& \frac{\mathrm{d} R_t}{\mathrm{d} t} =2 \int_{\mathbb{R}^d}\left\langle\boldsymbol{v}^*\left(X_t(\boldsymbol{x}), t\right)-\widehat{\boldsymbol{v}}(\widetilde{X}_t(\boldsymbol{x}), t), X_t(\boldsymbol{x})-\widetilde{X}_t(\boldsymbol{x})\right\rangle \pi_0(\boldsymbol{x}) \mathrm{d} \boldsymbol{x} \\
& =2 \int_{\mathbb{R}^d}\left\langle\boldsymbol{v}^*\left(X_t(\boldsymbol{x}), t\right)-\widehat{\boldsymbol{v}}\left(X_t(\boldsymbol{x}), t\right)+\widehat{\boldsymbol{v}}\left(X_t(\boldsymbol{x}), t\right)-\widehat{\boldsymbol{v}}(\widetilde{X}_t(\boldsymbol{x}), t), X_t(\boldsymbol{x})-\widetilde{X}_t(\boldsymbol{x})\right\rangle \pi_0(\boldsymbol{x}) \mathrm{d} \boldsymbol{x}.
\end{aligned}
\end{align}
Using the inequality $2\langle a, b\rangle \leq\|a\|^2+\|b\|^2$, we have
\begin{align}
\label{eq: dis 4}
\begin{aligned}
&2\left\langle\boldsymbol{v}^*\left(X_t(\boldsymbol{x}), t\right)-\widehat{\boldsymbol{v}}\left(X_t(\boldsymbol{x}), t\right), X_t(\boldsymbol{x})-\widetilde{X}_t(\boldsymbol{x})\right\rangle \\
&\leq \left\|\boldsymbol{v}^*\left(X_t(\boldsymbol{x}), t\right)-\widehat{\boldsymbol{v}}\left(X_t(\boldsymbol{x}), t\right)\right\|^2+\|X_t(\boldsymbol{x})-\widetilde{X}_t(\boldsymbol{x})\|^2.
\end{aligned}
\end{align}
Since $\widehat{\boldsymbol{v}} \in \mathcal{T}$ defined in Theorem \ref{theorem: gen 1} is $\gamma_{\boldsymbol{x}}$-Lipschitz continuous w.r.t. $\boldsymbol{x}$, the Cauchy-Schwartz inequality implies
\begin{align}
\label{eq: dis 5}
2\left\langle\widehat{\boldsymbol{v}}\left(X_t(\boldsymbol{x}), t\right)-\widehat{\boldsymbol{v}}(\widetilde{X}_t(\boldsymbol{x}), t), X_t(\boldsymbol{x})-\widetilde{X}_t(\boldsymbol{x})\right\rangle \leq 2 \gamma_{\boldsymbol{x}} \|X_t(\boldsymbol{x})-\widetilde{X}_t(\boldsymbol{x})\|^2.
\end{align}
Combining (\ref{eq: dis 3}), (\ref{eq: dis 4}) and (\ref{eq: dis 5}), we obtain
\begin{align*}
\frac{\mathrm{d} R_t}{\mathrm{d} t} \leq\left(1+2 \gamma_{\boldsymbol{x}}\right) R_t+\int_{\mathbb{R}^d}\left\|\boldsymbol{v}^*\left(X_t(\boldsymbol{x}), t\right)-\widehat{\boldsymbol{v}}\left(X_t(\boldsymbol{x}), t\right)\right\|^2 \pi_0(\boldsymbol{x}) \mathrm{d} \boldsymbol{x}.
\end{align*}
Therefore, by Lemma \ref{lemma: dis 1} and since $R_0=0$, we deduce
\begin{align*}
\begin{aligned}
R_T & \leq e^{1+2 \gamma_{\boldsymbol{x}}} \int_0^T \int_{\mathbb{R}^d}\left\|\boldsymbol{v}^*\left(X_t(\boldsymbol{x}), t\right)-\widehat{\boldsymbol{v}}\left(X_t(\boldsymbol{x}), t\right)\right\|^2 \pi_0(\boldsymbol{x}) \mathrm{d} \boldsymbol{x} \mathrm{d} t \\
& =e^{1+2 \gamma_{\boldsymbol{x}}} \int_0^T\left\|\boldsymbol{v}^*(\cdot, t)-\widehat{\boldsymbol{v}}(\cdot, t)\right\|_{L^2\left(\pi_t\right)}^2 \mathrm{d} t.
\end{aligned}
\end{align*}
By Theorem \ref{theorem: gen 1} and Jensen's inequality, we get the desired result.
\end{proof}

\subsection{Discretization Error} \label{appendix: dis.2}
Now we consider the gap between estimated continuous flow and its discretization:
\begin{align*}
\begin{aligned}
& \mathrm{d} \widetilde{X}_t(\boldsymbol{x})=\widehat{\boldsymbol{v}}(\widetilde{X}_t(\boldsymbol{x}), t) \mathrm{d} t,\  \widetilde{X}_0(\boldsymbol{x})=\boldsymbol{x} \sim \pi_0,\  0 \leq t \leq T, \\
& \mathrm{d} \widehat{X}_t(\boldsymbol{x})=\widehat{\boldsymbol{v}}(\widehat{X}_{t_k}(\boldsymbol{x}), t_k) \mathrm{d} t,\  t_k \leq t \leq t_{k+1},\  k=0,1, \ldots, N-1,\  \widehat{X}_0(\boldsymbol{x})=\boldsymbol{x} \sim \pi_0.
\end{aligned}
\end{align*}
Denote the distribution of $\widetilde{X}_t(\boldsymbol{x})$ and $\widehat{X}_t(\boldsymbol{x})$ by $\widetilde{\pi}_t$ and $\widehat{\pi}_t$, respectively.

\begin{lemma}
\label{lemma: dis 2}
Let $0=t_0<t_1<\cdots<t_N=T$ be the discretization points. For any neural network $\widehat{\boldsymbol{v}}$ in $\mathcal{T}$ defined in (\ref{eq: def tao}), we have
\begin{align*}
W_2\left(\widetilde{\pi}_T, \widehat{\pi}_T\right)=\mathcal{O}\left(e^{\gamma_{\boldsymbol{x}}} (\gamma_{\boldsymbol{x}} B + \gamma_t) \sqrt{\sum_{k=0}^{N-1}\left(t_{k+1}-t_k\right)^3}\right).
\end{align*}
\end{lemma}

\begin{proof}
By the same argument as in the proof of Proposition \ref{pro: dis 1}, we have
\begin{align*}
W_2^2\left(\widetilde{\pi}_t, \widehat{\pi}_t\right) \leq \int_{\mathbb{R}^d}\left\|\widetilde{X}_t(\boldsymbol{x})-\widehat{X}_t(\boldsymbol{x})\right\|^2 \pi_0(\boldsymbol{x}) \mathrm{d}  \boldsymbol{x}.
\end{align*}
Now, we consider the evolution of
\begin{align*}
L_t:=\int_{\mathbb{R}^d}\left\|\widetilde{X}_t(\boldsymbol{x})-\widehat{X}_t(\boldsymbol{x})\right\|^2 \pi_0(\boldsymbol{x}) \mathrm{d} \boldsymbol{x}.
\end{align*}
Since $\widehat{X}_t(\boldsymbol{x})$ is piece-wise linear, we consider the evolution of $L_t$ on each split interval $\left[t_k, t_{k+1}\right]$. On interval $\left[t_k, t_{k+1}\right]$, we have
\begin{align}
\frac{\mathrm{d} L_t}{\mathrm{~d} t}= & \int_{\mathbb{R}^d} 2\left\langle\widehat{\boldsymbol{v}}(\widetilde{X}_t(\boldsymbol{x}), t)-\widehat{\boldsymbol{v}}(\widehat{X}_{t_k}(\boldsymbol{x}), t_k), \widetilde{X}_t(\boldsymbol{x})-\widehat{X}_t(\boldsymbol{x})\right\rangle \pi_0(\boldsymbol{x}) \mathrm{d} \boldsymbol{x} \nonumber \\
= & \int_{\mathbb{R}^d} 2\left\langle\widehat{\boldsymbol{v}}(\widetilde{X}_t(\boldsymbol{x}), t)-\widehat{\boldsymbol{v}}(\widehat{X}_t(\boldsymbol{x}), t), \widetilde{X}_t(\boldsymbol{x})-\widehat{X}_t(\boldsymbol{x})\right\rangle \pi_0(\boldsymbol{x}) \mathrm{d} \boldsymbol{x} \label{eq: dis 6} \\ 
& +\int_{\mathbb{R}^d} 2\left\langle\widehat{\boldsymbol{v}}(\widehat{X}_t(\boldsymbol{x}), t)-\widehat{\boldsymbol{v}}(\widehat{X}_{t_k}(\boldsymbol{x}), t), \widetilde{X}_t(\boldsymbol{x})-\widehat{X}_t(\boldsymbol{x})\right\rangle \pi_0(\boldsymbol{x}) \mathrm{d} \boldsymbol{x} \label{eq: dis 7} \\ 
& +\int_{\mathbb{R}^d} 2\left\langle\widehat{\boldsymbol{v}}(\widehat{X}_{t_k}(\boldsymbol{x}), t)-\widehat{\boldsymbol{v}}(\widehat{X}_{t_k}(\boldsymbol{x}), t_k), \widetilde{X}_t(\boldsymbol{x})-\widehat{X}_t(\boldsymbol{x})\right\rangle \pi_0(\boldsymbol{x}) \mathrm{d} \boldsymbol{x}. \label{eq: dis 8}
\end{align}
For (\ref{eq: dis 6}), by Cauchy-Schwartz inequality and the fact that $\widehat{\boldsymbol{v}}$ is $\gamma_{\boldsymbol{x}}$-Lipschitz continuous w.r.t. $\boldsymbol{x}$, we get
\begin{align}
\begin{aligned}
\label{eq: dis 9}
& \int_{\mathbb{R}^d} 2\left\langle\widehat{\boldsymbol{v}}(\widetilde{X}_t(\boldsymbol{x}), t)-\widehat{\boldsymbol{v}}(\widehat{X}_t(\boldsymbol{x}), t), \widetilde{X}_t(\boldsymbol{x})-\widehat{X}_t(\boldsymbol{x})\right\rangle \pi_0(\boldsymbol{x}) \mathrm{d} \boldsymbol{x} \\
& \leq 2 \gamma_{\boldsymbol{x}} \int_{R^d}\left\|\widetilde{X}_t(\boldsymbol{x})-\widehat{X}_t(\boldsymbol{x})\right\|^2 \pi_0(\boldsymbol{x}) \mathrm{d} \boldsymbol{x}.
\end{aligned}
\end{align}
For (\ref{eq: dis 7}), note that $\widehat{X}_t(\boldsymbol{x})=\widehat{X}_{t_k}(\boldsymbol{x})+\left(t-t_k\right) \widehat{\boldsymbol{v}}(\widehat{X}_{t_k}(\boldsymbol{x}), t_k)$, we use the inequality $2\langle a, b\rangle \leq\|a\|^2+\|b\|^2$ and the fact that $\widehat{\boldsymbol{v}}$ is $\gamma_{\boldsymbol{x}}$-Lipschitz continuous w.r.t. $\boldsymbol{x}$ to get
\begin{align}
\label{eq: dis 10}
\begin{aligned}
& \int_{\mathbb{R}^d} 2\left\langle\widehat{\boldsymbol{v}}(\widehat{X}_t(\boldsymbol{x}), t)-\widehat{\boldsymbol{v}}(\widehat{X}_{t_k}(\boldsymbol{x}), t), \widetilde{X}_t(\boldsymbol{x})-\widehat{X}_t(\boldsymbol{x})\right\rangle \pi_0(\boldsymbol{x}) \mathrm{d} \boldsymbol{x} \\
\leq & \int_{\mathbb{R}^d}\left\|\widehat{\boldsymbol{v}}(\widehat{X}_t(\boldsymbol{x}), t)-\widehat{\boldsymbol{v}}(\widehat{X}_{t_k}(\boldsymbol{x}), t)\right\|^2 \pi_0(\boldsymbol{x}) \mathrm{d} \boldsymbol{x}+\int_{\mathbb{R}^d}\left\|\widetilde{X}_t(\boldsymbol{x})-\widehat{X}_t(\boldsymbol{x})\right\|^2 \pi_0(\boldsymbol{x}) \mathrm{d} \boldsymbol{x} \\
\leq & \gamma_{\boldsymbol{x}}^2\left(t-t_k\right)^2\|\widehat{\boldsymbol{v}}\|_{L^{\infty}}^2+L_t \\
\leq & \gamma_{\boldsymbol{x}}^2\left(t-t_k\right)^2 B^2+L_t,
\end{aligned}
\end{align}
where $B$ is the parameter of the neural networks in (\ref{eq: def tao}). For (\ref{eq: dis 8}), the fact that $\widehat{\boldsymbol{v}}$ is $\gamma_t$-Lipschitz continuous w.r.t. $t$ implies
\begin{align}
\label{eq: dis 11}
\begin{aligned}
& \int_{\mathbb{R}^d} 2\left\langle\widehat{\boldsymbol{v}}(\widehat{X}_{t_k}(\boldsymbol{x}), t)-\widehat{\boldsymbol{v}}(\widehat{X}_{t_k}(\boldsymbol{x}), t_k), \widetilde{X}_t(\boldsymbol{x})-\widehat{X}_t(\boldsymbol{x})\right\rangle \pi_0(\boldsymbol{x}) \mathrm{d} \boldsymbol{x} \\
\leq & \int_{\mathbb{R}^d}\left\|\widetilde{X}_t(\boldsymbol{x})-\widehat{X}_t(\boldsymbol{x})\right\|^2 \pi_0(\boldsymbol{x}) \mathrm{d} \boldsymbol{x} + \gamma_t^2\left(t-t_k\right)^2.
\end{aligned}
\end{align}
Combining (\ref{eq: dis 9}), (\ref{eq: dis 10}) and (\ref{eq: dis 11}), we obtain
\begin{align*}
\frac{\mathrm{d} L_t}{\mathrm{d} t} \leq \left(2 \gamma_{\boldsymbol{x}}+2\right) L_t + \left(\gamma_{\boldsymbol{x}}^2 B^2 + \gamma_t^2\right)\left(t-t_k\right)^2, \quad \text{ on }\left[t_k, t_{k+1}\right].
\end{align*}
Again, by Lemma \ref{lemma: dis 1}, we obtain
\begin{align*}
e^{-\left(2 \gamma_{\boldsymbol{x}}+2\right) t_{k+1}} L_{t_{k+1}}-e^{-\left(2 \gamma_{\boldsymbol{x}}+2\right) t_k} L_{t_k} \leq \frac{1}{3}\left(\gamma_{\boldsymbol{x}}^2 B^2 + \gamma_t^2\right)\left(t_{k+1}-t_k\right)^3.
\end{align*}
Summing over $k$ and noting that $t_N=T$, we get
\begin{align*}
L_T \leq \frac{1}{3} e^{2\left(\gamma_{\boldsymbol{x}}+1\right) T}\left(\gamma_{\boldsymbol{x}}^2 B^2 + \gamma_t^2\right) \sum_{k=0}^{N-1}\left(t_{k+1}-t_k\right)^3.
\end{align*}
Thus, we have
\begin{align*}
W_2\left(\widetilde{\pi}_T, \widehat{\pi}_T\right)=\mathcal{O}\left(e^{\gamma_{\boldsymbol{x}}} (\gamma_{\boldsymbol{x}} B + \gamma_t) \sqrt{\sum_{k=0}^{N-1}\left(t_{k+1}-t_k\right)^3}\right).
\end{align*}
\end{proof}

\begin{lemma}
\label{lemma: dis 3}
Suppose Assumption \ref{ass: bounded support gamma} holds, we have
\begin{align*}
W_2\left(\pi_T, \pi_1\right) = \mathcal{O} (1-T).
\end{align*}
\end{lemma}

\begin{proof}
We consider the error from early stopping. Note that $X_T$ and $X_1$ form a coupling of $\pi_T$ and $\pi_1$, by the definition of Wasserstein-2 distance, we obtain
\begin{align*}
W_2\left(\pi_T, \pi_1\right) &\leq \mathbb{E} [\|X_T-X_1\|^2]^{1 / 2} \\
&= \mathbb{E} [ \|(1-T^2)^{1/2} X_0 - (1-T)X_1\|^2]^{1 / 2} \\
&\leq \left(2(1-T^2)\mathbb{E} [\|X_0\|^2] + 2(1-T)^2 \mathbb{E}[\|X_1\|^2]\right)^{1 / 2}.
\end{align*}
Since we assume $\pi_1$ is supported on $[0,1]^d$ and $\mathbb{E}[\|X_0\|^2]=d$, we have $W_2\left(\pi_T, \pi_1\right) = \mathcal{O} (1-T)$.
\end{proof}

\begin{proof}[Proof of Theorem \ref{theorem: consistency in latent space}]
Combining Proposition \ref{pro: dis 1}, Lemma \ref{lemma: dis 2}, Lemma \ref{lemma: dis 3} and triangle inequality, we obtain
\begin{align*}
\mathbb{E}_{\mathcal{X}} [W_2(\widehat{\pi}_T, \pi_1)] = \widetilde{\mathcal{O}} \left((1-T) + e^{\gamma_{\boldsymbol{x}}} (\gamma_{\boldsymbol{x}} B + \gamma_t) \sqrt{\sum_{k=0}^{N-1}\left(t_{k+1}-t_k\right)^3} + 
e^{\gamma_{\boldsymbol{x}}} (1-T)^{-\frac{3d+5}{2}} n^{-\frac{1}{d+3}}\right).
\end{align*}
By the choice of neural network in Theorem \ref{theorem: gen 1}, we have $\gamma_{\boldsymbol{x}} \leq \frac{c_{14}}{(1-T)^3}, \gamma_{t} \leq \frac{c_{15} \sqrt{\log n}}{(1-T)^3}$. Letting $\max_{k=0,1 \ldots, N-1}\left|t_{k+1}-t_k\right|=\mathcal{O} (n^{-\frac{1}{d+3}}), T(n)=1-(\log n)^{-1/6}$ and omitting polynomials of logarithm, we obtain
\begin{align*}
\mathbb{E}_{\mathcal{X}} [W_2(\widehat{\pi}_T, \pi_1)] = \widetilde{\mathcal{O}}\left((\log n)^{-1/6} + e^{c_{14} \sqrt{\log n}} n^{-\frac{1}{d+3}}\right),
\end{align*}
which tends to $0$ as $n$ goes to infinity. 
\end{proof}

\subsection{Auxiliary lemma}
\begin{lemma}[Gr\"onwall's inequality]
\label{lemma: dis 1}
Given a function $f(t)$ defined on $[a, b] (a<b)$, satisfying $\frac{\mathrm{d} f(t)}{\mathrm{d} t} \leq$ $\alpha f(t)+g(t)$ on $[a, b]$ and $\alpha \geq 0$, we have
\begin{align*}
f(b) \leq e^{\alpha(b-a)} f(a)+\int_a^b e^{\alpha(b-t)} g(t) \mathrm{d} t.
\end{align*}
\end{lemma}

\begin{proof}
By multiplying $e^{-\alpha t}$ on both sides of $\frac{\mathrm{d} f(t)}{\mathrm{d} t} \leq \alpha f(t)+g(t)$ and some manipulation of algebra, we obtain
\begin{align*}
e^{-\alpha t} \frac{\mathrm{d} f(t)}{\mathrm{d} t}-\alpha e^{-\alpha t} f(t) \leq e^{-\alpha t} g(t).
\end{align*}
Integrating on interval $[a, b]$ on both sides, we get
\begin{align*}
e^{-\alpha b} f(b)-e^{-\alpha a} f(a) \leq \int_a^b e^{\alpha(b-t)} g(t) \mathrm{d} t.
\end{align*}
This concludes the proof.
\end{proof}

\section{Compressibility Analysis} \label{appendix: autoencoder}

\begin{definition}[Rademacher complexity]
The Rademacher complexity of a set $A \subseteq \mathbb{R}^N$ is defined as
\begin{align*}
\mathfrak{R}_N(A)=\mathbb{E}_{\left\{\sigma_i\right\}_{k=1}^N}\left[\sup_{a \in A} \frac{1}{N} \sum_{k=1}^N \sigma_k a_k\right],
\end{align*}
where $\{\sigma_k\}_{k=1}^N$ are $N$ i.i.d Rademacher variables with $\mathbb{P} (\sigma_k=1)=\mathbb{P} (\sigma_k=-1)=\frac{1}{2}$. The Rademacher complexity of function class $\mathcal{F}$ associate with random sample $\{X_k\}_{k=1}^N$ is defined as
\begin{align*}
\mathfrak{R}_N(\mathcal{F})=\mathbb{E}_{\{X_k, \sigma_k\}_{k=1}^N} \left[\sup_{u \in \mathcal{F}} \frac{1}{N} \sum_{k=1}^N \sigma_k u\left(X_k\right)\right].
\end{align*}
\end{definition}

\begin{proof}[Proof of Lemma \ref{lemma: ae rate}]\\
$\bullet$ \textbf{ Step 1: error decomposition.}\\
For each $\boldsymbol{D}_\theta \in \mathcal{D}$ and $\boldsymbol{E}_\theta \in \mathcal{E}$, we have
\begin{align*}
&\mathcal{R} (\widehat{\boldsymbol{D}}, \widehat{\boldsymbol{E}}) - \mathcal{R} (\boldsymbol{D}^*, \boldsymbol{E}^*) \\
&= \mathcal{R} (\widehat{\boldsymbol{D}}, \widehat{\boldsymbol{E}}) - \widehat{\mathcal{R}} (\widehat{\boldsymbol{D}}, \widehat{\boldsymbol{E}}) + \widehat{\mathcal{R}} (\widehat{\boldsymbol{D}}, \widehat{\boldsymbol{E}}) - \widehat{\mathcal{R}} (\boldsymbol{D}_\theta, \boldsymbol{E}_\theta) + \widehat{\mathcal{R}} (\boldsymbol{D}_\theta, \boldsymbol{E}_\theta) - \mathcal{R} (\boldsymbol{D}_\theta, \boldsymbol{E}_\theta) \\
&~~~ + \mathcal{R} (\boldsymbol{D}_\theta, \boldsymbol{E}_\theta) - \mathcal{R} (\boldsymbol{D}^*, \boldsymbol{E}^*) \\
&\leq \sup_{\boldsymbol{D} \in \mathcal{D}, \boldsymbol{E} \in \mathcal{E}} \mathcal{R} (\boldsymbol{D}, \boldsymbol{E}) - \widehat{\mathcal{R}} (\boldsymbol{D}, \boldsymbol{E}) + \sup_{\boldsymbol{D} \in \mathcal{D}, \boldsymbol{E} \in \mathcal{E}} \widehat{\mathcal{R}} (\boldsymbol{D}, \boldsymbol{E}) - \mathcal{R} (\boldsymbol{D}, \boldsymbol{E}) \\
&~~~ + \mathcal{R} (\boldsymbol{D}_\theta, \boldsymbol{E}_\theta) - \mathcal{R} (\boldsymbol{D}^*, \boldsymbol{E}^*),
\end{align*}
where the inequality is due to the fact that $\widehat{\mathcal{R}} (\widehat{\boldsymbol{D}}, \widehat{\boldsymbol{E}}) \leq \widehat{\mathcal{R}} (\boldsymbol{D}_\theta, \boldsymbol{E}_\theta)$. Then taking infimum over $\boldsymbol{D}_\theta \in \mathcal{D}$ and $\boldsymbol{E}_\theta \in \mathcal{E}$ yields
\begin{align}
\label{eq: ae 1}
\begin{aligned}
&\mathcal{R} (\widehat{\boldsymbol{D}}, \widehat{\boldsymbol{E}}) - \mathcal{R} (\boldsymbol{D}^*, \boldsymbol{E}^*) \\
&\leq \sup_{\boldsymbol{D} \in \mathcal{D}, \boldsymbol{E} \in \mathcal{E}} \mathcal{R} (\boldsymbol{D}, \boldsymbol{E}) - \widehat{\mathcal{R}} (\boldsymbol{D}, \boldsymbol{E}) + \sup_{\boldsymbol{D} \in \mathcal{D}, \boldsymbol{E} \in \mathcal{E}} \widehat{\mathcal{R}} (\boldsymbol{D}, \boldsymbol{E}) - \mathcal{R} (\boldsymbol{D}, \boldsymbol{E}) \\
&~~~ + \inf_{\boldsymbol{D} \in \mathcal{D}, \boldsymbol{E} \in \mathcal{E}} \mathcal{R} (\boldsymbol{D}, \boldsymbol{E}) - \mathcal{R} (\boldsymbol{D}^*, \boldsymbol{E}^*).
\end{aligned}
\end{align}

$\bullet$ \textbf{Step 2: approximation error.}\\
Suppose Assumption \ref{ass: bounded support gammahat} and \ref{ass: compressibility} hold. In (\ref{eq: def encoder network}) and (\ref{eq: def decoder network}), we specify the encoder network architecture as
\begin{align}
\mathcal{E} = \mathcal{T}_{D, d}(N_{\boldsymbol{E}}, h_{\boldsymbol{E}}, d_{\boldsymbol{E}, k}, d_{\boldsymbol{E}, v}, d_{\boldsymbol{E}, ff}, B_{\boldsymbol{E}}, J_{\boldsymbol{E}}, \gamma_{\boldsymbol{E}})
\end{align}
and the decoder network architecture as
\begin{align}
\mathcal{D} = \mathcal{T}_{d, D}\left(N_{\boldsymbol{D}}, h_{\boldsymbol{D}}, d_{\boldsymbol{D}, k}, d_{\boldsymbol{D}, v}, d_{\boldsymbol{D}, ff}, B_{\boldsymbol{D}}, J_{\boldsymbol{D}}, \gamma_{\boldsymbol{D}}\right).
\end{align}
By Theorem \ref{corollary: app 1}, given any $0<\varepsilon<1$, there exists an encoder network $\boldsymbol{E} \in \mathcal{E}$ with configuration
\begin{align*}
\begin{gathered}
N_{\boldsymbol{E}} = \mathcal{O} \left( \log \left(\frac{K_{\boldsymbol{E}}}{\varepsilon}\right)\right), \quad
h_{\boldsymbol{E}} = \mathcal{O} \left(  \left( \frac{K_{\boldsymbol{E}}}{\varepsilon} \right)^{D} \right), \quad
d_{\boldsymbol{E}, ff} = 8 h_{\boldsymbol{E}}, \quad 
d_{\boldsymbol{E}, k} = \mathcal{O} \left(K_{\boldsymbol{E}}\right), \\
d_{\boldsymbol{E}, v} = \mathcal{O} \left(K_{\boldsymbol{E}}\right), \quad
B_{\boldsymbol{E}} = \mathcal{O} \left(K_{\boldsymbol{E}}\right), \quad
J_{\boldsymbol{E}} = \mathcal{O} \left(\left( \frac{K_{\boldsymbol{E}}}{\varepsilon} \right)^{D} \log \left(\frac{K_{\boldsymbol{E}}}{\varepsilon}\right) \right), \quad
\gamma_{\boldsymbol{E}} = \mathcal{O} \left(K_{\boldsymbol{E}}\right),
\end{gathered}
\end{align*}
such that 
\begin{align*}
\|\boldsymbol{E}(\boldsymbol{y})-\boldsymbol{E}^*(\boldsymbol{y})\|_{L^\infty([0,1]^D)} \leq \varepsilon.
\end{align*}
Moreover, there exists an decoder network $\boldsymbol{D} \in \mathcal{D}$ with configuration
\begin{align*}
\begin{gathered}
N_{\boldsymbol{D}} = \mathcal{O} \left( \log \left(\frac{K_{\boldsymbol{D}}}{\varepsilon}\right)\right), \quad
h_{\boldsymbol{D}} = \mathcal{O} \left(  \left( \frac{K_{\boldsymbol{D}}}{\varepsilon} \right)^{d} \right), \quad
d_{\boldsymbol{D}, ff} = 8 h_{\boldsymbol{D}}, \quad 
d_{\boldsymbol{D}, k} = \mathcal{O} \left(K_{\boldsymbol{D}}\right), \\
d_{\boldsymbol{D}, v} = \mathcal{O} \left(K_{\boldsymbol{D}}\right), \quad
B_{\boldsymbol{D}} = \mathcal{O} \left(K_{\boldsymbol{D}}\right), \quad 
J_{\boldsymbol{D}} = \mathcal{O} \left(\left( \frac{K_{\boldsymbol{D}}}{\varepsilon} \right)^{d} \log \left(\frac{K_{\boldsymbol{D}}}{\varepsilon}\right) \right), \quad
\gamma_{\boldsymbol{D}} = \mathcal{O} \left(K_{\boldsymbol{D}}\right),
\end{gathered}
\end{align*}
such that 
\begin{align*}
\|\boldsymbol{D}(\boldsymbol{y})-\boldsymbol{D}^*(\boldsymbol{y})\|_{L^\infty([0,1]^d)} \leq \varepsilon.
\end{align*}
Then we have
\begin{align*}
&\|(\boldsymbol{D} \circ \boldsymbol{E})(\boldsymbol{y})-\boldsymbol{y}\|^2 - \|(\boldsymbol{D}^* \circ \boldsymbol{E}^*)(\boldsymbol{y})-\boldsymbol{y}\|^2 \\
&= \left\langle (\boldsymbol{D} \circ \boldsymbol{E})(\boldsymbol{y}) + (\boldsymbol{D}^* \circ \boldsymbol{E}^*)(\boldsymbol{y})- 2\boldsymbol{y}, (\boldsymbol{D} \circ \boldsymbol{E})(\boldsymbol{y}) - (\boldsymbol{D}^* \circ \boldsymbol{E}^*)(\boldsymbol{y}) \right\rangle \\
&\leq (B_{\boldsymbol{D}} + K_{\boldsymbol{D}} + 2\sqrt{D}) \left\| (\boldsymbol{D} \circ \boldsymbol{E})(\boldsymbol{y}) - (\boldsymbol{D}^* \circ \boldsymbol{E}^*)(\boldsymbol{y}) \right\| \\
&\leq (B_{\boldsymbol{D}} + K_{\boldsymbol{D}} + 2\sqrt{D}) \left( \left\| (\boldsymbol{D} \circ \boldsymbol{E})(\boldsymbol{y}) - (\boldsymbol{D}^* \circ \boldsymbol{E})(\boldsymbol{y}) \right\| + \left\| (\boldsymbol{D}^* \circ \boldsymbol{E})(\boldsymbol{y}) - (\boldsymbol{D}^* \circ \boldsymbol{E}^*)(\boldsymbol{y}) \right\|\right) \\
&\leq (B_{\boldsymbol{D}} + K_{\boldsymbol{D}} + 2\sqrt{D}) \left(\left\| \boldsymbol{D}(  \boldsymbol{E}(\boldsymbol{y})) - \boldsymbol{D}^* (\boldsymbol{E}(\boldsymbol{y})) \right\| + K_{\boldsymbol{D}} \left\| \boldsymbol{E}(\boldsymbol{y}) - \boldsymbol{E}^*(\boldsymbol{y}) \right\| \right) \\
&= \mathcal{O} (\varepsilon),
\end{align*}
which implies
\begin{align}
\label{eq: ae app error}
\begin{aligned}
&\inf_{\boldsymbol{D} \in \mathcal{D}, \boldsymbol{E} \in \mathcal{E}} \mathcal{R}(\boldsymbol{D}, \boldsymbol{E})-\mathcal{R}\left(\boldsymbol{D}^*, \boldsymbol{E}^*\right) \\
&= \inf_{\boldsymbol{D} \in \mathcal{D}, \boldsymbol{E} \in \mathcal{E}} \int_{\mathbb{R}^D} \|(\boldsymbol{D} \circ \boldsymbol{E})(\boldsymbol{y})-\boldsymbol{y}\|^2 - \|(\boldsymbol{D}^* \circ \boldsymbol{E}^*)(\boldsymbol{y})-\boldsymbol{y}\|^2 \mathrm{d} \widetilde{\gamma}_1 \\
&= \mathcal{O} (\varepsilon).
\end{aligned}
\end{align}

$\bullet$ \textbf{Step 3: generalization error.}\\
For simplicity, denote $\mathcal{G} = \{ l(\boldsymbol{y}) = \|(\boldsymbol{D} \circ \boldsymbol{E})(\boldsymbol{y})-\boldsymbol{y}\|^2: \boldsymbol{E} \in \mathcal{E}, \boldsymbol{D}\in\mathcal{D} \}$. We introduce a ghost dataset $\mathcal{Y}^{\prime}=\{\boldsymbol{y}_i^\prime\}_{i=1}^m$ drawn i.i.d. from $\widetilde{\gamma}_1$, and let $\sigma=\left\{\sigma_i\right\}_{i=1}^m$ be a sequence of i.i.d. Rademacher variables independent of both $\mathcal{Y}$ and $\mathcal{Y}^{\prime}$. Then we have
\begin{align}
\label{eq: ae gen error 1}
\begin{aligned}
& \mathbb{E}_{\mathcal{Y}} \left[ \sup_{\boldsymbol{D} \in \mathcal{D}, \boldsymbol{E} \in \mathcal{E}} \mathcal{R} (\boldsymbol{D}, \boldsymbol{E}) - \widehat{\mathcal{R}} (\boldsymbol{D}, \boldsymbol{E}) \right] \\
&= \mathbb{E}_{\mathcal{Y}} \left[ \sup_{l \in \mathcal{G}} \mathbb{E} [l(\boldsymbol{y})] - \frac{1}{m} \sum_{i=1}^m l(\boldsymbol{y}_i) \right] \\
&= \frac{1}{m} \mathbb{E}_{\mathcal{Y}} \left[ \sup_{l \in \mathcal{G}} \mathbb{E}_{\mathcal{Y}^\prime} \left[ \sum_{i=1}^m l(\boldsymbol{y}_i^\prime) \right] - \sum_{i=1}^m l(\boldsymbol{y}_i) \right] \\
&\leq \frac{1}{m} \mathbb{E}_{\mathcal{Y}, \mathcal{Y}^\prime} \left[ \sup_{l \in \mathcal{G}} \sum_{i=1}^m \left(l(\boldsymbol{y}_i^\prime) - l(\boldsymbol{y}_i)\right) \right] \\
&= \frac{1}{m} \mathbb{E}_{\mathcal{Y}, \mathcal{Y}^\prime, \sigma} \left[ \sup_{l \in \mathcal{G}} \sum_{i=1}^m \sigma_i \left(l(\boldsymbol{y}_i^\prime) - l(\boldsymbol{y}_i)\right)  \right] \\
&= \frac{1}{m} \mathbb{E}_{\mathcal{Y}^\prime, \sigma} \left[ \sup_{l \in \mathcal{G}} \sum_{i=1}^m \sigma_i l(\boldsymbol{y}_i^\prime) \right] + \frac{1}{m} \mathbb{E}_{\mathcal{Y}, \sigma} \left[ \sup_{l \in \mathcal{G}} \sum_{i=1}^m (-\sigma_i) l(\boldsymbol{y}_i)  \right] \\
&= 2\mathbb{E}_{\mathcal{Y}, \sigma} \left[ \sup_{l \in \mathcal{G}} \frac{1}{m} \sum_{i=1}^m \sigma_i l(\boldsymbol{y}_i)  \right] \\
&= 2 \mathfrak{R}_m(\mathcal{G}),
\end{aligned}
\end{align}
where we use the fact that randomly interchange of the corresponding components of $\mathcal{Y}$ and $\mathcal{Y}^{\prime}$ doesn't affect the joint distribution of $\mathcal{Y}$ and $\mathcal{Y}^{\prime}$, $\boldsymbol{y}_i$ and $\boldsymbol{y}_i^{\prime}$ have the same distribution, and $\sigma_i$ and $-\sigma_i$ have the same distribution. Similarly, we have 
\begin{align}
\label{eq: ae gen error 2}
\mathbb{E}_{\mathcal{Y}} \left[ \sup_{\boldsymbol{D} \in \mathcal{D}, \boldsymbol{E} \in \mathcal{E}} \widehat{\mathcal{R}} (\boldsymbol{D}, \boldsymbol{E}) - \mathcal{R} (\boldsymbol{D}, \boldsymbol{E}) \right] \leq 2 \mathfrak{R}_m(\mathcal{G}).
\end{align}

By employing the chaining technique, an upper bound on the Rademacher complexity of a function class can be established via its covering number. What remains is to determine a bound for the covering number. To establish a bound on $\mathcal{N}(\delta, \mathcal{G}, d_{\mathcal{Y}, \infty})$, we first define necessary subsets of the encoder and decoder function spaces. We use the superscript $i$ to denote the $i$-th component. Let $\mathcal{E}^{(i)}:=\{ E^{(i)}: \boldsymbol{E} \in \mathcal{E}\}$ for $i = 1, \ldots, d$, and similarly, let $\mathcal{D}^{(i)}:=\{ D^{(i)}: \boldsymbol{D} \in \mathcal{D}\}$ for $i = 1, \ldots, D$. We then construct an $\delta$-cover of $\mathcal{E}^{(i)}$, denoted as $\mathcal{E}^{(i)}_\delta$, where $|\mathcal{E}^{(i)}_\delta|=\mathcal{N}(\delta, \mathcal{E}^{(i)}, d_{\mathcal{Y}, \infty})$. This means for every $E^{(i)}\in\mathcal{E}^{(i)}$, there is an $E^{(i)}_\delta$ in $\mathcal{E}^{(i)}_\delta$ such that the distance $d_{\mathcal{Y}, \infty}(E^{(i)}, E^{(i)}_\delta) \leq \delta$. Using the triangle inequality, we can assert the existence of $\widetilde{E}^{(i)}_\delta \in \mathcal{E}^{(i)}$ for which $d_{\mathcal{Y}, \infty}(E^{(i)}, \widetilde{E}^{(i)}_\delta) \leq 2\delta$. In a similar manner, we apply the discretization to the decoder part. With a fixed encoder network $\boldsymbol{E}$, an $\delta$-cover for $\mathcal{D}^{(i)}$ is established, named $\mathcal{D}^{(i)}_\delta$, with $|\mathcal{D}^{(i)}_\delta|= \mathcal{N}(\delta, \mathcal{D}^{(i)}, d_{\boldsymbol{E}(\mathcal{Y}), \infty})$, where $\boldsymbol{E}(\mathcal{Y})$ denotes the set $\{\boldsymbol{E}(\boldsymbol{y}_i)\}_{i=1}^m$. This setup ensures that for any $D^{(i)} \in \mathcal{D}^{(i)}$, there exists a corresponding $D^{(i)}_\delta \in \mathcal{D}^{(i)}_\delta$ such that $d_{\boldsymbol{E}(\mathcal{Y}), \infty}(D^{(i)}, D^{(i)}_\delta) \leq \delta$. We can also find a $\widetilde{D}^{(i)}_\delta \in \mathcal{D}^{(i)}$ that satisfies $d_{\boldsymbol{E}(\mathcal{Y}), \infty}(D^{(i)}, \widetilde{D}^{(i)}_\delta) \leq 2\delta$.

For any $l\in \mathcal{G}$ with $l(\boldsymbol{y})=\|(\boldsymbol{D} \circ \boldsymbol{E})(\boldsymbol{y})-\boldsymbol{y}\|^2$, where $\boldsymbol{E}\in \mathcal{E}, \boldsymbol{D}\in \mathcal{D}$,
the discretization above determines a $\widetilde{l}$ with $\widetilde{l}(\boldsymbol{y})=\|(\widetilde{\boldsymbol{D}} \circ \widetilde{\boldsymbol{E}})(\boldsymbol{y})-\boldsymbol{y}\|^2$, where $\widetilde{\boldsymbol{E}} = (\widetilde{E}^{(1)}_\delta, \ldots, \widetilde{E}^{(d)}_\delta)^\top$ and $\widetilde{\boldsymbol{D}} = (\widetilde{D}^{(1)}_\delta, \ldots, \widetilde{D}^{(D)}_\delta)^\top$. Then
\begin{align*}
& d_{\mathcal{Y}, \infty} (l, \widetilde{l}) \\
&= \max_{k=1,\ldots,m} \left| l(\boldsymbol{y}_k) - \widetilde{l}(\boldsymbol{y}_k) \right| \\
&= \max_{k=1,\ldots,m} \left| \|(\boldsymbol{D} \circ \boldsymbol{E})(\boldsymbol{y}_k)-\boldsymbol{y}_k\|^2 - \|(\widetilde{\boldsymbol{D}} \circ \widetilde{\boldsymbol{E}})(\boldsymbol{y}_k)-\boldsymbol{y}_k\|^2 \right| \\
&= \max_{k=1,\ldots,m} \left| \left\langle (\boldsymbol{D} \circ \boldsymbol{E})(\boldsymbol{y}_k) + (\widetilde{\boldsymbol{D}} \circ \widetilde{\boldsymbol{E}})(\boldsymbol{y}_k) - 2\boldsymbol{y}_k, (\boldsymbol{D} \circ \boldsymbol{E})(\boldsymbol{y}_k) - (\widetilde{\boldsymbol{D}} \circ \widetilde{\boldsymbol{E}})(\boldsymbol{y}_k) \right\rangle \right| \\
&\leq \max_{k=1,\ldots,m} \left\|  (\boldsymbol{D} \circ \boldsymbol{E})(\boldsymbol{y}_k) + (\widetilde{\boldsymbol{D}} \circ \widetilde{\boldsymbol{E}})(\boldsymbol{y}_k) - 2\boldsymbol{y}_k \right\| \cdot \left\| (\boldsymbol{D} \circ \boldsymbol{E})(\boldsymbol{y}_k) - (\widetilde{\boldsymbol{D}} \circ \widetilde{\boldsymbol{E}})(\boldsymbol{y}_k)  \right\| \\
&\leq \left((D+1)B_{\boldsymbol{D}}+2\sqrt{D}\right) \max_{k=1,\ldots,m} \left\| (\boldsymbol{D} \circ \boldsymbol{E})(\boldsymbol{y}_k) - (\boldsymbol{D} \circ \widetilde{\boldsymbol{E}})(\boldsymbol{y}_k)  \right\| + \left\| (\boldsymbol{D} \circ \widetilde{\boldsymbol{E}})(\boldsymbol{y}_k) - (\widetilde{\boldsymbol{D}} \circ \widetilde{\boldsymbol{E}})(\boldsymbol{y}_k)  \right\| \\
&\leq \left((D+1)B_{\boldsymbol{D}}+2\sqrt{D}\right) \max_{k=1,\ldots,m} \gamma_{\boldsymbol{D}} \left\| \boldsymbol{E}(\boldsymbol{y}_k) - \widetilde{\boldsymbol{E}}(\boldsymbol{y}_k)  \right\| + \left\| \boldsymbol{D}(\widetilde{\boldsymbol{E}}(\boldsymbol{y}_k)) - \widetilde{\boldsymbol{D}}( \widetilde{\boldsymbol{E}}(\boldsymbol{y}_k))  \right\| \\
&\leq \left((D+1)B_{\boldsymbol{D}}+2\sqrt{D}\right) \max_{k=1,\ldots,m} \gamma_{\boldsymbol{D}} \left\| \boldsymbol{E}(\boldsymbol{y}_k) - \widetilde{\boldsymbol{E}}(\boldsymbol{y}_k)  \right\|_1 + \left\| \boldsymbol{D}(\widetilde{\boldsymbol{E}}(\boldsymbol{y}_k)) - \widetilde{\boldsymbol{D}}( \widetilde{\boldsymbol{E}}(\boldsymbol{y}_k))  \right\|_1 \\
&\leq \left((D+1)B_{\boldsymbol{D}}+2\sqrt{D}\right) \left( \gamma_{\boldsymbol{D}} \sum_{i=1}^d d_{\mathcal{Y}, \infty}\left(E^{(i)}, \widetilde{E}_{\delta}^{(i)}\right) + \sum_{i=1}^D d_{\widetilde{\boldsymbol{E}}(\mathcal{Y}), \infty}\left(D^{(i)}, \widetilde{D}_{\delta}^{(i)}\right)  \right) \\
&\leq \left((D+1)B_{\boldsymbol{D}}+2\sqrt{D}\right) \left( 2d\gamma_{\boldsymbol{D}} + 2D \right)  \delta,
\end{align*}
which implies
\begin{align*}
&\mathcal{N}\left(\left((D+1)B_{\boldsymbol{D}}+2\sqrt{D}\right) \left( 2d\gamma_{\boldsymbol{D}} + 2D \right)  \delta, \mathcal{G}, d_{\mathcal{Y}, \infty}\right) \\
&\leq \sum_{\widetilde{\boldsymbol{E}}} \prod_{i=1}^D \mathcal{N}\left(\delta, \mathcal{D}^{(i)}, d_{\widetilde{\boldsymbol{E}}(\mathcal{Y}),\infty}\right) \\
&\leq \sum_{\widetilde{\boldsymbol{E}}} \prod_{i=1}^D \max_{\widetilde{\boldsymbol{E}}} \mathcal{N}\left(\delta, \mathcal{D}^{(i)}, d_{\widetilde{\boldsymbol{E}}(\mathcal{Y}),\infty}\right) \\
&= \prod_{i=1}^d \mathcal{N}\left(\delta, \mathcal{E}^{(i)}, d_{\mathcal{Y},\infty}\right) \cdot \prod_{i=1}^D \max_{\widetilde{\boldsymbol{E}}} \mathcal{N}\left(\delta, \mathcal{D}^{(i)}, d_{\widetilde{\boldsymbol{E}}(\mathcal{Y}),\infty}\right).
\end{align*}
Using that 
\begin{align*}
\begin{aligned}
\mathcal{E}^{(i)} \subseteq \mathcal{T}_{D, 1}(N_{\boldsymbol{E}}, h_{\boldsymbol{E}}, d_{\boldsymbol{E}, k}, d_{\boldsymbol{E}, v}, d_{\boldsymbol{E}, ff}, B_{\boldsymbol{E}}, J_{\boldsymbol{E}}, \gamma_{\boldsymbol{E}}), \\
\mathcal{D}^{(i)} \subseteq \mathcal{T}_{d, 1}(N_{\boldsymbol{D}}, h_{\boldsymbol{D}}, d_{\boldsymbol{D}, k}, d_{\boldsymbol{D}, v}, d_{\boldsymbol{D}, ff}, B_{\boldsymbol{D}}, J_{\boldsymbol{D}}, \gamma_{\boldsymbol{D}}),
\end{aligned}
\end{align*}
we get 
\begin{align*}
& \log \mathcal{N}\left(\left((D+1)B_{\boldsymbol{D}}+2\sqrt{D}\right) \left( 2d\gamma_{\boldsymbol{D}} + 2D \right)  \delta, \mathcal{G}, d_{\mathcal{Y}, \infty}\right) \\
&\leq \sum_{i=1}^d \log \mathcal{N}\left(\delta, \mathcal{E}^{(i)}, d_{\mathcal{Y},\infty}\right) + \sum_{i=1}^D \max_{\widetilde{\boldsymbol{E}}} \log \mathcal{N}\left(\delta, \mathcal{D}^{(i)}, d_{\widetilde{\boldsymbol{E}}(\mathcal{Y}),\infty}\right) \\
&\leq \sum_{i=1}^d \log \mathcal{N}\left(\delta, \mathcal{T}_{D, 1}(N_{\boldsymbol{E}}, h_{\boldsymbol{E}}, d_{\boldsymbol{E}, k}, d_{\boldsymbol{E}, v}, d_{\boldsymbol{E}, ff}, B_{\boldsymbol{E}}, J_{\boldsymbol{E}}, \gamma_{\boldsymbol{E}}), d_{\mathcal{Y},\infty}\right) \\
&~~~ + \sum_{i=1}^D \max_{\widetilde{\boldsymbol{E}}} \log \mathcal{N}\left(\delta, \mathcal{T}_{d, 1}(N_{\boldsymbol{D}}, h_{\boldsymbol{D}}, d_{\boldsymbol{D}, k}, d_{\boldsymbol{D}, v}, d_{\boldsymbol{D}, ff}, B_{\boldsymbol{D}}, J_{\boldsymbol{D}}, \gamma_{\boldsymbol{D}}), d_{\widetilde{\boldsymbol{E}}(\mathcal{Y}),\infty}\right).
\end{align*}
Lemma \ref{lemma:gen 5} yields that 
\begin{align*}
& \log \mathcal{N}\left(\left((D+1)B_{\boldsymbol{D}}+2\sqrt{D}\right) \left( 2d\gamma_{\boldsymbol{D}} + 2D \right)  \delta, \mathcal{G}, d_{\mathcal{Y}, \infty}\right) \\
&\leq c_{16} \bigg( N_{\boldsymbol{E}}^2 J_{\boldsymbol{E}} \log \left( \max \left\{N_{\boldsymbol{E}}, h_{\boldsymbol{E}}, d_{\boldsymbol{E}, k}, d_{\boldsymbol{E}, v}, d_{\boldsymbol{E}, ff}\right\} \right) \log\frac{B_{\boldsymbol{E}} m}{\delta} \\
&\hspace{1.2cm} + N_{\boldsymbol{D}}^2 J_{\boldsymbol{D}} \log \left( \max \left\{N_{\boldsymbol{D}}, h_{\boldsymbol{D}}, d_{\boldsymbol{D}, k}, d_{\boldsymbol{D}, v}, d_{\boldsymbol{D}, ff}\right\} \right) \log\frac{B_{\boldsymbol{D}} m}{\delta} \bigg).
\end{align*}
Given the chosen parameters, we deduce that 
\begin{align*}
\log \mathcal{N}\left(\delta, \mathcal{G}, d_{\mathcal{Y}, \infty}\right) = \mathcal{O} \left( \frac{1}{\varepsilon^D} \left(\log \frac{1}{\varepsilon}\right)^4 \log\frac{m}{\delta} \right).
\end{align*}
Applying Lemma \ref{lemma: ae 1}, we obtain 
\begin{align}
\label{eq: ae gen error}
\begin{aligned}
\mathfrak{R}_m\left(\mathcal{G}\right) &\leq 4 \inf_{0<\delta< B_{\boldsymbol{D}}^2+D}\left(\delta+\frac{3}{\sqrt{m}} \int_\delta^{B_{\boldsymbol{D}}^2+D} \sqrt{\log \mathcal{N} \left(\gamma, \mathcal{G}, d_{\mathcal{Y}, \infty}\right)} \mathrm{d} \gamma\right) \\
&\leq 4 \inf_{0<\delta< B_{\boldsymbol{D}}^2+D}\left(\delta+\frac{3}{\sqrt{m}} \left(B_{\boldsymbol{D}}^2+D\right) \sqrt{\log \mathcal{N} \left(\delta, \mathcal{G}, d_{\mathcal{Y}, \infty}\right)} \right) \\
&= \mathcal{O} \left( m^{-1/2} \varepsilon^{-D/2} \left(\log\frac{1}{\varepsilon}\right)^2 (\log m)^{1/2}  \right).
\end{aligned}
\end{align}

$\bullet$ \textbf{Step 4: balancing error terms.}\\
Combining (\ref{eq: ae app error}), (\ref{eq: ae gen error 1}), (\ref{eq: ae gen error 2}) and (\ref{eq: ae gen error}), we have
\begin{align*}
\mathbb{E}_{\mathcal{Y}}\left[ \mathcal{R}(\widehat{\boldsymbol{D}}, \widehat{\boldsymbol{E}})-\mathcal{R}\left(\boldsymbol{D}^*, \boldsymbol{E}^*\right) \right] &= \mathcal{O} \left( \varepsilon + m^{-1/2} \varepsilon^{-D/2} \left(\log\frac{1}{\varepsilon}\right)^2 (\log m)^{1/2} \right).
\end{align*}
Setting $\varepsilon = m^{-\frac{1}{D + 2}}$ gives rise to 
\begin{align*}
\mathbb{E}_{\mathcal{Y}}\left[ \mathcal{R}(\widehat{\boldsymbol{D}}, \widehat{\boldsymbol{E}})-\mathcal{R}\left(\boldsymbol{D}^*, \boldsymbol{E}^*\right) \right] = \mathcal{O} \left( m^{-\frac{1}{D+2}} \left(\log m\right)^{5/2} \right).
\end{align*}
$\mathcal{R} (\boldsymbol{D}^*, \boldsymbol{E}^*) = \varepsilon_{\widetilde{\gamma}_1}$ concludes the proof.
\end{proof}

\begin{proof}[Proof of Theorem \ref{theorem: main result}]
Given $m$ samples drawn from the pre-training data distribution $\widetilde{\gamma}_1$, we determine $\widehat{\boldsymbol{D}}, \widehat{\boldsymbol{E}}$ through empirical risk minimization as in (\ref{eq: autoencoder erm}). Given $n$ samples drawn from the target distribution $\gamma_1$, the encoder $\widehat{\boldsymbol{E}}$ maps these samples to a low-dimensional latent space, where flow matching and sampling are completed. The decoder $\widehat{\boldsymbol{E}}$ remaps the sampled data to the high-dimensional space, conforming to distribution $\widehat{\gamma}_T$. In our framework, we have
\begin{align*}
& \mathbb{E}_{\mathcal{X}, \mathcal{Y}} [W_2(\widehat{\gamma}_T, \gamma_1)] \\
&= \mathbb{E}_{\mathcal{Y}} [\mathbb{E}_{\mathcal{X}} [W_2(\widehat{\boldsymbol{D}}_{\#}\widehat{\pi}_T, \gamma_1)]] \\
&\leq \mathbb{E}_{\mathcal{Y}} [\mathbb{E}_{\mathcal{X}} [W_2(\widehat{\boldsymbol{D}}_\#\widehat{\pi}_T, (\widehat{\boldsymbol{D}} \circ \widehat{\boldsymbol{E}})_\#\gamma_1)]] + \mathbb{E}_{\mathcal{Y}} [W_2((\widehat{\boldsymbol{D}} \circ \widehat{\boldsymbol{E}})_\#\gamma_1, (\widehat{\boldsymbol{D}} \circ \widehat{\boldsymbol{E}})_\#\widetilde{\gamma}_1)] \\
&~~~~ + \mathbb{E}_{\mathcal{Y}} [W_2((\widehat{\boldsymbol{D}} \circ \widehat{\boldsymbol{E}})_\#\widetilde{\gamma}_1, \widetilde{\gamma}_1)] + W_2(\widetilde{\gamma}_1, \gamma_1) \\
&\leq \mathbb{E}_{\mathcal{Y}} [\mathrm{Lip}(\widehat{\boldsymbol{D}})\, \mathbb{E}_{\mathcal{X}} [W_2(\widehat{\pi}_T, \pi_1)]] + \mathbb{E}_{\mathcal{Y}} [\mathrm{Lip}(\widehat{\boldsymbol{D}})\mathrm{Lip}(\widehat{\boldsymbol{E}})\, W_2(\gamma_1, \widetilde{\gamma}_1)] + \mathbb{E}_{\mathcal{Y}} [\mathcal{R}(\widehat{\boldsymbol{\boldsymbol{D}}}, \widehat{\boldsymbol{E}})]^{1/2} + W_2(\widetilde{\gamma}_1, \gamma_1) \\
&\leq \gamma_{\boldsymbol{D}} \mathbb{E}_{\mathcal{Y}} [\mathbb{E}_{\mathcal{X}} [W_2(\widehat{\pi}_T, \pi_1)]] + \mathbb{E}_{\mathcal{Y}} [\mathcal{R}(\widehat{\boldsymbol{D}}, \widehat{\boldsymbol{E}})]^{1/2} + (\gamma_{\boldsymbol{E}} \gamma_{\boldsymbol{D}} + 1) \varepsilon_{\widetilde{\gamma}_1, \gamma_1} \\
&= \mathcal{O} (\sqrt{\varepsilon_{\widetilde{\gamma}_1}} + \varepsilon_{\widetilde{\gamma}_1, \gamma_1}),
\end{align*}
where the second inequality follows from $\mathbb{E}_{\mathcal{Y}} [W_2((\widehat{\boldsymbol{D}} \circ \widehat{\boldsymbol{E}})_\#\widetilde{\gamma}_1, \widetilde{\gamma}_1)] \leq \mathbb{E}_{\mathcal{Y}} [\mathcal{R}(\widehat{\boldsymbol{D}}, \widehat{\boldsymbol{E}})^{1/2}] \leq \mathbb{E}_{\mathcal{Y}} [\mathcal{R}(\widehat{\boldsymbol{D}}, \widehat{\boldsymbol{E}})]^{1/2}$ and Lemma \ref{lemma: lip contract}, the third inequality follows from $\widehat{\boldsymbol{D}} \in \mathcal{D}, \widehat{\boldsymbol{E}} \in \mathcal{E}$, and the last equality follows from Lemma \ref{lemma: ae rate} and Theorem \ref{theorem: consistency in latent space}.
\end{proof}

\subsection{Auxiliary lemma}
\begin{lemma}
\label{lemma: lip contract}
Let $\mu, \nu$ be distributions on $\mathbb{R}^d$ and let $\boldsymbol{f}: \mathbb{R}^d \rightarrow \mathbb{R}^{d^\prime}$ be a Lipschitz continuous mapping with Lipschitz constant $\operatorname{Lip}(\boldsymbol{f})<\infty$. Then,
\begin{align*}
W_2(\boldsymbol{f}_\# \mu, \boldsymbol{f}_\# \nu) \leq \operatorname{Lip}(\boldsymbol{f}) W_2(\mu, \nu).
\end{align*}
\end{lemma}

\begin{proof}
Similar proof can be found in \cite[Lemma 9]{perekrestenko2021high}.
Let $\pi$ be a coupling between $\mu$ and $\nu$ and let $\boldsymbol{g}: \mathbb{R}^{d} \times \mathbb{R}^{d} \rightarrow \mathbb{R}^{d^\prime} \times \mathbb{R}^{d^\prime}, \left(\boldsymbol{y}_1, \boldsymbol{y}_2\right) \mapsto (\boldsymbol{f}(\boldsymbol{y}_1), \boldsymbol{f}(\boldsymbol{y}_2))$. Then $\boldsymbol{g}_\# \pi$ is a coupling between $\boldsymbol{f}_\# \mu$ and $\boldsymbol{f}_\# \nu$ and
\begin{align*}
\begin{aligned}
W_2(\boldsymbol{f}_\# \mu, \boldsymbol{f}_\# \nu) & \leq \left(\int_{\mathbb{R}^d \times \mathbb{R}^d} \left\|\boldsymbol{y}_1-\boldsymbol{y}_2\right\|^2 \mathrm{d}(\boldsymbol{g}_\# \pi)\left(\boldsymbol{y}_1, \boldsymbol{y}_2\right) \right)^{1/2} \\
& =\left(\int_{\mathbb{R}^d \times \mathbb{R}^d} \left\|\boldsymbol{f}\left(\boldsymbol{y}_1\right)-\boldsymbol{f}\left(\boldsymbol{y}_2\right)\right\|^2 \mathrm{d} \pi\left(\boldsymbol{y}_1, \boldsymbol{y}_2\right) \right)^{1/2} \\
& \leq \operatorname{Lip}(\boldsymbol{f}) \left(\int_{\mathbb{R}^d \times \mathbb{R}^d} \left\|\boldsymbol{y}_1-\boldsymbol{y}_2\right\|^2 \mathrm{d} \pi\left(\boldsymbol{y}_1, \boldsymbol{y}_2\right) \right)^{1/2}.
\end{aligned}
\end{align*}
Taking the infimum over all $\pi \in \Pi(\mu, \nu)$, we get 
\begin{align*}
W_2(\boldsymbol{f}_\# \mu, \boldsymbol{f}_\# \nu) \leq \operatorname{Lip}(\boldsymbol{f}) W_2(\mu, \nu).
\end{align*}
\end{proof}

\begin{lemma}
\label{lemma: ae 1}
Let $\mathcal{F}$ be a function class defined on $\Omega$ and $\mathcal{D} = \{\boldsymbol{x}_1, \ldots, \boldsymbol{x}_n \}$. If $\sup _{f \in \mathcal{F}}\|f\|_{L^{\infty}(\Omega)} \leq B$, then
\begin{align*}
\mathfrak{R}_n\left(\mathcal{F}\right) \leq 4 \inf _{0<\delta<B / 2}\left(\delta+\frac{3}{\sqrt{n}} \int_\delta^{B / 2} \sqrt{\log \mathcal{N} \left(\varepsilon, \mathcal{F}, d_{\mathcal{D}, \infty}\right)} \mathrm{d} \varepsilon\right).
\end{align*}
\end{lemma}

\begin{proof}
See the proof of \citet[Lemma 27.4]{shalev2014understanding} and \citet[Lemma 3.11]{duan2022convergence}.
\end{proof}

\section{Properties of true velocity field}

\subsection{Computation of true velocity field}
\begin{lemma}
\label{lemma: true 2}
The true velocity field $\boldsymbol{v}^*$ can be written as:
\begin{align}
\label{true: eq 1}
\boldsymbol{v}^*(\boldsymbol{x}, t)=\frac{1}{t} \nabla_{\boldsymbol{x}} \log \pi_t(\boldsymbol{x})+\frac{1}{t} \boldsymbol{x},
\end{align}
where $\pi_t$ is the density of $X_t$, and $X_t=\sqrt{1-t^2} X_0 + t X_1$.
\end{lemma}

\begin{proof}
By some manipulation of algebra, (\ref{eq:gen 13}) implies
\begin{align*}
\begin{aligned}
\boldsymbol{v}^*(\boldsymbol{x}, t) & =\mathbb{E} \left[X_1-\frac{t}{\sqrt{1-t^2}}X_0 \bigg| X_t=\boldsymbol{x}\right] \\
& =\mathbb{E}\left[X_1-\frac{t}{1-t^2}\left(\sqrt{1-t^2} X_0+t X_1-t X_1\right) \bigg| X_t=\boldsymbol{x}\right] \\
& =\frac{1}{1-t^2} \mathbb{E}\left[X_1 | X_t=\boldsymbol{x}\right]-\frac{t}{1-t^2} \boldsymbol{x} \\
& =\frac{1}{1-t^2} \int \frac{\boldsymbol{x}_1 \pi_{t | 1}\left(\boldsymbol{x} | \boldsymbol{x}_1\right) \pi_1\left(\boldsymbol{x}_1\right)}{\pi_t(\boldsymbol{x})} \mathrm{d} \boldsymbol{x}_1-\frac{t}{1-t^2} \boldsymbol{x} \\
& =\frac{1}{1-t^2} \int \frac{1}{\sqrt{(2 \pi)^d (1-t^2)^d}} \frac{\boldsymbol{x}_1 \exp \left(-\frac{\left\|\boldsymbol{x}-t \boldsymbol{x}_1\right\|^2}{2(1-t^2)}\right) \pi_1\left(\boldsymbol{x}_1\right)}{\pi_t(\boldsymbol{x})} \mathrm{d} \boldsymbol{x}_1-\frac{t}{1-t^2} \boldsymbol{x} \\
& =\frac{1}{t} \int \frac{1}{\sqrt{(2 \pi)^d (1-t^2)^d}} \frac{\left(\frac{t \boldsymbol{x}_1-\boldsymbol{x}}{1-t^2}+\frac{\boldsymbol{x}}{1-t^2}\right) \exp \left(-\frac{\left\|\boldsymbol{x}-t \boldsymbol{x}_1\right\|^2}{2(1-t^2)}\right) \pi_1\left(\boldsymbol{x}_1\right)}{\pi_t(\boldsymbol{x})} \mathrm{d} \boldsymbol{x}_1-\frac{t}{1-t^2} \boldsymbol{x} \\
& =\frac{1}{t} \int \frac{1}{\sqrt{(2 \pi)^d(1-t^2)^d}} \frac{\nabla_{\boldsymbol{x}} \exp \left(-\frac{\left\|\boldsymbol{x}-t \boldsymbol{x}_1\right\|^2}{2(1-t^2)}\right) \pi_1\left(\boldsymbol{x}_1\right)}{\pi_t(\boldsymbol{x})} \mathrm{d} \boldsymbol{x}_1 + \left(\frac{1}{t (1-t^2)}-\frac{t}{1-t^2}\right) \boldsymbol{x} \\
& =\frac{1}{t} \nabla_{\boldsymbol{x}} \log \pi_t(\boldsymbol{x})+\frac{1}{t} \boldsymbol{x},
\end{aligned}
\end{align*}
where $\pi_{t | 1}$ is the density of $X_t$ conditioned on $X_1$. It concludes the proof.
\end{proof}

\subsection{Computation of partial derivative regarding \texorpdfstring{$t$}{}}

\begin{lemma}
\label{lemma: true 1}
$\partial_t \boldsymbol{v}^*(\boldsymbol{x}, t)= -\frac{1+t^2}{(1-t^2)^2} \boldsymbol{x}
+\frac{2t}{(1-t^2)^2} \mathbb{E}[X_1 | X_t=\boldsymbol{x}]
+\frac{1+t^2}{(1-t^2)^3} \operatorname{Cov}[X_1 | X_t=\boldsymbol{x}] \boldsymbol{x}
-\frac{t}{(1-t^2)^3}(\mathbb{E}[X_1\|X_1\|^2 | X_t=\boldsymbol{x}]-\mathbb{E}[X_1 | X_t=\boldsymbol{x}] \mathbb{E}[\|X_1\|^2 | X_t=\boldsymbol{x}])$, where $\operatorname{Cov}[X_1 | X_t=\boldsymbol{x}]$ is the covariance matrix of $X_1$ conditioned on $X_t=\boldsymbol{x}$.
\end{lemma}

\begin{proof}
To ease notation, we define $\phi_t(\boldsymbol{x}):=\int \exp \left(-\frac{\left\|\boldsymbol{x}-t \boldsymbol{x}_1\right\|^2}{2(1-t^2)}\right) \pi_1\left(\mathrm{d} \boldsymbol{x}_1\right)$, which is the unnormalized version of $\pi_t(\boldsymbol{x})$. Note that $\nabla \log \phi_t(\boldsymbol{x})=\nabla \log \pi_t(\boldsymbol{x})$, using the product rule of the derivatives, (\ref{true: eq 1}) implies:
\begin{align}
\label{true: eq 2}
\begin{aligned}
\partial_t \boldsymbol{v}^*(\boldsymbol{x}, t) & =-\frac{1}{t^2} \nabla_{\boldsymbol{x}} \log \pi_t(\boldsymbol{x})+\frac{1}{t} \partial_t \nabla_{\boldsymbol{x}} \log \pi_t(\boldsymbol{x})-\frac{1}{t^2} \boldsymbol{x} \\
& =-\frac{1}{t(1-t^2)} \mathbb{E}[X_1 | X_t=\boldsymbol{x}]+\frac{1}{t^2(1-t^2)} \boldsymbol{x}+\frac{1}{t} \partial_t\left(\frac{\nabla \phi_t(\boldsymbol{x})}{\phi_t(\boldsymbol{x})}\right)-\frac{1}{t^2} \boldsymbol{x} \\
& =\frac{1}{1-t^2} \boldsymbol{x}-\frac{1}{t(1-t^2)} \mathbb{E}[X_1 | X_t=\boldsymbol{x}]+\frac{1}{t}\left(\frac{\partial_t \nabla \phi_t(\boldsymbol{x})}{\phi_t(\boldsymbol{x})}-\frac{\partial_t \phi_t(\boldsymbol{x}) \nabla \phi_t(\boldsymbol{x})}{\left(\phi_t(\boldsymbol{x})\right)^2}\right)
\end{aligned}
\end{align}

Then we focus on the computation of the last term above. We first compute $\frac{\partial_t \nabla \phi_t(\boldsymbol{x})}{\phi_t(\boldsymbol{x})}$ as follows:
\begin{align}
\label{true: eq 3}
\begin{aligned}
& \frac{\partial_t \nabla \phi_t(\boldsymbol{x})}{\phi_t(\boldsymbol{x})} \\
& =\frac{1}{\phi_t(\boldsymbol{x})} \partial_t \int \frac{t \boldsymbol{x}_1-\boldsymbol{x}}{1-t^2} \exp \left(-\frac{\left\|\boldsymbol{x}-t \boldsymbol{x}_1\right\|^2}{2(1-t^2)}\right) \pi_1\left(\mathrm{d} \boldsymbol{x}_1\right) \\
& =\frac{1}{\phi_t(\boldsymbol{x})} \int\left(
\frac{(1-t^2) \boldsymbol{x}_1 +2t\left(t \boldsymbol{x}_1-\boldsymbol{x}\right)}{(1-t^2)^2} \exp \left(-\frac{\left\|\boldsymbol{x}-t \boldsymbol{x}_1\right\|^2}{2(1-t^2)}\right)\right. \\
&~~~ \left. -\frac{t \boldsymbol{x}_1-\boldsymbol{x}}{1-t^2} \exp \left(-\frac{\left\|\boldsymbol{x}-t \boldsymbol{x}_1\right\|^2}{2(1-t^2)}\right) \frac{(t\left\|\boldsymbol{x}_1\right\|^2-\boldsymbol{x}_1^\top \boldsymbol{x})(1-t^2)+t\left\|\boldsymbol{x}-t \boldsymbol{x}_1\right\|^2}{(1-t^2)^2}\right) \pi_1\left(\mathrm{d} \boldsymbol{x}_1\right) \\
& =\frac{1+t^2}{(1-t^2)^2} \mathbb{E}[X_1 | X_t=\boldsymbol{x}]-\frac{2t}{(1-t^2)^2} \boldsymbol{x}-\frac{t^2}{(1-t^2)^3} \mathbb{E}[X_1 \|X_1\|^2 | X_t=\boldsymbol{x}] \\
&~~~ +\frac{t(1+t^2)}{(1-t^2)^3} \mathbb{E}[X_1 X_1^\top | X_t=\boldsymbol{x}] \boldsymbol{x}-\frac{t^2}{(1-t^2)^3} \mathbb{E}[X_1 | X_t=\boldsymbol{x}]\|\boldsymbol{x}\|^2 \\
&~~~ +\frac{t}{(1-t^2)^3} \mathbb{E}[\|X_1\|^2 | X_t=\boldsymbol{x}] \boldsymbol{x}-\frac{1+t^2}{(1-t^2)^3} \mathbb{E}[X_1^\top \boldsymbol{x} | X_t=\boldsymbol{x}] \boldsymbol{x}+\frac{t}{(1-t^2)^3}\|\boldsymbol{x}\|^2 \boldsymbol{x}
\end{aligned}
\end{align}
By some calculations, we obtain
\begin{align}
\label{true: eq 4}
\begin{aligned}
\frac{\partial_t \phi_t(\boldsymbol{x})}{\phi_t(\boldsymbol{x})} &=\frac{1}{\phi_t(\boldsymbol{x})} \int \left( -\frac{t}{(1-t^2)^2}\left\|\boldsymbol{x}\right\|^2-\frac{t}{(1-t^2)^2}\left\|\boldsymbol{x}_1\right\|^2+\frac{1+t^2}{(1-t^2)^2}\boldsymbol{x}_1^\top \boldsymbol{x} \right) \\
&\hspace{7cm} \cdot \exp \left(-\frac{\left\|\boldsymbol{x}-t \boldsymbol{x}_1\right\|^2}{2(1-t^2)}\right) \pi_1\left(\mathrm{d} \boldsymbol{x}_1\right) \\
& =-\frac{t}{(1-t^2)^2} \mathbb{E}[\|X_1\|^2 | X_t=\boldsymbol{x}]+\frac{1+t^2}{(1-t^2)^2} \mathbb{E}[X_1^\top \boldsymbol{x} | X_t=\boldsymbol{x}]-\frac{t}{(1-t^2)^2}\|\boldsymbol{x}\|^2 
\end{aligned}
\end{align}
and
\begin{align}
\label{true: eq 5}
\begin{aligned}
\frac{\nabla \phi_t(\boldsymbol{x})}{\phi_t(\boldsymbol{x})} & =\frac{1}{\phi_t(\boldsymbol{x})} \int \frac{t \boldsymbol{x}_1-\boldsymbol{x}}{1-t^2} \exp \left(-\frac{\left\|\boldsymbol{x}-t \boldsymbol{x}_1\right\|^2}{2(1-t^2)}\right) \pi_1\left(\mathrm{d} \boldsymbol{x}_1\right) \\
& =-\frac{\boldsymbol{x}}{1-t^2}+\frac{t}{1-t^2} \mathbb{E}[X_1 | X_t=\boldsymbol{x}].
\end{aligned}
\end{align}

Combining (\ref{true: eq 2}), (\ref{true: eq 3}), (\ref{true: eq 4}) and (\ref{true: eq 5}), we obtain
\begin{align*}
\partial_t \boldsymbol{v}^*(\boldsymbol{x}, t) &= -\frac{1+t^2}{(1-t^2)^2} \boldsymbol{x}+\frac{2t}{(1-t^2)^2} \mathbb{E}[X_1 | X_t=\boldsymbol{x}]+\frac{1+t^2}{(1-t^2)^3} \operatorname{Cov}[X_1 | X_t=\boldsymbol{x}] \boldsymbol{x} \\
&~~~ -\frac{t}{(1-t^2)^3}\left(\mathbb{E}[X_1 \|X_1\|^2 | X_t=\boldsymbol{x}]-\mathbb{E}[X_1 | X_t=\boldsymbol{x}] \mathbb{E}[\|X_1\|^2 | X_t=\boldsymbol{x}]\right).
\end{align*}
It concludes the proof.
\end{proof}

\subsection{An upper bound for velocity field}
\begin{lemma}
\label{lemma: true 3}
Suppose Assumption \ref{ass: bounded support gamma} holds. Then $\sup_{t \in[0, T]} \sup_{\boldsymbol{x} \in [-R, R]^d} \left|v_i^*(\boldsymbol{x}, t)\right| \leq \frac{1+R}{1-T^2}$.
\end{lemma}

\begin{proof}
For the $i$-coordinate, we have $v_i^*=\frac{1}{1-t^2} \mathbb{E}[X_1^{(i)} | X_t=\boldsymbol{x}]-\frac{t}{1-t^2} x_i$, where $X_1^{(i)}$ denotes the $i$-coordinate of $X_1$. Note that $\pi_1$ is supported on $[0,1]^d$ as stated in Assumption \ref{ass: bounded support gamma}, then
\begin{align*}
\sup_{t \in[0, T]} \sup_{\boldsymbol{x} \in [-R, R]^d}\left|v_i^*(\boldsymbol{x}, t)\right| \leq \frac{1+R}{1-T^2}.
\end{align*}
\end{proof}

\subsection{An upper bound of partial derivative regarding \texorpdfstring{$t$}{}}
\begin{lemma}
\label{lemma: true 4}
Suppose Assumption \ref{ass: bounded support gamma} holds. Then $\sup_{t \in [0, T]} \sup_{\boldsymbol{x} \in [-R, R]^d} \left\|\partial_t \boldsymbol{v}^*(\boldsymbol{x}, t)\right\|=\mathcal{O} \left(\frac{R}{(1-T)^3}\right)$.
\end{lemma}

\begin{proof}
From Lemma \ref{lemma: true 1}, we have
\begin{align*}
\left\|\partial_t \boldsymbol{v}^*(\boldsymbol{x}, t)\right\| &\leq \frac{1+t^2}{(1-t^2)^2} \|\boldsymbol{x}\|+\frac{2t}{(1-t^2)^2} \|\mathbb{E}[X_1 | X_t=\boldsymbol{x}]\|+\frac{1+t^2}{(1-t^2)^3} \|\operatorname{Cov}[X_1 | X_t=\boldsymbol{x}]\|_{\mathrm{op}} \|\boldsymbol{x}\| \\
&~~~ +\frac{t}{(1-t^2)^3}\left(\|\mathbb{E}[X_1\|X_1\|^2 | X_t=\boldsymbol{x}]\| + \|\mathbb{E}[X_1 | X_t=\boldsymbol{x}]\| \|\mathbb{E}[\|X_1\|^2 | X_t=\boldsymbol{x}]\|\right).
\end{align*}

Note that $\pi_1$ is assumed to be supported on $[0,1]^d$, we have $\left\|\mathbb{E}[X_1 | X_t=\boldsymbol{x}]\right\| \leq \mathbb{E}[\left\|X_1\right\|^2 | X_t=\boldsymbol{x}]^{1 / 2} \leq d^{1 / 2}$ and $\|\mathbb{E}[X_1\|X_1\|^2 | X_t=\boldsymbol{x}]\| \leq \mathbb{E}[\|X_1\|^6 | X_t=\boldsymbol{x}]^{1 / 2} \leq d^{3 / 2}$. To bound $\left\|\operatorname{Cov}[X_1 | X_t=\boldsymbol{x}\right]\|_{\mathrm{op}}$, we have the following inequality for any $\boldsymbol{u} \in \mathbb{R}^d$,
\begin{align*}
\begin{aligned}
\boldsymbol{u}^\top \operatorname{Cov}[X_1 | X_t=\boldsymbol{x}] \boldsymbol{u} & =\mathbb{E}[\boldsymbol{u}^\top X_1 X_1^\top \boldsymbol{u} | X_t=\boldsymbol{x}]-\mathbb{E}[\boldsymbol{u}^\top X_1 | X_t=\boldsymbol{x}] \mathbb{E}[X_1^\top \boldsymbol{u} | X_t=\boldsymbol{x}] \\
& =\mathbb{E}[(\boldsymbol{u}^\top X_1)^2 | X_t=\boldsymbol{x}]-\mathbb{E}[\boldsymbol{u}^\top X_1 | X_t=\boldsymbol{x}]^2 \\
& \leq 2 d\|\boldsymbol{u}\|^2
\end{aligned}
\end{align*}
Hence we have $\left\|\operatorname{Cov}[X_1 | X_t=x]\right\|_{\mathrm{op}} \leq 2 d$. Using these above inequalities, we have
\begin{align*}
\sup_{t \in [0, T]} \sup_{\boldsymbol{x} \in [-R, R]^d}\left\|\partial_t \boldsymbol{v}^*(\boldsymbol{x}, t)\right\| \leq \frac{\sqrt{d}(1+T^2)R}{(1-T^2)^2} +\frac{2\sqrt{d}T}{(1-T^2)^2} +\frac{2d^{3/2}(1+T^2)R}{(1-T^2)^3}  +\frac{2d^{3/2}T}{(1-T^2)^3}.
\end{align*}
Since $0<T<1$, the above inequality implies $\sup_{t \in [0, T]} \sup_{\boldsymbol{x} \in [-R, R]^d}\|\partial_t \boldsymbol{v}^*(\boldsymbol{x}, t)\|=\mathcal{O} \left(\frac{R}{(1-T)^3}\right)$.
\end{proof}

\subsection{Computation of partial derivative regarding spatial variable \texorpdfstring{$\boldsymbol{x}$}{}}
\begin{lemma}
We have the following identity:
\begin{align*}
\nabla \boldsymbol{v}^*(\boldsymbol{x}, t)=\frac{t}{(1-t^2)^2} \operatorname{Cov}[X_1 | X_t=\boldsymbol{x}]-\frac{t}{1-t^2} I_d.
\end{align*}
\end{lemma}

\begin{proof}
By Lemma \ref{lemma: true 2}, we have
\begin{align*}
\nabla \boldsymbol{v}^*(\boldsymbol{x}, t)=\frac{1}{t} \nabla^2 \log \pi_t(\boldsymbol{x})+\frac{1}{t} I_d.
\end{align*}

Further, the Hessian $\nabla^2 \log \pi_t(\boldsymbol{x})$ can be computed as
\begin{align*}
\begin{aligned}
\nabla^2 \log \pi_t(\boldsymbol{x})= & \nabla\left(\frac{\int_{\mathbb{R}^d} \frac{t \boldsymbol{x}_1-\boldsymbol{x}}{1-t^2} \exp \left(-\frac{\left\|\boldsymbol{x}-t \boldsymbol{x}_1\right\|^2}{2(1-t^2)}\right) \pi_1\left(\mathrm{d} \boldsymbol{x}_1\right)}{\int_{\mathbb{R}^d} \exp \left(-\frac{\left\|\boldsymbol{x}-t \boldsymbol{x}_1\right\|^2}{2(1-t^2)}\right) \pi_1\left(\mathrm{d} \boldsymbol{x}_1\right)}\right) \\
= & -\frac{1}{1-t^2} I_d+\frac{\int_{\mathbb{R}^d}\left(\frac{t \boldsymbol{x}_1-\boldsymbol{x}}{1-t^2}\right)^{\otimes 2} \exp \left(-\frac{\left\|\boldsymbol{x}-t \boldsymbol{x}_1\right\|^2}{2(1-t^2)}\right) \pi_1\left(\mathrm{d} \boldsymbol{x}_1\right)}{\int_{\mathbb{R}^d} \exp \left(-\frac{\left\|\boldsymbol{x}-t \boldsymbol{x}_1\right\|^2}{2(1-t^2)}\right) \pi_1\left(\mathrm{d} \boldsymbol{x}_1\right)} \\
& -\left(\frac{\int_{\mathbb{R}^d} \frac{t \boldsymbol{x}_1-\boldsymbol{x}}{1-t^2} \exp \left(-\frac{\left\|\boldsymbol{x}-t \boldsymbol{x}_1\right\|^2}{2(1-t^2)}\right) \pi_1\left(\mathrm{d} \boldsymbol{x}_1\right)}{\int_{\mathbb{R}^d} \exp \left(-\frac{\left\|\boldsymbol{x}-t \boldsymbol{x}_1\right\|^2}{2(1-t^2)}\right) \pi_1\left(\mathrm{d} \boldsymbol{x}_1\right)}\right)^{\otimes 2} \\
= & -\frac{1}{1-t^2} I_d+\frac{t^2}{(1-t^2)^2} \operatorname{Cov}[X_1 | X_t=\boldsymbol{x}].
\end{aligned}
\end{align*}
Combing the above identities, we get the desired result.
\end{proof}

\begin{lemma}
\label{lemma: true 5}
Suppose Assumption \ref{ass: bounded support gamma} holds. Then $\sup_{t \in [0, T]} \sup_{\boldsymbol{x} \in [-R, R]^d} \|\nabla \boldsymbol{v}^*(\boldsymbol{x}, t)\|_{\mathrm{op}} \leq \frac{T d}{(1-T^2)^2}$.
\end{lemma}

\begin{proof}
Since we assume the target distribution $\pi_1$ is supported on $[0,1]^d$, we have the following evaluation of the covariance matrix
\begin{align*}
0 \preceq \operatorname{Cov}[X_1 | X_t=\boldsymbol{x}] \preceq d I_d.
\end{align*}
Thus, we have
\begin{align*}
-\frac{t}{1-t^2} I_d \preceq \nabla \boldsymbol{v}^*(\boldsymbol{x}, t) \preceq\left(\frac{t d}{(1-t^2)^2}-\frac{t}{1-t^2}\right) I_d.
\end{align*}
The above inequality implies the upper bound.
\end{proof}

\vskip 0.2in
\bibliography{references}

\end{document}